\documentclass[11pt,reqno]{article}
\usepackage[numbers]{natbib}
\usepackage{etoolbox}

\usepackage[utf8]{inputenc} 
\usepackage[T1]{fontenc}    
\usepackage[linktocpage=true,colorlinks=true,citecolor=NavyBlue]{hyperref}       
\usepackage{url}            
\usepackage{booktabs}       
\usepackage{amsfonts}       
\usepackage{nicefrac}       
\usepackage{microtype}      
\usepackage[a4paper,margin=1in]{geometry}

\usepackage{amsmath}
\usepackage{color}
\usepackage{enumitem}
\usepackage{amssymb}
\usepackage{amsthm}
\usepackage{graphicx}
\usepackage{dsfont}
\usepackage{subcaption}
\usepackage[T1]{fontenc}
\usepackage{lmodern}
\usepackage{mathrsfs}
\usepackage{tikz}
\usepackage{subcaption}
\usepackage{easyReview}
\allowdisplaybreaks

\DeclareSymbolFont{rsfs}{U}{rsfs}{m}{n}
\DeclareSymbolFontAlphabet{\mathscrsfs}{rsfs}

\def\Risk{\mathscrsfs{R}}
\def\tRisk{\overline{\mathscrsfs{R}}}

\def\rP{{\rm P}}
\def\<{\langle}
\def\>{\rangle}
\def\normal{{\mathcal{N}}}

\def\mf{\mbox{\tiny\rm mf}}
\def\amf{a^{\mbox{\tiny\rm mf}}}
\def\smf{s^{\mbox{\tiny\rm mf}}}
\def\red{\mbox{\tiny\rm red}}
\def\init{\rm{init}}
\def\hrho{\hat{\rho}}
\def\de{\mathrm{d}}

\newcommand{\op}{\mathrm{op}}
\newcommand{\norm}[1]{\left\|{#1}\right\|}
\newcommand{\iid}{\mathrm{i.i.d.}}

\DeclareMathOperator*{\argmin}{argmin}

\DeclareMathOperator{\sign}{sign}

\definecolor{NavyBlue}{rgb}{0.1,0.1,0.6}

\newcommand{\modif}[1]{#1}

\def\id{I}
\def\He{\mathrm{He}}
\def\bfone{\mathds{1}}
\def\eps{\varepsilon}
\def\var{{\rm var\,}}
\newcommand{\diff}{\mathrm{d}}

\newcommand{\E}{\mathbb{E}}
\newcommand{\R}{\mathbb{R}}

\renewcommand{\S}{\mathbb{S}}
\def\reals{{\mathbb R}}

\renewcommand{\P}{\mathds{P}}

\newcommand{\cF}{\mathcal{F}}

\newcommand{\Unif}{\mathrm{Unif}}
\newcommand{\sN}{\mathsf{N}}
\def\os{\overline{s}}
\def\oa{\overline{a}}

\def\oa{\overline{a}}
\def\ou{\overline{u}}
\def\orho{\overline{\rho}}
\def\otheta{\overline{\theta}}
\def\ta{\widetilde{a}}
\def\ts{\widetilde{s}}

\newcommand{\bu}{\boldsymbol u}

\newcommand{\veps}{\varepsilon}

\newcommand{\floor}[1]{\lfloor #1 \rfloor}

\renewcommand{\leq}{\leqslant}
\renewcommand{\geq}{\geqslant}

\renewcommand{\hat}{\widehat}

\newtheorem{definition}{Definition}

\newtheorem{proposition}{Proposition}
\newtheorem{thm}{Theorem}
\newtheorem{lem}{Lemma}
\newtheorem{coro}{Corollary}

\newtheoremstyle{myremark} 
    {\topsep}                    
    {\topsep}                    
    {\rm}                        
    {}                           
    {\bf}                        
    {.}                          
    {.5em}                       
    {}  

\theoremstyle{myremark}
\newtheorem{remark}{Remark}[section]

\begin{document}
\title{Learning time-scales in  two-layers neural networks}
\author{Raphaël Berthier\thanks{EPFL, email address: \href{mailto:raphael.berthier1@gmail.com}{raphael.berthier1@gmail.com}}, \;\; Andrea Montanari\thanks{Department of Electrical Engineering and Department of Statistics, 
		Stanford University, email address: \href{mailto:montanar@stanford.edu}{montanar@stanford.edu}}, \;\; Kangjie Zhou\thanks{Department of Statistics, 
		Stanford University, email address: \href{mailto:kangjie@stanford.edu}{kangjie@stanford.edu}}}
\date{\today}
\maketitle

\begin{abstract}
Gradient-based learning in multi-layer neural networks displays a number of striking 
features. In particular, the decrease rate of empirical risk is non-monotone even after averaging
over large batches. Long plateaus in which one observes barely any progress 
alternate with intervals of rapid decrease. These successive phases of learning often
take place on very different time scales. Finally, models learnt in an early phase 
are typically `simpler' or `easier to learn' although in a way that is difficult to formalize.

Although theoretical explanations of these phenomena have been put forward,
each of them captures at best certain specific regimes. In this
paper, we study the gradient flow dynamics of a 
wide two-layer neural network in high-dimension, when data are distributed
according to a single-index model (i.e., the target function depends on a one-dimensional 
projection of the covariates).
Based on a mixture of new rigorous results, non-rigorous mathematical derivations, and
numerical simulations,
we  propose a scenario for the learning dynamics in this setting.
In particular, the proposed evolution exhibits separation of timescales and intermittency.
These behaviors arise naturally because the population gradient flow can
be recast as a singularly perturbed dynamical system.
\end{abstract}


\vspace{1em}


\vspace{1em}


\newpage

\tableofcontents

\section{Introduction}

It is a recurring empirical observation that the training dynamics of 
 neural networks exhibits a whole range of surprising behaviors:
\begin{enumerate}
\item \emph{Plateaus.} Plotting the training and test error as a function of SGD
steps, using either small stepsize or large batches to average out stochasticity, 
reveals striking patterns. These error curves display long plateaus where barely anything 
seems to be happening, which are followed by rapid drops
\citep{saad1995line,yoshida2019data,power2022grokking}.
\item \emph{Time-scales separation.} The time window for this rapid descent is much shorter than
the time spent in the plateaus. Additionally, subsequent phases of learning take 
increasingly longer times \citep{ghorbani2020discussion,barak2022hidden}. 
\item \emph{Incremental learning.}
Models learnt in the first phases of learning appear to be simpler than in later phases.
Among others,  \cite{arpit2017closer}  demonstrated that easier examples in a dataset are 
learned earlier; \cite{kalimeris2019sgd} showed that models learnt in the first phase of
training correlate well with linear models; \cite{gissin2019implicit} showed that,
 in many simplified models, the dynamics of gradient descent explores
 the solution space in an incremental order of complexity; \cite{power2022grokking} demonstrated that,
 in certain settings, a function that approximates well the target 
  is only learnt past the point of overfitting.
\end{enumerate}
Understanding these phenomena is not a matter of intellectual curiosity.
 In particular, incremental learning plays a key role in our understanding
of generalization in deep learning. Indeed, in this scenario, stopping the learning at 
a certain time $t$ amounts to controlling the complexity of the model learnt. 
The notion of complexity corresponds to the order in which the space of models is explored.

While a number of groups have developed models to explain these phenomena,
it is fair to say that a complete picture is still lacking. An exhaustive overview of
these works is out of place here.  We will outline three possible explanations that
have been developed in the past, and 
provide more pointers in Section \ref{sec:Related}.

\paragraph{Theory $\#1$: Dynamics near singular points.} Several early works
\citep{saad1995line,fukumizu2000local,wei2008dynamics} pointed out that the parametrization of
multi-layer neural networks presents symmetries and degeneracies. 
For instance, the function represented 
by a multi-layer perceptron is invariant under permutations of 
the neurons in the same layer. As a consequence, the population risk
has multiple local minima connected through saddles or other singular sub-manifolds.
Dynamics near these sub-manifolds 
naturally exhibits plateaus. Further, random or agnostic initializations typically
place the network close to such submanifolds.

\paragraph{Theory $\#2$: Linear networks.} Following the pioneering work of \cite{baldi1989neural},
a number of authors, most notably \cite{saxe2013exact, li2020towards}, studied the behavior of deep
neural networks with  linear activations. While such networks can only represent 
linear functions, the training dynamics is highly non-linear. As demonstrated in 
\cite{saxe2013exact}, learning happens through stages that correspond to the singular value
decomposition of the input-output covariance. Time scales are determined by the singular values.

\paragraph{Theory $\#3$: Kernel regime.}  Following an initial insight of \cite{jacot2018neural},
a number of groups proved that, for certain initializations, the training dynamics and model learnt 
by overparametrized  neural networks is well approximated by certain linearly parametrized models.  
In the limit of very wide networks, the training dynamics of these models
 converges in turn to the training dynamics of kernel ridge(less)
regression (KRR) with respect to a deterministic kernel (independent of the random initialization.)
We  refer to \cite{bartlett2021deep} for an overview and pointers to this literature. 
Recently \cite{ghosh2021three} show that, in high dimension, the learning dynamics of KRR
also exhibits plateaus and waterfalls, and learns functions of increasing complexity over a 
diverging sequence of timescales.

\vspace{0.3cm}

While each of these theories offers useful insights, it is important to realize 
that they do not agree on the basic  mechanism  that explains  plateaus, time-scales
separation, and incremental learning.
In theory $\# 1$, plateaus are associated to singular manifolds and high-dimensional saddles,
while  in theories $\# 2$  and   $\# 3$ they are related to a hierarchy of singular values of a
certain matrix.
In $\#2$, the relevant singular values are the ones of the input-output covariance,
and the fact that these singular values are well separated is postulated to be a property
of the data distribution.
In contrast, in $\#3$ the relevant singular values are the eigenvalues of the kernel
operator, and hence completely independent of the output (the target function).
In this case, eigenvalues which are very different are proved 
to exist under natural high-dimensional distributions.

Not only these theories propose different explanations, but they are also motivated by very different
simplified models. 
Theory $\#1$ has been developed only for networks with a small number of hidden units. 
 Theory $\#2$ only applies to networks with multiple output units, because otherwise
 the input-output covariance is a $d\times 1$ matrix and hence has only one 
 non-trivial singular value.
 Finally, theory $\#3$ applies under the conditions of 
 the linear (a.k.a.~lazy) regime, namely large overparametrization and suitable initialization
 (see, e.g., \cite{bartlett2021deep}).

In order to better understand the origin of plateaus, time-scales separation, and 
incremental learning, we attempt a detailed analysis of gradient flow
for two-layer neural networks. We consider a simple data-generation model,
and propose a precise scenario for the behavior of learning dynamics. 
We do not assume any of the simplifying features of the theories described above:
activations are non-linear; the number of hidden neurons is large; we place ourselves outside
the linear (lazy) regime.

Our analysis is based on methods from dynamical systems theory: singular perturbation theory
and matched asymptotic expansions. Unfortunately, we fall short of providing a general rigorous
proof of the proposed scenario, but we can nevertheless 
prove it in several special cases and provide a heuristic argument supporting its generality.

The rest of the paper is organized as follows. Section \ref{sec:Scenario}
describes our data distribution, learning model, and the proposed  scenario for the learning 
dynamics. We review further related work in Section \ref{sec:Related}.
Section~\ref{sec:LargeNet} describes the reduction of the gradient flow to 
a `mean field' dynamics that will be the starting point of our analysis.
 Section \ref{sec:Empirical} presents numerical evidence of the proposed learning scenario.
Finally,  Sections~\ref{sec:matched}
to \ref{sec:gf_sgd} present our analysis of the learning dynamics.

\paragraph{Notations.} In this paper, we use the classical asymptotic notations. The notations $f(\varepsilon) = o(g(\varepsilon))$ or $g(\varepsilon) = \omega(f(\varepsilon))$ as $\varepsilon \to 0$ both denote that $|f(\varepsilon)| / |g(\varepsilon)\vert \to 0$ in the limit $\varepsilon \to 0$. The notations $f(\varepsilon) = O(g(\varepsilon))$ or $g(\varepsilon) = \Omega(f(\varepsilon))$ both denote that the ratio $|f(\varepsilon)| / |g(\varepsilon)|$ remains upper bounded in the limit. The notation $f(\varepsilon) = \Theta(g(\varepsilon))$ or $f(\varepsilon) \asymp g(\varepsilon)$ denote that $f(\varepsilon) = O(g(\varepsilon))$ and $g(\varepsilon) = O(f(\varepsilon))$ both hold. Finally, $f(\varepsilon) \sim g(\varepsilon)$ denotes that $f(\varepsilon) / g(\varepsilon) \to 1$ in the limit. 

\section{Setting and \modif{canonical learning order}}
\label{sec:Scenario}

We are given pairs $\{(x_i,y_i)\}_{i\le n}$, where $x_i\in\R^d$ is a feature vector
and $y_i\in\R$ is a response variable. We are interested in cases in which the 
feature vector is high-dimensional but does not contain strong structure, but
the response depends on a low-dimensional projection of the data.
We assume the simplest model of this type, the so-called single-index model:
\begin{equation}\label{eq:single_index}
y_i = \varphi(\langle u_*, x_i \rangle) \,, \qquad \ x_i \sim \sN(0, I_d), \; u_* \in \S^{d-1},
\end{equation}
where $\varphi: \R \to \R$ is a link function, $\sN(0, I_d)$ denotes the standard multivariate Gaussian distribution in dimension $d$, and $\S^{d-1}:=\{v\in\R^d:\, \|v\|_2=1\}$.
We study the ability to learn model~\eqref{eq:single_index} using a two-layers 
neural network with $m$ hidden neurons:
\begin{equation}
f(x;a,u) = \frac{1}{m} \sum_{i=1}^{m} a_i \sigma(\langle u_i, x \rangle), \qquad 
\ a_1, \cdots, a_m \in \R, \ u_1, \cdots, u_m \in \S^{d - 1},\label{eq:First-NNET}
\end{equation}
where $(a,u):=(a_1, \cdots, a_m,u_1, \cdots, u_m )$ collectively denotes all the model's parameter \modif{and $\sigma : \R \to \R$ is the activation function of the neural network}.
The factor $1/m$ in the definition is relevant for the initialization and learning rate.
We anticipate that we will initialize the $a_i$'s to be of order one, which results in second layer
coefficients $a_i/m= \Theta(1/m)$.  
\modif{\begin{remark}
Standard initializations in deep learning frameworks yield second-layer
coefficients $a_i/m= \Theta(1/\sqrt{m})$ 
\citep{lecun2002efficient,glorot2010understanding,he2015delving}. However, it is
increasingly clear that this initialization presents fundamental limitations
for large $m$. Notably, two-layers networks with this initialization converges
to kernel methods \citep{oymak2020toward}, and the latter cannot learn ridge functions
from polynomially many samples \citep{ghorbani2021linearized,yehudai2019power}. 

It is well understood that, in order to drive the learning process outside the 
kernel regime (for $m\to\infty$), it is necessary to set $a_i/m= \Theta(1/m)$.
This is often referred to as the `mean-field initialization'
\citep{mei2018mean,chizat2018global,ghorbani2020neural,abbe2022merged}. 
We notice that suitable generalizations of the mean-field initialization
are currently used in state-of-the-art implementations \citep{yang2020feature,yang2021tensor}.
\end{remark}}

The bulk of our work will be devoted to the analysis of  projected gradient flow in
 $(a_i,u_i)_{1 \leq i \leq m}$ on the population risk 
\begin{align}
\Risk(a,u) &= \frac{1}{2} \E \big\{\big(y - f(x;a,u)\big)^2\big\}\\
&  = \frac{1}{2} \E \Big\{
\Big(\varphi(\langle u_*, x \rangle) - \frac{1}{m} \sum_{i=1}^{m} a_i \sigma(\langle u_i, x \rangle)\Big)^2 
\Big\}\, .
\end{align}
In Section \ref{sec:gf_sgd}, we will bound the distance between stochastic
gradient descent (SGD) and gradient flow in
population risk. As a consequence, we will establish finite sample generalization guarantees for 
SGD learning.

Projected gradient flow  with respect to the risk $\Risk(a,u)$ is defined by the following
 ordinary differential equations (ODEs):
\begin{align}
\partial_t (\veps a_i) &= - m \partial_{a_i} \Risk(a,u) \, ,  \label{eq:GF-1} \\
\partial_t u_i &= - m (I_d - u_i u_i^\top) \nabla_{u_i} \Risk(a,u) \, . \label{eq:GF-2}
\end{align}
\modif{Here, $\veps$ can be viewed as the relative step size, namely the ratio between the first and second-layer step sizes.} It is useful to make a few remarks about the definition of gradient flow:
\begin{itemize}
\item The 
 projection $I_d - u_i u_i^\top$ ensures that $u_i$ remains on the unit sphere $\S^{d-1}$.
\item The overall scaling of time is arbitrary, and the matching to SGD steps will be carried out in Section \ref{sec:gf_sgd}.
The factors $m$ on the right-hand side are introduced for mathematical convenience, since
the partial derivatives are of order $1/m$.
\item 

\modif{As aforementioned, the factor $\eps$ introduced in the gradient flow of the $a_i$'s plays the role of the relative step size.}
Throughout the paper, we will keep $\eps$ as a free parameter independent of $m$, and study 
the evolution of gradient flow for small $\eps$. \modif{This corresponds to a 
setting in which the second-layer coefficients are learned much faster than the first-layer weights.
We emphasize however that the small $\eps$ limit is taken after the large $m,d$ limits. Thus,
despite the second-layer weights are learnt faster, the evolution of first layer weights will be crucial, 
and lead to true feature learning.}
\end{itemize}
 
 We assume the initialization to be random with i.i.d.~components $(a_{i,\rm{init}},u_{i,\rm{init}})$:
 \begin{align}
 (a_{i,\rm{init}},u_{i,\rm{init}})\sim \rP_A\otimes \Unif(\S^{d-1})\, ,\label{eq:Initialization}
\end{align}
where $\rP_A$ is a probability measure on $\R$. The unique solution of the gradient flow ODEs with this initialization will be denoted by
$(a(t),u(t))$.
We will be interested  in the case of large networks ($m\to\infty$) in high dimension
($d\to\infty$). As shown below, the two limits commute (over fixed time horizons).

Our main finding is that, in a number of cases, $\varphi$ is  learnt incrementally.
Namely, the function $f(x;a(t),u(t))$ evolves over time according to 
a sequence of polynomial approximations of $\varphi(\langle u_*,x \rangle)$. 
These polynomial approximations are given by 
the decomposition of $\varphi$ in
$L^2 (\R, \phi (x) \mathrm{d} x)$, where $\phi (x)$ is the  standard normal density:
$\phi (x) = \exp(- x^2 / 2) / \sqrt{2 \pi}$. (For notational simplicity, we will use the shorthand 
$L^2$ instead of $L^2 (\R, \phi (x) \mathrm{d} x)$ in the sequel.)

In order to describe the polynomial approximations learnt during the training more explicitly, 
we decompose $\varphi$ and $\sigma$ into normalized Hermite polynomials:
\begin{align}
&\varphi (z) = \sum_{k=0}^{\infty} \varphi_k \mathrm{He}_k (z) \, , 
\;\;\;\;\sigma (z) =
 \sum_{k=0}^{\infty} \sigma_k \mathrm{He}_k (z) \, .\label{eq:PolynomialDec}
\end{align}
Here, $\mathrm{He}_k$ denotes the $k$-th Hermite polynomial, normalized so that 
$\norm{\mathrm{He}_k}_{L^2 (\R, \phi (x) \mathrm{d} x)} = 1$.

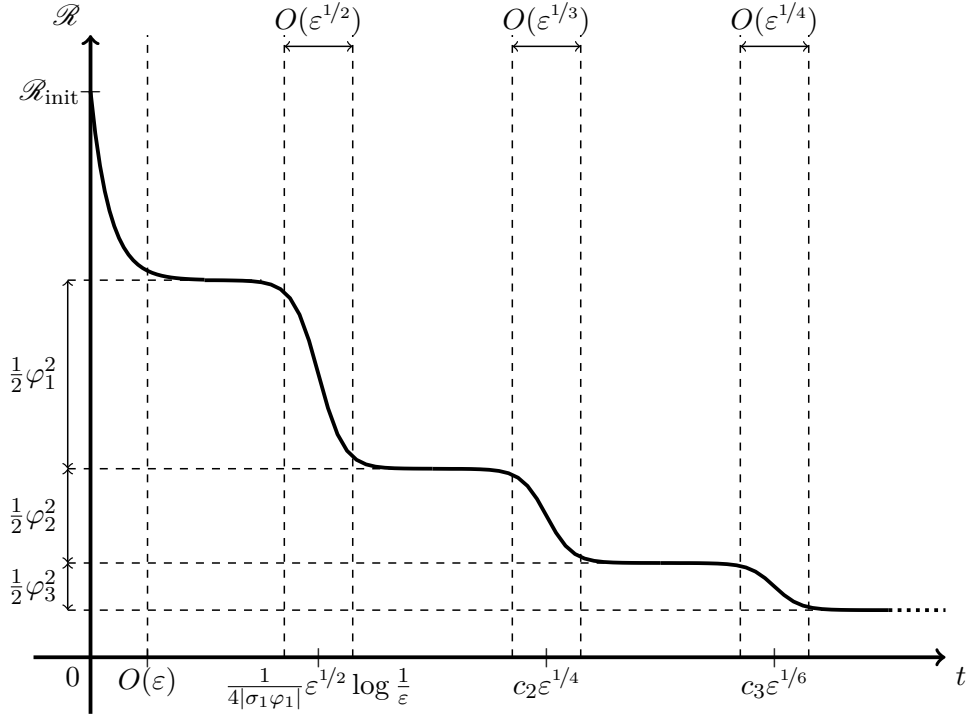
\begin{figure}
\begin{center}
	\begin{tikzpicture}[scale=1.5]
	\draw [domain=0:1, scale=1, variable=\t, line width = 0.5mm] plot (\t,{5*(1/3*exp(-6*\t)+2/3)});
	\draw [domain=1:3, scale=1, variable=\t, line width = 0.5mm] plot (\t,{5*(1/3+1/3*1/(1+exp(9*(\t-2))))});
	\draw [domain=3:5, scale=1, variable=\t, line width = 0.5mm] plot (\t,{5*(1/6+1/6*1/(1+exp(9*(\t-4))))});
	\draw [domain=5:7, scale=1, variable=\t, line width = 0.5mm] plot (\t,{5*(1/12+1/12*1/(1+exp(9*(\t-6))))});
	\draw [domain=7:7.5, scale=1, variable=\t, line width = 0.5mm, dotted] plot (\t,{5*1/12});
	
	\draw[->,line width = 0.5mm] (-0.5,0) -- (7.5,0);
	\draw[->, line width = 0.5mm] (0,-0.5) -- (0,5.5);
	\draw (0,0) node[below left] {$0$};
	\draw (7.5,0) node[below right] {$t$};
	\draw (0,5) node[left] {$\Risk_{\rm{init}}$} node{$+$};
	\draw (0,5.5) node[above left] {$\Risk$};
	
	\draw[-, dashed, line width = 0.2mm] (0.5,0) -- (0.5,5.5);
	\draw[-, dashed, line width = 0.2mm] (1.7,0) -- (1.7,5.5);
	\draw[-, dashed, line width = 0.2mm] (2.3,0) -- (2.3,5.5);
	\draw[-, dashed, line width = 0.2mm] (3.7,0) -- (3.7,5.5);
	\draw[-, dashed, line width = 0.2mm] (4.3,0) -- (4.3,5.5);
	\draw[-, dashed, line width = 0.2mm] (5.7,0) -- (5.7,5.5);
	\draw[-, dashed, line width = 0.2mm] (6.3,0) -- (6.3,5.5);
	
	\draw (0.5,0) node[below] {$O(\varepsilon)$} node{$+$};
	\draw (2,0) node[below] {$\frac{1}{4|\sigma_1 \varphi_1|}\varepsilon^{\nicefrac{1}{2}} \log \frac{1}{\varepsilon}$} node{$+$};
	\draw (4,0) node[below] {$c_2 \varepsilon^{\nicefrac{1}{4}}$} node{$+$};
	\draw (6,0) node[below] {$c_3 \varepsilon^{\nicefrac{1}{6}}$} node{$+$};
	
	\draw[<->, line width = 0.2mm] (1.7,5.4) -- (2.3,5.4) node[above, midway]{$O(\varepsilon^{\nicefrac{1}{2}})$};
	\draw[<->, line width = 0.2mm] (3.7,5.4) -- (4.3,5.4) node[above, midway]{$O(\varepsilon^{\nicefrac{1}{3}})$};
	\draw[<->, line width = 0.2mm] (5.7,5.4) -- (6.3,5.4) node[above, midway]{$O(\varepsilon^{\nicefrac{1}{4}})$};
	
	\draw[-, dashed, line width = 0.2mm] (-0.2,5*2/3) -- (1.5,5*2/3);
	\draw[-, dashed, line width = 0.2mm] (-0.2,5*1/3) -- (3.5,5*1/3);
	\draw[-, dashed, line width = 0.2mm] (-0.2,5*1/6) -- (5.5,5*1/6);
	\draw[-, dashed, line width = 0.2mm] (-0.2,5*1/12) -- (6.5,5*1/12);
	
	\draw[<->, line width = 0.2mm] (-0.2,5*2/3) -- (-0.2,5*1/3) node[left, midway]{$\frac{1}{2}\varphi_1^2$};
	\draw[<->, line width = 0.2mm] (-0.2,5*1/3) -- (-0.2,5*1/6) node[left, midway]{$\frac{1}{2}\varphi_2^2$};
	\draw[<->, line width = 0.2mm] (-0.2,5*1/6) -- (-0.2,5*1/12) node[left, midway]{$\frac{1}{2}\varphi_3^2$};
	\end{tikzpicture}
\end{center}
\caption{Cartoon illustration of the evolution of the population risk within the \modif{canonical learning order}
of Definition \ref{def:StandardScenario}.}
\label{fig:conj}
\end{figure}
%
As we will see, the  incremental learning behavior arises for small $\veps$.
By the law of large numbers (see below), the following almost sure limit exists 
(provided $\rP_A$ is square integrable)
\begin{align}
  \Risk_{\rm{init}}:= \lim_{m\to\infty}\lim_{d\to\infty}\Risk(a_{\rm{init}},u_{\rm{init}})\, = \frac{1}{2} \left( \varphi_0 - \sigma_0 \int\! a\,  \rP_A (\diff a) \right)^2 + \frac{1}{2} \sum_{k \geq 1} \varphi_k^2.
\end{align}
We are now in position to describe the scenario that we will study in the rest
of the paper.
\begin{definition}\label{def:StandardScenario}
We say that the \emph{\modif{canonical learning order} holds up to level $L$} for
a certain target function $\varphi$, activation $\sigma$, and distribution $\rP_A$,
if the followings hold:
 \begin{enumerate}
 \item The limit below exists:
 \begin{align}
 \Risk_{\infty}(t,\veps) = \lim_{m\to\infty}\lim_{d\to\infty}
\Risk(a(t), u(t)).
 \end{align}
 %
\item  There exist constants
 $c_2, \dots, c_{L+1} > 0$ such that the following asymptotic holds as
  $\varepsilon \to 0$, $t \to 0$:
	\begin{equation*}
	\Risk_{\infty}(t,\veps)\xrightarrow[\varepsilon \to 0,\, t \to 0]{} \begin{cases}
	\Risk_{\rm{init}} &\text{if }t = o(\varepsilon) \, ,  \\
	 \frac{1}{2} \sum_{k \geq 1} \varphi_k^2 &\text{if }t = \omega(\varepsilon) \text{ and } t =  \frac{1}{4 |\sigma_1 \varphi_1|} \varepsilon^{\nicefrac{1}{2}}  \log \frac{1}{\varepsilon} - \omega(\varepsilon^{\nicefrac{1}{2}}) \, , \\ 
	  \frac{1}{2} \sum_{k \geq 2} \varphi_k^2 &\text{if } t =  \frac{1}{4 |\sigma_1 \varphi_1|} \varepsilon^{\nicefrac{1}{2}}  \log \frac{1}{\varepsilon} + \omega(\varepsilon^{\nicefrac{1}{2}}) \text{ and } t = c_2 \varepsilon^{\nicefrac{1}{4}} - \omega(\varepsilon^{\nicefrac{1}{3}}) \, , \\
	   \frac{1}{2} \sum_{k \geq l} \varphi_k^2 &\text{if } t = c_{l-1} \varepsilon^{\nicefrac{1}{2(l-1)}} + \omega(\varepsilon^{\nicefrac{1}{l}})  \text{ and } t = c_l \varepsilon^{\nicefrac{1}{2l}}  - \omega(\varepsilon^{\nicefrac{1}{l+1}}) \, , \\
	   &\qquad\text{ for all } 3 \leq l \leq L+1.
	\end{cases}
	\end{equation*}
\end{enumerate}
\end{definition}
Figure \ref{fig:conj} provides a cartoon illustration of the \modif{canonical learning order}.

\modif{At first sight, the setting of Eq.~\eqref{eq:First-NNET} is overly restrictive because we 
require $\|u_i\|_2=1$ and we do not have offsets in the activations. Therefore, it might seem 
that $s=1$ and $\varphi=\sigma$ is required in order to approximate arbitrarily well the target function.
In contrast, the next proposition shows that the network \eqref{eq:First-NNET} enjoys universal 
approximation properties.
\begin{proposition}\label{propo:Approx}
Assume that $\sigma$ is Lipschitz continuous and generic in the following sense:
the decomposition of $\sigma$ into Hermite polynomials does not
have any coefficient equal to $0$.
For any Lipschitz function $\varphi:\reals\to\reals$, $\|u_*\|_2=1$, and $x \sim \normal(0, I_d)$ such that $\E\{\varphi(\<u_*,x\>)^2\}<\infty$, there exists a sequence $m\to\infty$
and $a^{(m)}, u^{(m)}$ with $\Vert u_i^{(m)} \Vert_2 = 1$ such that
\begin{align*}
\lim_{d\to\infty}\lim_{m\to\infty} \E\big\{\big(\varphi(\<u_*,x\>)-f(x;a^{(m)},u^{(m)})\big)^2\big\} = 0\, .
\end{align*}
\end{proposition}
This result is not surprising in view of the arguments in the next sections,
which suggest that indeed gradient flow constructs such an approximation for a broad class
of functions of the form $f_*(x) = \varphi(\<u_*,x\>)$. We nevertheless give an independent 
proof in Appendix \ref{app:Approx}.
}

A specific realization of our general setup is determined by the triple $(\sigma,\varphi,\rP_A)$,
In the rest of the paper, we will provide evidence showing that the \modif{canonical learning order} holds in a number of cases. Nevertheless, we can also construct 
examples in which it does not hold:
\begin{itemize}
\item If one or more of the Hermite coefficients of the activation vanish, then 
the \modif{canonical learning order} does not hold for general $\varphi$. Specifically, if $\sigma_k=0$, then
for any $t$ 
the function $f(x;a(t),u(t))$ remains
orthogonal to 
$\He_k(\langle u_*,x\>)$. In particular, if $\varphi_k\neq 0$ then the risk remains bounded
away from zero for every $t$. We refer to Appendix~\ref{sec:counter_ex1} for a formal statement.
\item If the first \modif{$k+1$ Hermite} coefficients of $\varphi$ vanish,
$\varphi_0= \dots = \varphi_k=0$, $k\ge 1$, then the \modif{canonical learning order} does not hold. (See Appendix~\ref{sec:counter_ex2} for the proof.)
\item In fact, we expect the \modif{canonical learning order} might fail every time one or more of
the coefficients $\varphi_k$ vanish, for $k\ge 1$. 
Appendix~\ref{sec:counter_ex3} provides some heuristic justification for this failure.
\end{itemize}

\begin{remark}
We can compare the \modif{canonical learning order} described here to the ones 
in earlier literature and described as theory $\#1$, $\#2$, $\#3$ in the introduction.
There appears points of contact, but also important differences with
both theory $\#1$ and $\#3$:
\begin{itemize}
\item As in theory $\#1$, the plateaus and separation of time scales arise because 
the trajectory of gradient flow is approximated by a sequence of motions along submanifolds
in the space of parameters $(a,u)$.
Along the $l$-th such  submanifold $f(x;a,u)$ is well-approximated by a degree-$l$ polynomial.
Escaping each submanifold takes an increasingly longer time. 

This is \modif{reminiscent} of the motion between saddles investigated in earlier work
\citep{saad1995line,fukumizu2000local,wei2008dynamics}.
However, unlike in earlier work, we will see that this applies to networks with a
large (possibly diverging) number of hidden neurons. Also, we identify the subsequent phases
of learning with the polynomial decomposition of Eq.~\eqref{eq:PolynomialDec}.
\item As in theory $\#3$, subsequent phases of learning correspond to
increasingly accurate polynomial approximations of the target function $\varphi(\< u_*,x\>)$. 
However, the underlying mechanism and time scales are completely different.
In the linear regime, the different time scales emerge because of increasingly small 
eigenvalues of the neural tangent kernel. In that case, the time required to learn 
degree-$l$ polynomials is of order $d^{l}$ \citep{ghosh2021three}.

In contrast, in the \modif{canonical learning order}, polynomials of degree $l$ are learnt
on a time scale of order one in $d$
(and only depending on the learning rate $\veps$). This of course has important implications 
when approximating gradient flow by SGD. Within the linear regime, 
the sample size required to learn \modif{a} polynomial of order $l$ scales like $d^{l}$  \citep{ghosh2021three},
while in the \modif{canonical learning order}, it is only of order $d$ (see Section \ref{sec:gf_sgd}).
\end{itemize} 
\end{remark}

\section{Further related work}
\label{sec:Related}

As we mentioned in the introduction, plateaus and time scales in the learning dynamics of 
kernel models were analyzed by \cite{ghosh2021three}. A sharp analysis for
the related random features model was developed by \cite{bodin2021model}.

Our analysis builds upon the mean-field description of learning in two-layer neural networks, 
which was developed in a sequence of works, see, e.g.,  \citep{mei2018mean,rotskoff2018parameters,chizat2018global,mei2019mean}.
In particular, we leverage the fact that, for the data distribution \eqref{eq:single_index},
the population risk function is invariant under rotations around the axis $u_*$, and 
this allows for a dimensionality reduction in the mean field description. Similar
symmetry argument were used by \cite{mei2018mean} and, more recently, by \cite{abbe2022merged}.

The single-index model can be learnt using simpler methods 
than large two-layer networks. Limiting ourselves to the case of gradient descent 
algorithms, \cite{mei2018landscape} proved that gradient descent with respect to the non-convex
empirical risk $\widehat{R}_n(u) := n^{-1}\sum_{i=1}^n(y_i-\varphi(u^{\top}x_i))^2$ converges to
a near global optimum, provided $\varphi$ is strictly increasing.
\cite{arous2021online} considered online SGD under more challenging learning 
scenarios and characterized the time (sample size) for $|\<u,u_*\>|$ to become 
significantly larger than for a random unit vector $u$.

Learning in overparametrized two-layer networks under model \eqref{eq:single_index} (or its variations)
has been studied recently by several groups. In particular,
\cite{ba2022high} considers a training procedure which runs a single step gradient descent
followed by freezing the first layer and performing ridge regression with respect to the 
second layer. This scheme is amenable to a precise characterization of the generalization error.
\cite{bietti2022learning} consider a similar scheme in which a first phase of gradient descent is run 
to achieve positive correlation with the unknown direction $u_{*}$.
\cite{damian2022neural} also consider a two-phases scheme, and prove consistency and excess risk bounds 
for a more general class of target functions whereby the first equation
in \eqref{eq:single_index} is replaced by
\begin{align}
y_i = \varphi(U_*^{\top} x_i )+\eps_i\, , \;\;\; U_{*}\in \reals^{d \times k}\, ,
\varphi:\reals^k\to\reals\, ,\label{eq:LatentSpaceModel}
\end{align}
with $k\ll d$. In particular, near optimal error bounds are obtained under a non-degeneracy
condition on $\nabla^2\varphi$. 

\cite{abbe2022merged} consider a similar model whereby $x\sim \Unif(\{+1,-1\}^d)$,
and $y =\varphi(x_S)$ where $S\subseteq [d]$, and $x_S=(x_i)_{i\in S}$ (i.e., $x_S$
contains the coordinates of $x$ indexed by entries of $S$). Under a structural assumption 
on $\varphi$ (the `merged staircase property'), and for $|S|$  fixed,
they prove the two stages algorithm learns the target function with sample complexity
of order $d$. This paper is technically related to ours in that it 
uses mean-field theory to obtain a characterization of learning in terms of a PDE
in a reduced $(k+2)$-dimensional space.

A similar model was studied by \cite{barak2022hidden} that bounds the sample complexity by $d^{O(k)}$ for learning parities on $k$ bits
using gradient descent with large batches (if $k=O(1)$,   \cite{barak2022hidden} require
$O(1)$ steps with batch size $d^{O(k)}$).

Let us emphasize that our objective is quite different from these works. 
We \modif{implement a simple online SGD algorithm with additional projection steps,} and try to derive a precise picture
of the successive phases of learning (in particular, we do not consider 
two-stage schemes or layer-by-layer learning). On the other hand, we focus on a relatively simple
model.

To clarify the difference, it is perhaps useful to rephrase our claims in terms of sample complexity.
While previous works show that the target function can be learnt with $O(d)$
samples, we claim that it is learnt by online SGD with test error $r$ from 
about $C(r,\eps)d$ samples and characterize the dependence of $C(r,\eps)$ on $r$
for small $\eps$. (Falling short of a proof in 
the general case.)

After posting an initial version of this paper, we became aware that \cite{arnaboldi2023high} independently derived equations
	similar to \eqref{eq:risk}-\eqref{eq:dynamics-4}, \eqref{eq:FinalApproxRisk}, \eqref{eq:dynamics_perp}. There are technical differences, and hence we cannot apply their results directly. However, Section~\ref{sec:connect_mf} and Appendix~\ref{sec:derive_MF} are analogous to their work.

\section{The large-network, high-dimensional limit}
\label{sec:LargeNet}

The first step of our analysis is a reduction of the system of ODEs \eqref{eq:GF-1},
 \eqref{eq:GF-2}, with dimension $m(d+1)$ to a system of ODEs in $2m$ dimensions.
 We will achieve this reduction in two steps:
 \begin{itemize}
 \item[$(i)$]~First we reduce to a system in $m(m+3)/2$ dimensions for the 
 variables  $a_i$, $\langle u_i,u_j\>$, $\langle u_i,u_*\>$. This reduction is exact and is quite standard. It is done in Section \ref{sec:reduction-1}.
 \item[$(ii)$]~We then show that the products $\< u_i,u_j\>$ can be eliminated, with an 
 error $O(1/m)$. This is done in Section \ref{sec:reduction-2}. As further discussed below, the resulting dynamics could also be derived from the
  mean field  theory of \cite{mei2018mean,rotskoff2018parameters,chizat2018global,mei2019mean}
 (with the required modifications for the constraints $\|u_i\|=1$).
 \end{itemize}
In order to define formally the reduced system, we define the functions
$U,V:[-1,1]\to \R$ via:
\begin{align}
V(s) &:=  \E\{\varphi(G)\,\sigma(G_s)\}= \sum_{k\geq 0} \varphi_k  \sigma_k s^k \, , 
\;\;\;\;\; (G,G_s)\sim \normal\left(0, \left[\begin{matrix}1 & s\\ s& 1\end{matrix}\right]\right)\, ,\\
U(s) &:=\E\{\sigma(G)\,\sigma(G_s)\} = \sum_{k\geq 0}  \sigma_k^2 s^k \, .
\end{align}
Note that the above identities follow from \citep[Proposition 11.31]{o2014analysis}.
Throughout this section, we will make the following assumptions.
\begin{description}\label{ass:MF_limit}
	\item [A1.] The distribution of weights at initialization, $\rP_A$ is supported on $[-M_1, M_1]$. 
	\item [A2.] The activation function is bounded: $\norm{\sigma}_{\infty} \le M_2$. 
	Additionally, the functions $V$ and $U$ are bounded and of class $C^2$, with uniformly bounded first and second derivatives over $s \in [-1, 1]$. A sufficient condition for this is
	\begin{equation*}
		\sup \left\{ \norm{\sigma'}_{L^2}, \, \norm{\sigma''}_{L^2} \right\} \le M_2, \quad \ \sup \left\{ \norm{\varphi}_{L^2}, \, \norm{\varphi'}_{L^2}, \, \norm{\varphi''}_{L^2} \right\} \le M_2.
	\end{equation*}
	\item [A3.]  Responses  are bounded, i.e., $\|\varphi\|_{\infty}\le M_3$. 
\end{description}
\begin{remark}\label{rem:ass_A2_suff}
	We hereby briefly explain the sufficiency of $L^2$-boundedness of derivatives of $\sigma$ and $\varphi$ as claimed in Assumption A2. Suppose for example that $\norm{\sigma'}_{L^2} \le M_2$ and $\norm{\varphi'}_{L^2} \le M_2$, then we have
	\begin{equation}
		\sup_{s \in [-1, 1]} \left\vert V'(s) \right\vert \stackrel{(a)}{=} \sup_{s \in [-1, 1]} \left\vert \E\{\varphi'(G)\,\sigma'(G_s)\} \right\vert \stackrel{(b)}{\le} \norm{\varphi'}_{L^2} \norm{\sigma'}_{L^2} \le M_2^2,
	\end{equation}
    where $(a)$ follows from Gaussian integration by parts and $(b)$ follows from Cauchy-Schwarz inequality.
\end{remark}

\modif{\subsection{Reduction to $d$-independent flow}\label{sec:reduction-1}}

\noindent Our first statement establishes reduction $(i)$ mentioned above.
The proof of this fact is presented in Appendix \ref{sec:proof-prop-dynamics}.
\begin{proposition}[Reduction to $d$-independent flow]\label{prop:dynamics}
	Define $s_i = \langle u_i, u_* \rangle$, $r_{ij} = \langle u_i, u_j \rangle$ for 
	$i,j= 1, \dots, m$. Then, letting  $R=(r_{ij})_{i,j\le m}$, we have
	\begin{equation}\label{eq:risk}
	\Risk(a,u)= \Risk_{\red} (a, s, R) := \frac{1}{2} \Vert \varphi \Vert^2_{L^2} -
	 \frac{1}{m} \sum_{i=1}^m a_i V(s_i) +  \frac{1}{2m^2} \sum_{i,j=1}^m a_i a_j U(r_{ij}) \, .
	\end{equation}
	If $(a(t), u(t))$ solve the gradient flow ODEs 
	\eqref{eq:GF-1}-\eqref{eq:GF-2} then   $(a(t), s(t), R(t))$ are the unique solution of
	 the following set of ODEs (note that $r_{ii} =1$ identically)
	\begin{align}
	\veps\partial_t a_i = \, &  V(s_i) - \frac{1}{m} \sum_{j=1}^{m} a_j U(r_{ij}) \,, \label{eq:dynamics-1}\\
	\partial_t s_i = \, &  a_i \left( V'(s_i) (1-s_i^2) - \frac{1}{m} \sum_{j = 1}^{m} a_j U'(r_{ij}) (s_j - r_{ij} s_i) \right)\, , \label{eq:dynamics-2} \\
	\partial_t r_{ij} = \, & a_i \left( V'(s_i) (s_j - s_i r_{ij}) - \frac{1}{m} \sum_{p=1}^{m} a_p U'(r_{ip}) (r_{jp} - r_{ip} r_{ij}) \right)\, , \label{eq:dynamics-3}\\
	& +  a_j \left( V'(s_j) (s_i - s_j r_{ij}) - \frac{1}{m} \sum_{p=1}^{m} a_p U'(r_{jp}) (r_{ip} - r_{jp} r_{ij}) \right) \, . \label{eq:dynamics-4}
	\end{align}
\end{proposition}
The input dimension $d$ does not appear in the reduced ODEs, 
 Eqs.~\eqref{eq:dynamics-1} to \eqref{eq:dynamics-4}, and only plays a role in 
 the initialization of the $s_i$'s and the $r_{ij}$'s. 
 Namely, since $u_{i,\init}\sim \Unif(\S^{d-1})$, we can represent $u_{i,\init}= g_i/\|g_i\|_2$
 with $g_i\sim\sN(0,\id_d/d)$. By concentration of $\|g_i\|_2$, this implies
 that, for $1 \le i<j\le m$,  $s_i$, $r_{ij}$ are approximately $\sN(0,1/d)$.
 
 This discussion immediately yields the following consequence.
 \begin{coro}\label{coro:high_dim_lim}
 Let $(a(t), u(t))$ be the solution of the gradient flow ODEs 
	\eqref{eq:GF-1}, \eqref{eq:GF-2} with initialization~\eqref{eq:Initialization}, and
	let $(a^0(t),s^0(t), R^0(t))$ be the unique solution of  Eqs.~\eqref{eq:dynamics-1} to \eqref{eq:dynamics-4},
	with initialization $a^0_i(0) = a_i(0)$, $s^0_i(0) = 0$, $r^0_{ij}(0) = 0$ for $i\neq j$.
	Then, for any fixed $T$ (possibly dependent on $m$ but not on $d$), the followings holds with probability at least $1 - \exp(-C' m)$ over the i.i.d. initialization $(a_i (0), u_i (0))_{i \in [m]}$:
	\begin{align}
	&\sup_{t \in [0,T]}\big|\Risk(a(t),u(t))- \Risk_{\red} (a^0(t), s^0(t), R^0(t))\big| \le \frac{C M}{\sqrt{d}} \exp \left( M T (1 + T)^2 / \veps^2 \right) \, ,\label{eq:Risk-LargeD}\\
	& \max \left( \sup_{t \in [0,T]}\frac{1}{\sqrt{m}}\|a(t)-a^0(t)\|_2, \frac{1}{\sqrt{m}} \sup_{t \in [0,T]}\|s(t)-s^0(t)\|_2 \right) \le \frac{1}{\sqrt{d}} \cdot C \exp \left( M T (1 + T)^2 / \veps^2 \right) \, , \\
	&\sup_{t \in [0, T]}  \frac{1}{m}\|R(t)-R^0(t)\|_{\rm F} \le \frac{1}{\sqrt{d}} \cdot C \exp \left( M T (1 + T)^2 / \veps^2 \right) \,.\label{eq:R-LargeD}
	 \end{align}
     Here $C, C'$ are absolute constants and $M$ only depends on the $M_i$'s in Assumptions A1-A3.
 \end{coro}

The proof of Corollary~\ref{coro:high_dim_lim} is deferred to Appendix~\ref{sec:proof_only_coro}.
 From now on, we will assume the initialization $s^0_i(0) = 0$, $r^0_{ij}(0) = 0$ for $i \neq j$,
 but drop the superscript $0$ for notational simplicity. We notice in passing that
 the right-hand sides of Eqs.~\eqref{eq:Risk-LargeD} to \eqref{eq:R-LargeD} are independent
 of $m$: this approximation step holds uniformly over $m$. (Note that the left hand sides
 are normalized by $m$ as to yield the root mean square error per entry.)
 
\modif{\subsection{Elimination of the products $\langle u_i, u_j \rangle$}\label{sec:reduction-2}}

\noindent In order to state the reduction $(ii)$ outlined above, we  define the mean field
risk as
\begin{align}\label{eq:RmfDef}
\Risk_{\mf}(a,s):= \Risk_{\red} (a, s, R=ss^\top) = \frac{1}{2} \Vert \varphi \Vert^2_{L^2} -
	 \frac{1}{m} \sum_{i=1}^m a_i V(s_i) +  \frac{1}{2m^2} \sum_{i,j=1}^m a_i a_j U(s_{i}s_{j})
	 \, .
\end{align}
Further, we denote by
  $\{ \amf_i (t), \smf_i (t) \}_{i=1}^{m}$ the solution to the following ODEs:
\begin{equation}\label{eq:coupled_MF}
\begin{split}
\veps\partial_t a_i = \, & V(s_i) - \frac{1}{m} \sum_{j=1}^{m} a_j U(s_i s_j) \, ,\\
\partial_t s_i = \, &  a_i \left( 1 - s_i^2 \right) \left( V'(s_i) - \frac{1}{m} \sum_{j = 1}^{m} a_j U'(s_i s_j) s_j \right) \, .
\end{split}
\end{equation}
Note that \eqref{eq:coupled_MF} would be identical to \eqref{eq:dynamics-1}-\eqref{eq:dynamics-2} 
if we had $r_{ij}= s_is_j$. A priori, this is not the case. However, the two systems of 
equations are close to each other for large $m$ as made precise by our next proposition, which
formalizes reduction $(ii)$. 

\modif{The intuitive explanation for the approximation $r_{ij}\approx s_is_j$
is quite interesting. For large $m$, due to `propagation of chaos', the neuron weights 
$\{(u_i,a_i)\}_{i\le m}$ are approximately independent. Further, because
of the symmetry of the problem under rotations that keep $u_*$
fixed,  weights $(u_i)_{i\le n}$ are approximately uniformly distributed
conditional on $s_i=\<u_i,u_*\>$. As a consequence, decomposing $u_i=s_iu_*+u_i^{\perp}$,
we have $r_{ij} = s_is_j+ \<u_i^{\perp},u_j^{\perp}\>$, with $u_i^{\perp}$, $u_j^{\perp}$
approximately uniform on $\operatorname{span} \{ u_* \}^{\perp}$ and independent. Therefore, in high dimensions we have
$r_{ij} \approx  s_is_j$.}
\vspace{0.5em}
\begin{proposition}[Reduction to flow in $\R^{2m}$]\label{prop:coupling_W2_bd}
Let $(a_i(t),s_i(t),r_{ij}(t))_{1 \le i<j\le m}$ be the unique solution of  the ODEs~\eqref{eq:dynamics-1}-\eqref{eq:dynamics-4} 
with initialization $s_i (0) = 0$, $r_{ij} (0) = 0$ for all $1 \le i \neq j\le m$.
Let $(\amf_i (t), \smf_i (t))_{i\le m}$  
be the unique solution of  the ODEs~\eqref{eq:coupled_MF}
with initialization $\smf_i (0) = 0$, $\amf_{i} (0) = a_i(0)$ for all $i\le m$.

If assumptions A1-A3 hold, then for any $T < \infty$ there exists a constant
\begin{equation}
	C(T) = M \exp(M T (1 + T)^2 / \veps^2)
\end{equation}
(with $M$ depending 
on the constants $\{ M_i \}_{1 \le i \le 3}$ appearing in Assumptions A1-A3 only) such that:
	\begin{equation*}
		\sup_{t \in [0, T]}\frac{1}{m} \sum_{i = 1}^{m} \big\|(a_i (t), s_i (t))-
		(\amf_i (t), \smf_i (t))\big\|_2^2 \le \frac{C(T)}{m} \, .
	\end{equation*}
	Consequently, 
	\begin{equation*}
		\sup_{t \in [0, T]} \left\vert \Risk_{\red} \left( a(t), s(t), R(t) \right) - 
		\Risk_{\mf} \left( \amf(t), \smf(t)\right) \right\vert \le \frac{C(T)}{\sqrt{m}} \, .
	\end{equation*}
\end{proposition}
The proof of this proposition is deferred to Appendix~\ref{sec:proof_MF}. Now, combining the propositions and corollaries in this section, we deduce that with high probability over the i.i.d. initialization,
\begin{equation}
	\sup_{t \in [0, T]} \left\vert \Risk(a(t),u(t)) - 
	\Risk_{\mf} \left( \amf(t), \smf(t)\right) \right\vert \le \left( \frac{1}{\sqrt{d}} + \frac{1}{\sqrt{m}} \right) C M \exp(M T (1+T)^2 / \veps^2).\label{eq:FinalApproxRisk}
\end{equation}

\subsection{Connection with mean field theory}\label{sec:connect_mf}
Consider the empirical distributions of the neurons:
\begin{align}
\hrho_t &:= \frac{1}{m}\sum_{i=1}^{m} \delta_{(a_i (t), s_i (t))}\, ,\\
\rho_t &:= \frac{1}{m}\sum_{i=1}^{m} \delta_{(\amf_i (t), \smf_i (t))}\,  ,
\end{align}
with $(a_i(t),s_i(t))_{i\le m}$, $(\amf_i (t), \smf_i (t))_{i\le m}$  as
in the statement of Proposition \ref{prop:coupling_W2_bd}, i.e., solving
(respectively) Eqs.~\eqref{eq:dynamics-1}-\eqref{eq:dynamics-4} and
Eq.~\eqref{eq:coupled_MF} with initial conditions as given there.

Then, it is immediate  to show that  $\rho_t$  solves (in weak sense) the following
continuity partial differential equation (PDE) 
(we refer to \cite{ambrosio2005gradient,santambrogio2015optimal} for the definition of weak solutions 
and basic properties, and Appendix~\ref{sec:derive_MF} for a short derivation.)
\begin{align}\label{eq:MF_dynamics_rho_bar}
		\partial_t \rho_t (a, s) &= - \nabla \cdot \left( \rho_t \Psi \left( a, s; \rho_t \right) \right) \\
		&:= - \left( \partial_a \left( \rho_t \Psi_a \left( a, s; \rho_t \right) \right) + \partial_s \left( \rho_t \Psi_s \left( a, s; \rho_t \right) \right) \right),
\end{align}
	where  $\Psi = (\Psi_a, \Psi_s)$ is given by
\begin{align}
		\Psi_a (a, s; \rho) = \, & \veps^{-1}\cdot \left( V(s) - \int_{\R^2} a_1 U(s s_1) \rho (\mathrm{d} a_1, \mathrm{d} s_1) \right), \\
		\Psi_s (a, s; \rho) = \, & a(1 - s^2) \cdot \left( V'(s) - \int_{\R^2} a_1s_1 U'(s s_1) \rho(\mathrm{d} a_1, \mathrm{d} s_1) \right). \label{eq:psi-s}
\end{align}
This equation can be extended to a flow in the whole space $(\mathscr{P} (\R^2), W_2)$
(all probability measures on $\R^2$ equipped with the second Wasserstein distance),
and interpreted as gradient flow with respect to this metric in the following risk:
\begin{align}\label{eq:risk_mf_star}
\Risk_{\mf,*}(\rho):=  \frac{1}{2} \Vert \varphi \Vert^2_{L^2} -
	 \int\! a V(s) \, \rho (\de a, \de s) +  \frac{1}{2} 
	  \int\! a_1a_2 U(s_1s_2) \, \rho (\de a_1, \de s_1)\,\rho (\de a_2, \de s_2)\, \,,
\end{align}
which is the obvious extension of $\Risk_{\mf}(a,s)$ of Eq.~\eqref{eq:RmfDef}
to general probability distributions. 
Proposition~\ref{prop:coupling_W2_bd} implies 
	that for any $T < \infty$, and under the above initial conditions,
\begin{align}	
	\sup_{t \in [0, T]} W_2 (\rho_t,\hrho_t) \le \sqrt{\frac{M \exp(M T (1 + T)^2 / \veps^2)}{m}} \, .
\end{align}
If we further denote by $\rho_t^d$ the empirical distribution of $(a_i(t),s_i(t))$, $i\le m$,
when $s_i(0)=\<u_i(0),u_*\>$, $u_i(0)\sim\Unif(\S^{d-1})$, a further application
of Corollary \ref{coro:high_dim_lim} yields
\begin{align}	
	\sup_{t \in [0, T]} W_2 (\rho^d_t, \rho_t) \le \sqrt{\frac{M \exp(M T (1 + T)^2 / \veps^2)}{m\wedge d}} \, .
	\label{eq:FinalApproxW2}
\end{align}

Starting with 	\cite{mei2018mean,chizat2018global,rotskoff2018parameters}, several authors
used continuity PDEs of the form \eqref{eq:MF_dynamics_rho_bar} to study the learning dynamics 
of two-layer neural networks. Following the physics tradition, this is referred to as 
the `mean-field theory' of two-layer neural networks. 
 Appendix \ref{sec:details} sketches an alternative approach to prove bounds of the form
 \eqref{eq:FinalApproxRisk}, \eqref{eq:FinalApproxW2} using the results of 
 \cite{mei2018mean,mei2019mean}. The present derivation has the advantages of yielding
 a sharper bound and of being self-contained.

\subsection{A general formulation}
\label{sec:general}

As mentioned above, the system of ODEs in Eq.~\eqref{eq:coupled_MF} 
is a special case of the Wasserstein gradient flow of Eq.~\eqref{eq:MF_dynamics_rho_bar} 
whereby we set $\rho_0= m^{-1}\sum_{i=1}^m\delta_{(a_i^{\mf}(0),s_i^{\mf}(0))}$.
In order to study the solutions of Eq.~\eqref{eq:MF_dynamics_rho_bar} 
(hence Eq.~\eqref{eq:coupled_MF})  we adopt the following framework. 
Let $(\Omega, \rho)$ denote a probability space. 
Let $a = a(\omega, t)$ and $s = s(\omega,t)$ ($\omega \in \Omega$, $t \geq 0$) be two 
measurable functions satisfying (dropping dependencies in $t$ below)
	\begin{equation}
		\label{eq:general-dynamics}
		\begin{split}
			\veps\partial_t a(\omega) = \, & V(s(\omega)) - \int  a(\nu) U(s(\omega) s(\nu)) \diff \rho(\nu) \, ,\\
			\partial_t s(\omega) = \, &  a(\omega) \left( 1 - s(\omega)^2 \right) \left( V'(s(\omega)) -  \int  a(\nu) U'(s(\omega) s(\nu)) s(\nu) \diff \rho(\nu) \right) \, .
		\end{split}
	\end{equation}
If $\omega = i \in \Omega = \{1, \dots, m\}$ endowed with the uniform measure, we obtain the equations \eqref{eq:coupled_MF}. In general, the push-forward $\rho_t$ of the measure $\rho$ through the map $\omega \in \Omega \mapsto (a(\omega,t), s(\omega,t)) \in \R^2$ satisfies the mean-field equation \eqref{eq:MF_dynamics_rho_bar}. As a consequence, the dynamics \eqref{eq:general-dynamics} can be viewed as a gradient flow on the risk 
\begin{equation}
	\label{eq:risk-general}
	\Risk_{\mf,*}(\rho) = \frac{1}{2} \Vert \varphi \Vert^2 - \int a(\omega) V(s(\omega))  \diff \rho(\omega) + \frac{1}{2} \int a(\omega_1)a(\omega_2) U(s(\omega_1)s(\omega_2))  \diff  \rho(\omega_1)  \diff \rho(\omega_2)  \, .
\end{equation}

\modif{We next characterize the landscape of the risk function $\Risk_{\mf,*} (\rho)$. In particular,
 we establish that under certain conditions, the global infimum of $\Risk_{\mf,*} (\rho)$ is $0$.
\begin{proposition}\label{prop:risk_infimum}
	The risk function $\Risk_{\mf,*} (\rho)$ can be expressed as
	\begin{equation}
		\Risk_{\mf,*} (\rho) = \, \frac{1}{2} \sum_{k=0}^{\infty} \left( \varphi_k - \sigma_k \int a(\omega) s(\omega)^k \diff \rho(\omega) \right)^2.
	\end{equation}
	Assume that $\sigma_k \neq 0$ for all $k \ge 0$, and that
	$$
	\sum_{k=0}^{\infty} \sigma_k^2 < \infty, \quad \sum_{k=0}^{\infty} \varphi_k^2 < \infty.
	$$ 
	Then, for any $\delta > 0$, there exists a triple $(a, s, \rho)$ such that $a \in L^2 (\rho)$, $s \in [-1, 1]$, and $\Risk_{\mf,*} (\rho) \le \delta^2$.
\end{proposition}
This proposition is proved in Appendix \ref{sec:proof_risk_infimum}.
}

\modif{\begin{remark}
Proposition \ref{prop:risk_infimum} complements Proposition \ref{propo:Approx}
which establishes approximability of the target function $f_*(x) = \varphi(\<u_*,x\>)$
 using the networks \eqref{eq:First-NNET} (Proposition \ref{prop:risk_infimum} can be seen as an 
 $m=d=\infty$ version of the latter). 
 We note that the proofs of these  propositions also provides insight into
 the structure of approximators. Namely, we can take the weights $u_i$ to be i.i.d.~with distribution $\rho(u)\de u$ that is symmetric under rotations around $u_*$, and $a_i=\alpha(u_i)$,
 $a(u)=\alpha(u)\rho(u)$ is concentrated close to $\<u_*,u\>=0$ (on a scale 
 that can rely on the desired approximation error). 
 
 Indeed, the analysis of gradient flow in Section \ref{sec:matched} reveals that the solutions 
 found by gradient flow are of this nature. Namely, neurons develop a small but strictly 
 positive alignment with $u_*$. The distribution and size of the alignment evolves over time.
\end{remark}
}

\modif{\begin{remark}\label{rmk:MultiIndex}
The results in this section can be generalized to multi-index models: $y = \varphi(U_*^\top x)$ where $U_* \in O(d, k)$, the space of $d \times k$ orthogonal matrices. Further, the corresponding limiting dynamics become
	\begin{equation*}
	\begin{split}
		\veps\partial_t a(\omega) = \, & V(s(\omega)) - \int  a(\nu) U \left( s(\omega)^\top s(\nu) \right) \diff \rho(\nu) \, ,\\
		\partial_t s(\omega) = \, &  a(\omega) \left( I_k - s(\omega) s(\omega)^\top \right) \left( \nabla V (s(\omega)) -  \int  a(\nu) U' \left( s(\omega)^\top s(\nu) \right) s(\nu) \diff \rho(\nu) \right) \, .
	\end{split}
\end{equation*}
Here, $s(\omega) \in \R^k$ represents $U_*^\top u(\omega)$, and for $s \in \R^k$, $\norm{s}_2 \le 1$:
\begin{equation*}
	V(s) = \E \left[ \varphi(G) \sigma(G_s) \right], \quad (G, G_s)\sim \normal\left(0, \left[\begin{matrix} I_k & s\\ s^\top & 1\end{matrix}\right]\right).
\end{equation*}
The definition of $U$ is the same as before.
\end{remark}}

\section{Numerical solution}
\label{sec:Empirical}

\begin{figure}
	\begin{center}
		\begin{subfigure}{0.48\linewidth}
			\includegraphics[width=\linewidth]{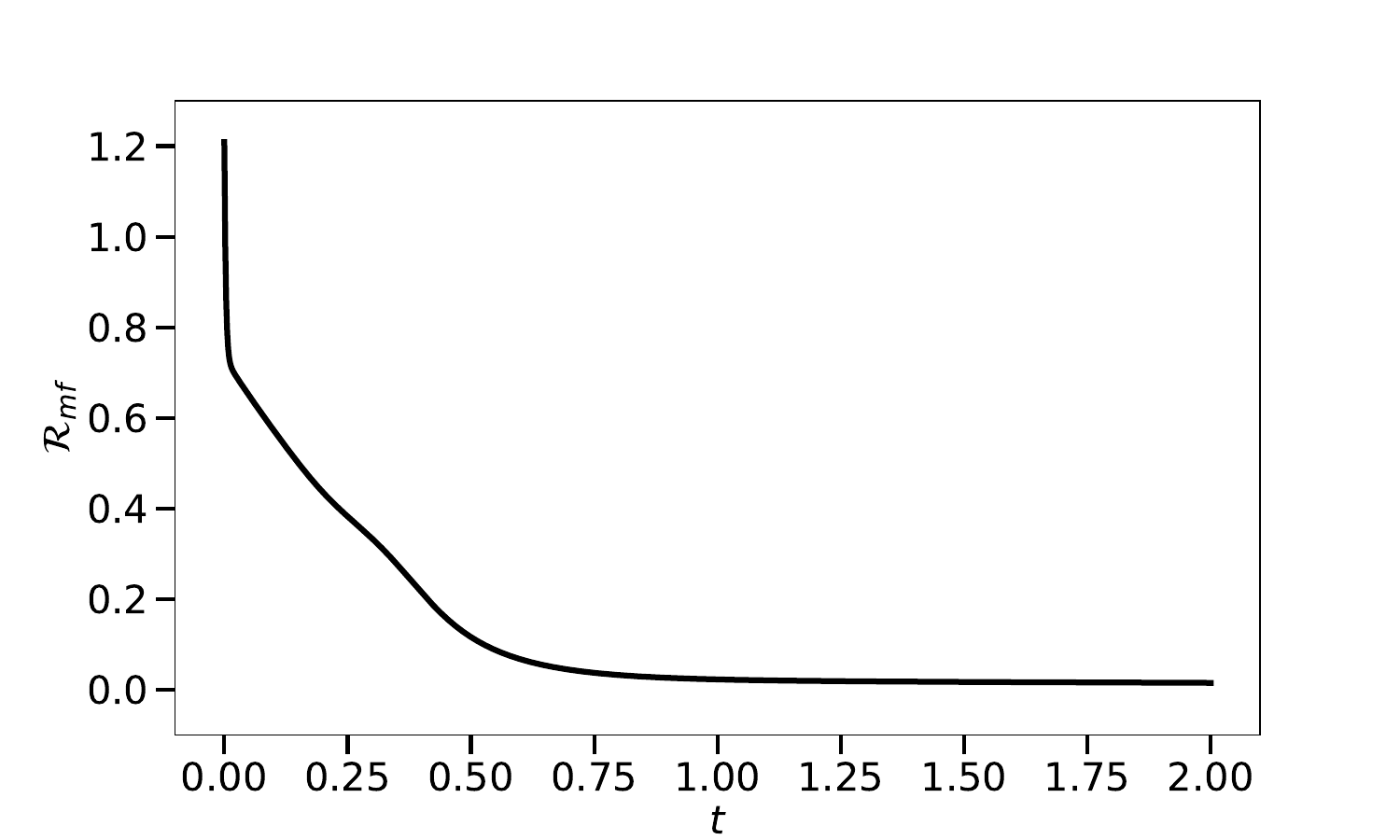}
			\includegraphics[width=\linewidth]{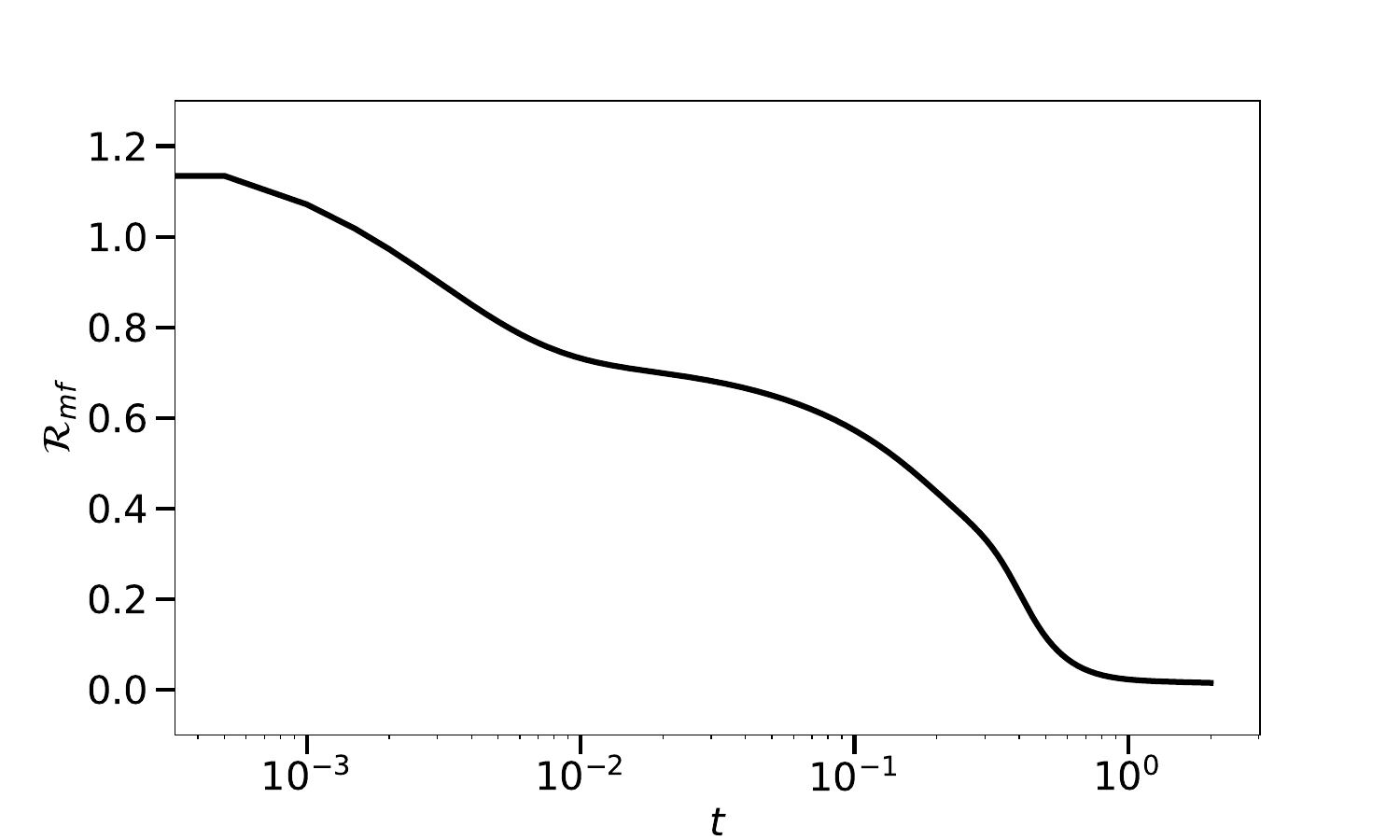}
			\includegraphics[width=\linewidth]{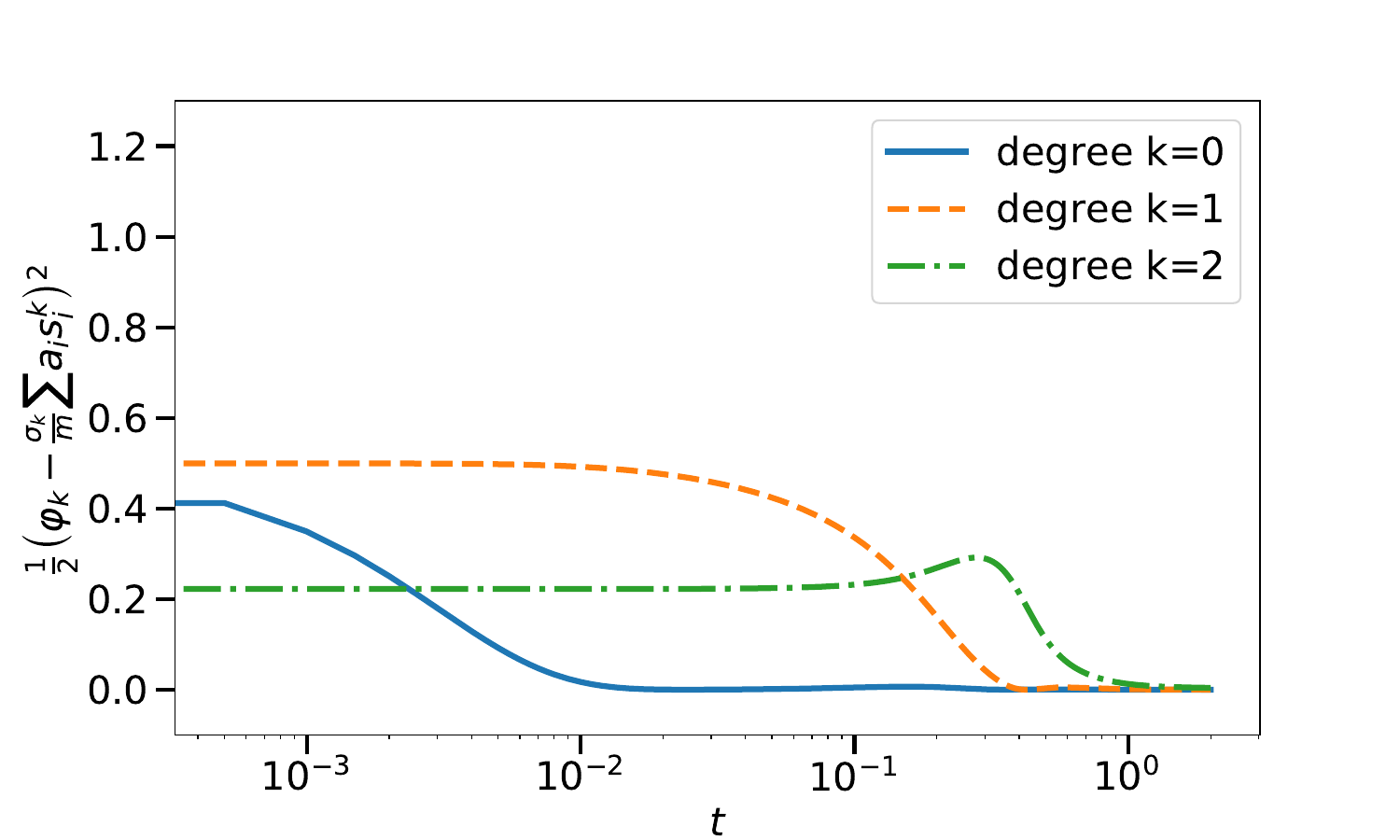}
			\caption{$\varepsilon = 10^{-3}$}
		\end{subfigure}
	\begin{subfigure}{0.48\linewidth}
		\includegraphics[width=\linewidth]{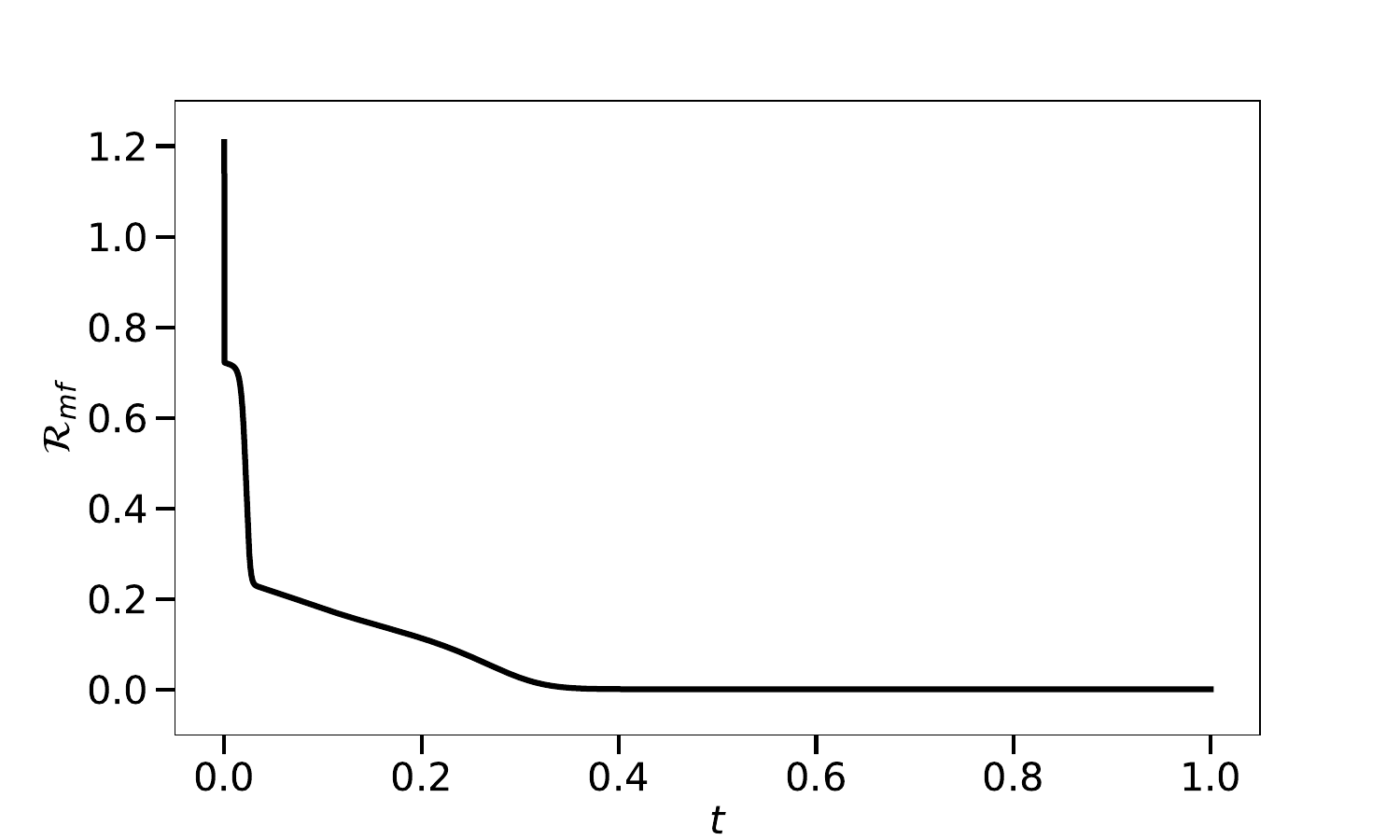}
	\includegraphics[width=\linewidth]{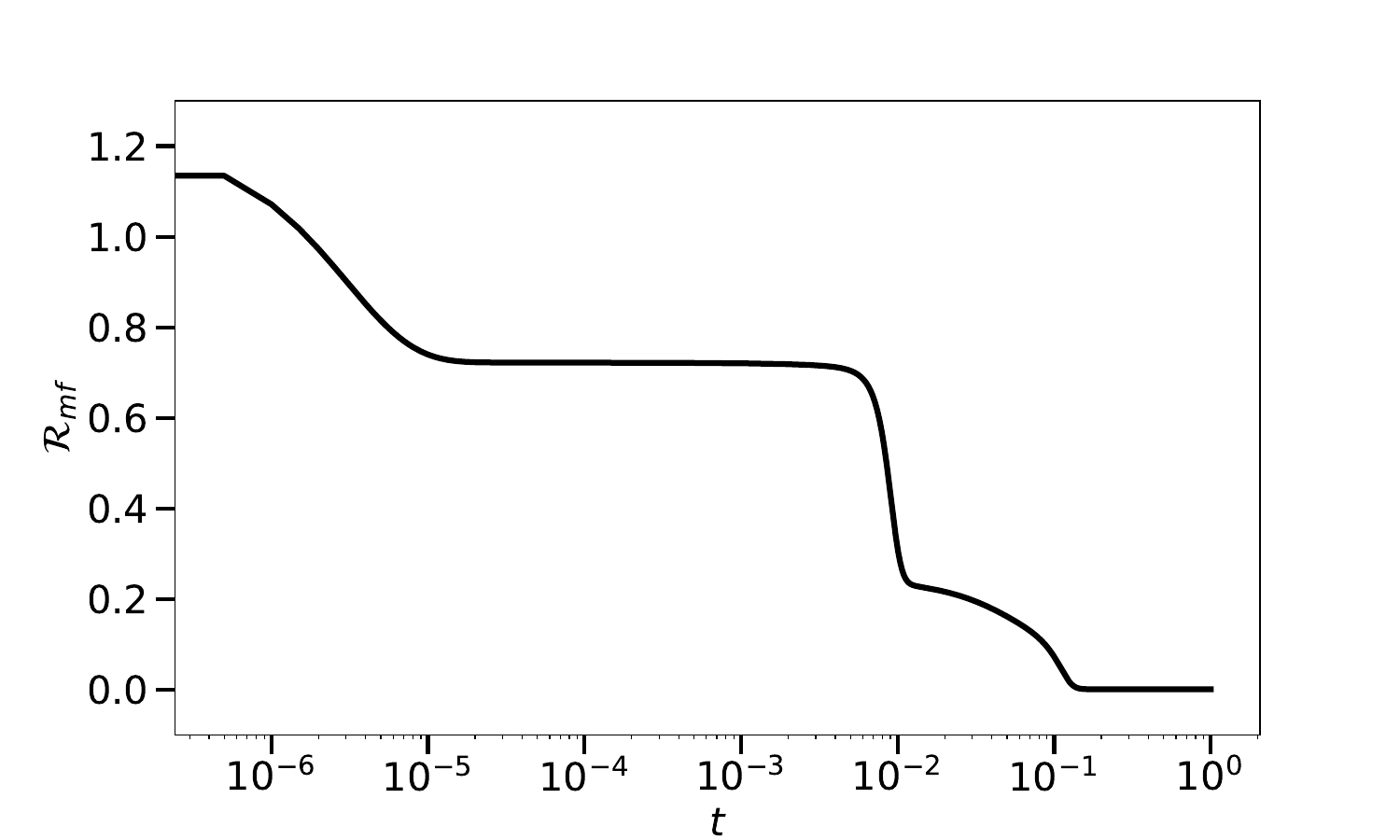}
		\includegraphics[width=\linewidth]{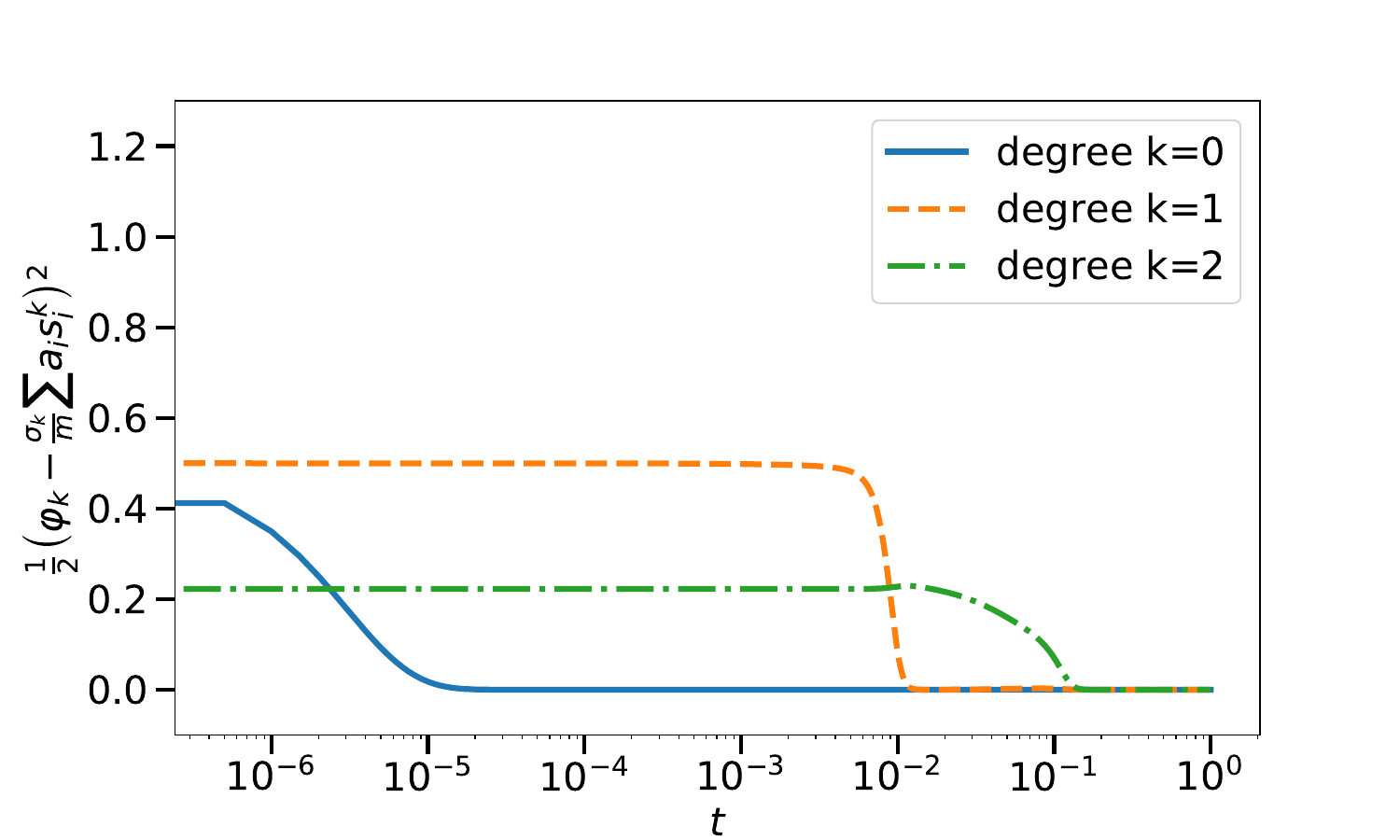}
		\caption{$\varepsilon = 10^{-6}$}
		\end{subfigure}
		\caption{Simulation of the \modif{simplified} neuron dynamics of Eqs.~\eqref{eq:coupled_MF}, with the
		 target function  of
		 Eq.~\eqref{eq:ExampleSimulations} and ReLU activations.
		 We use learning rate ratios $\varepsilon = 10^{-3}$ (left) and 
		 $\varepsilon = 10^{-6}$ (right) and we use $m=10$ neurons. 
		 First two rows: evolution of the risk $\Risk_{\mf}$ of
		 Eq.~\eqref{eq:RmfDef}, in linear and log-scales. Third row:
		 evolution of the first three terms of the sum of \eqref{eq:RiskHermite}.}
		\label{fig:simulations}
	\end{center}
\end{figure}

\begin{figure}
	\begin{center}
			\includegraphics[width=0.49\linewidth]{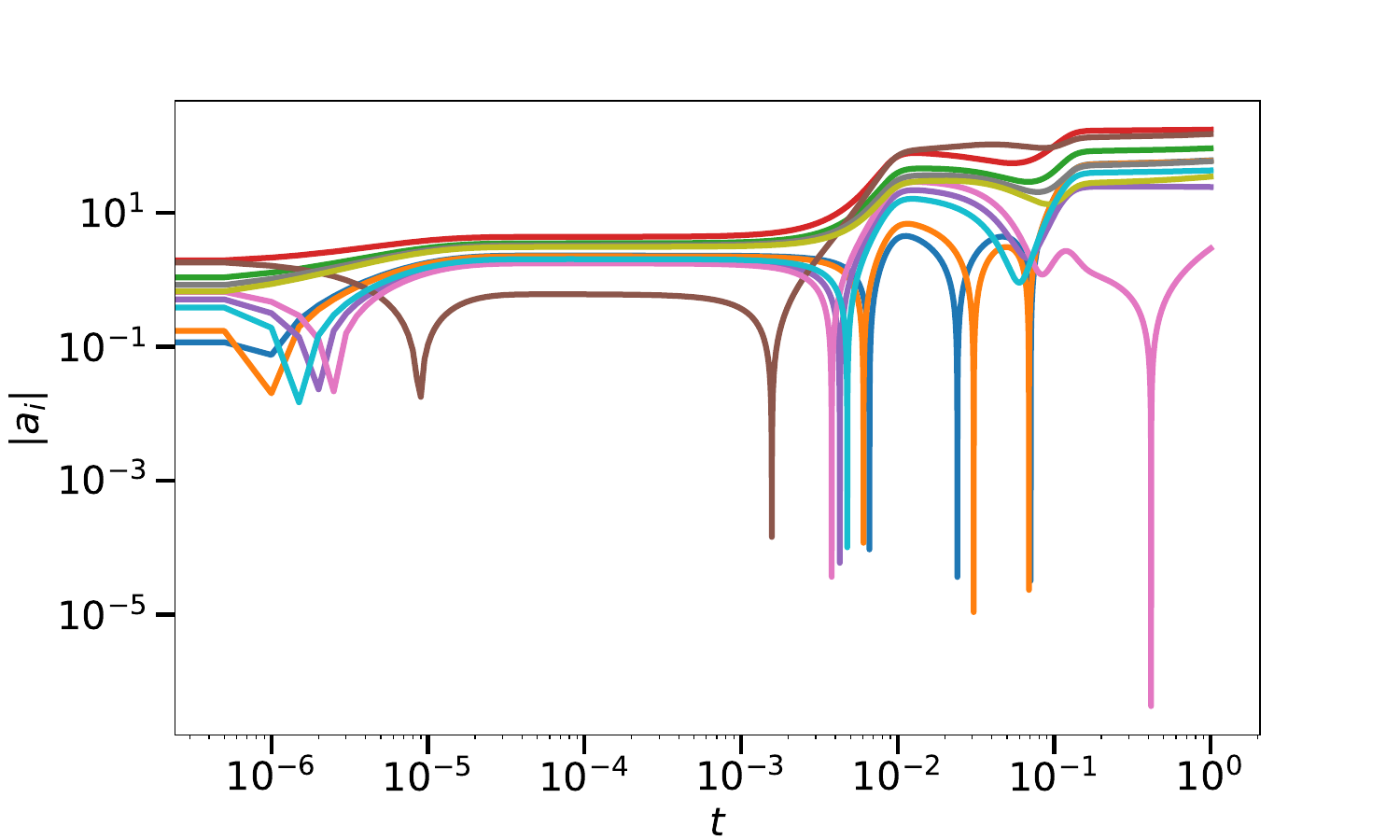}
			\includegraphics[width=0.49\linewidth]{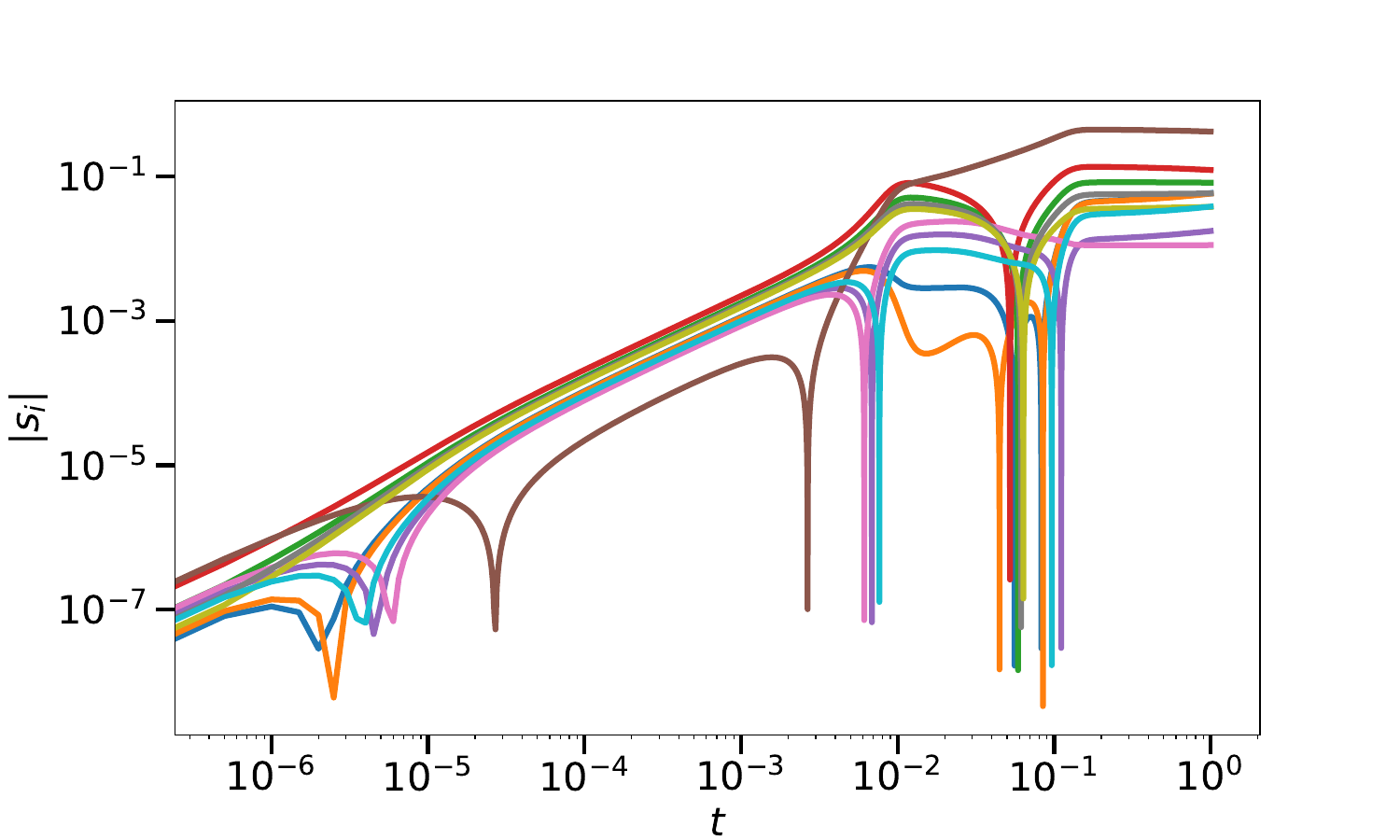}
		\caption{Same simulation as in Figure \ref{fig:simulations} (b). In these plots, we 
		show the evolution of the $a_i$ and the $s_i$ for $i\in\{1, \dots, m\}$ following 
		a discretization of Eqs.~\eqref{eq:coupled_MF}.}
		\label{fig:simulations-values}
	\end{center}
\end{figure}

\begin{figure}
	\begin{center}
		\begin{subfigure}{0.48\linewidth}
			\includegraphics[width=\linewidth]{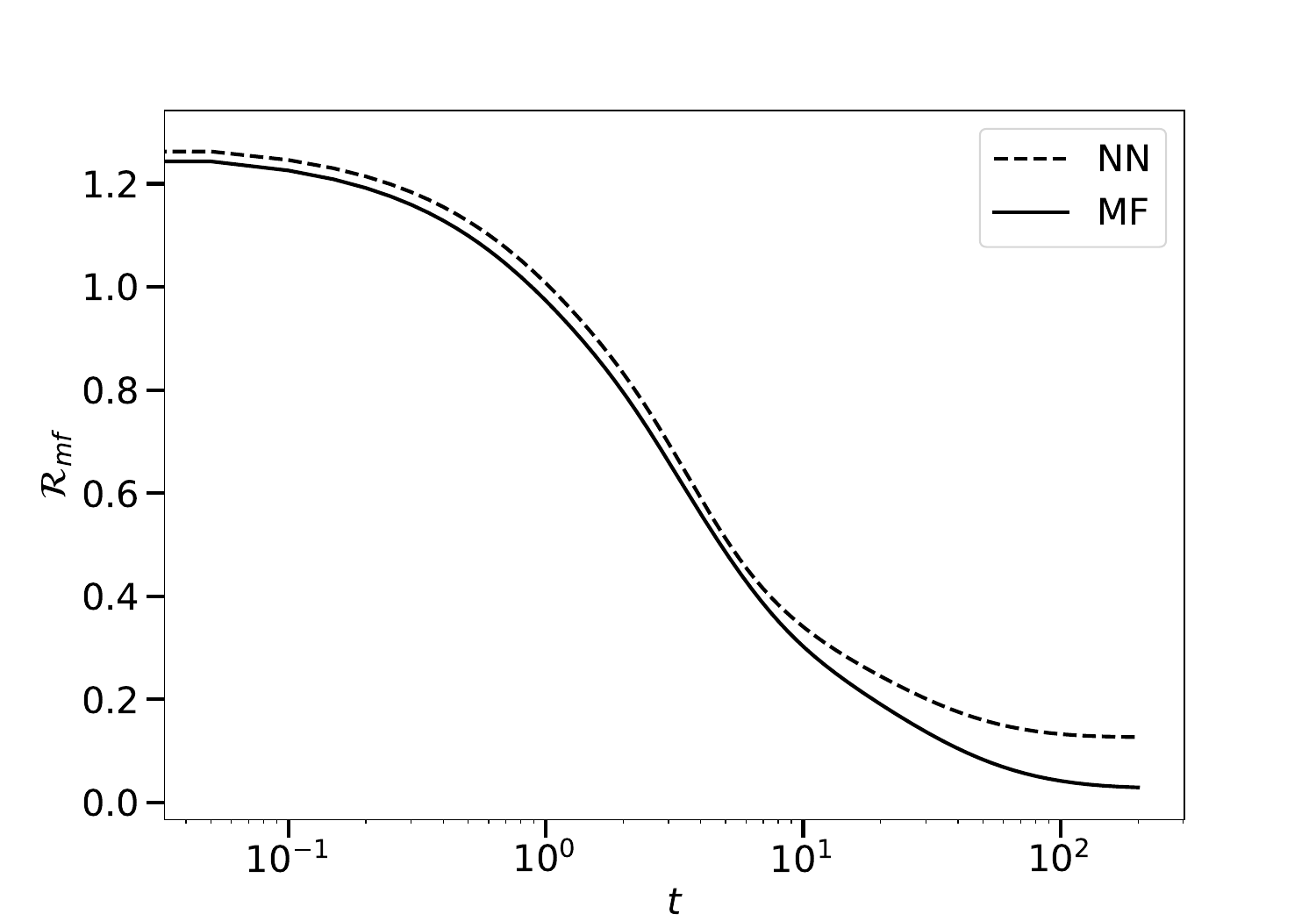}
			\includegraphics[width=\linewidth]{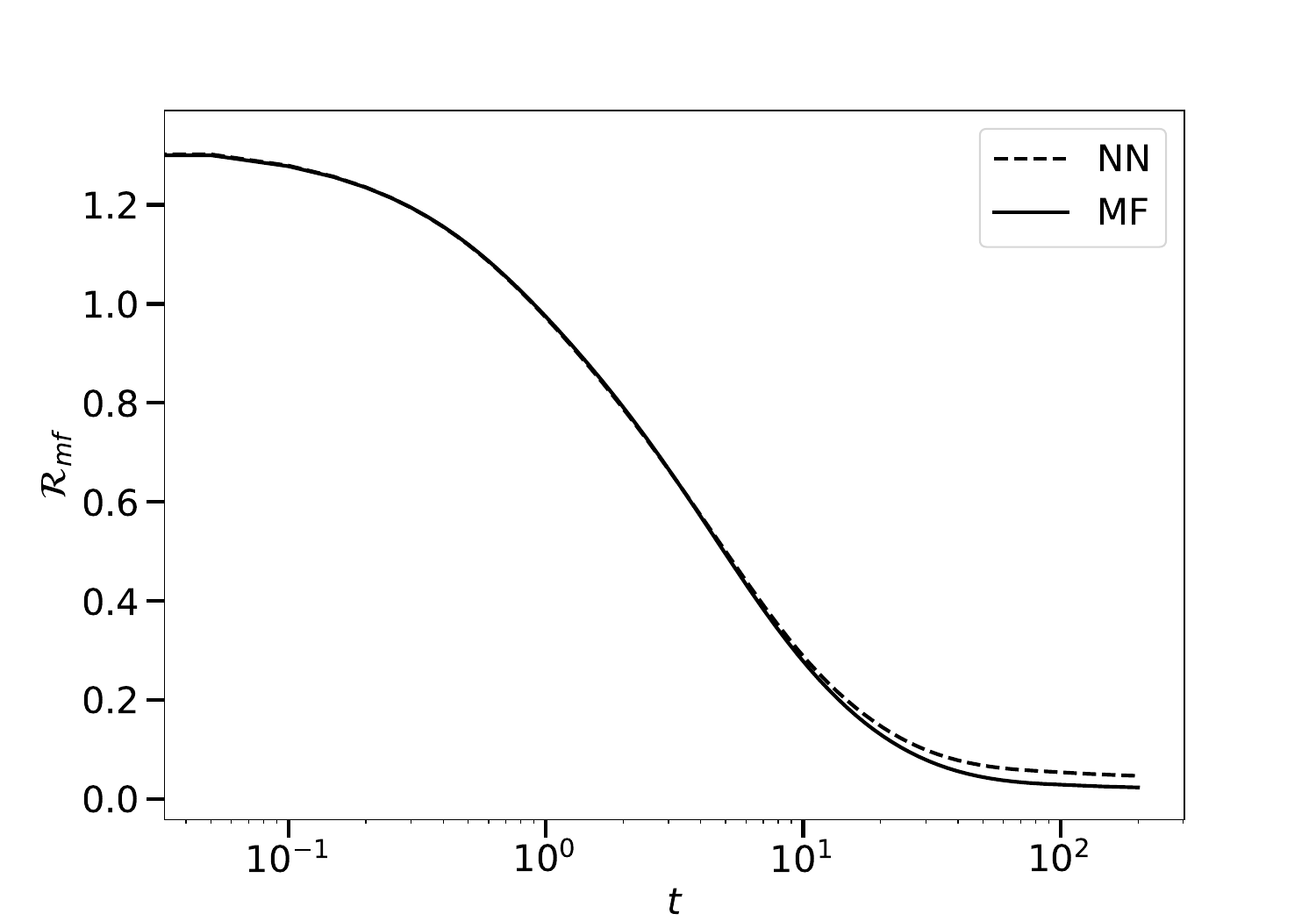}
			\includegraphics[width=\linewidth]{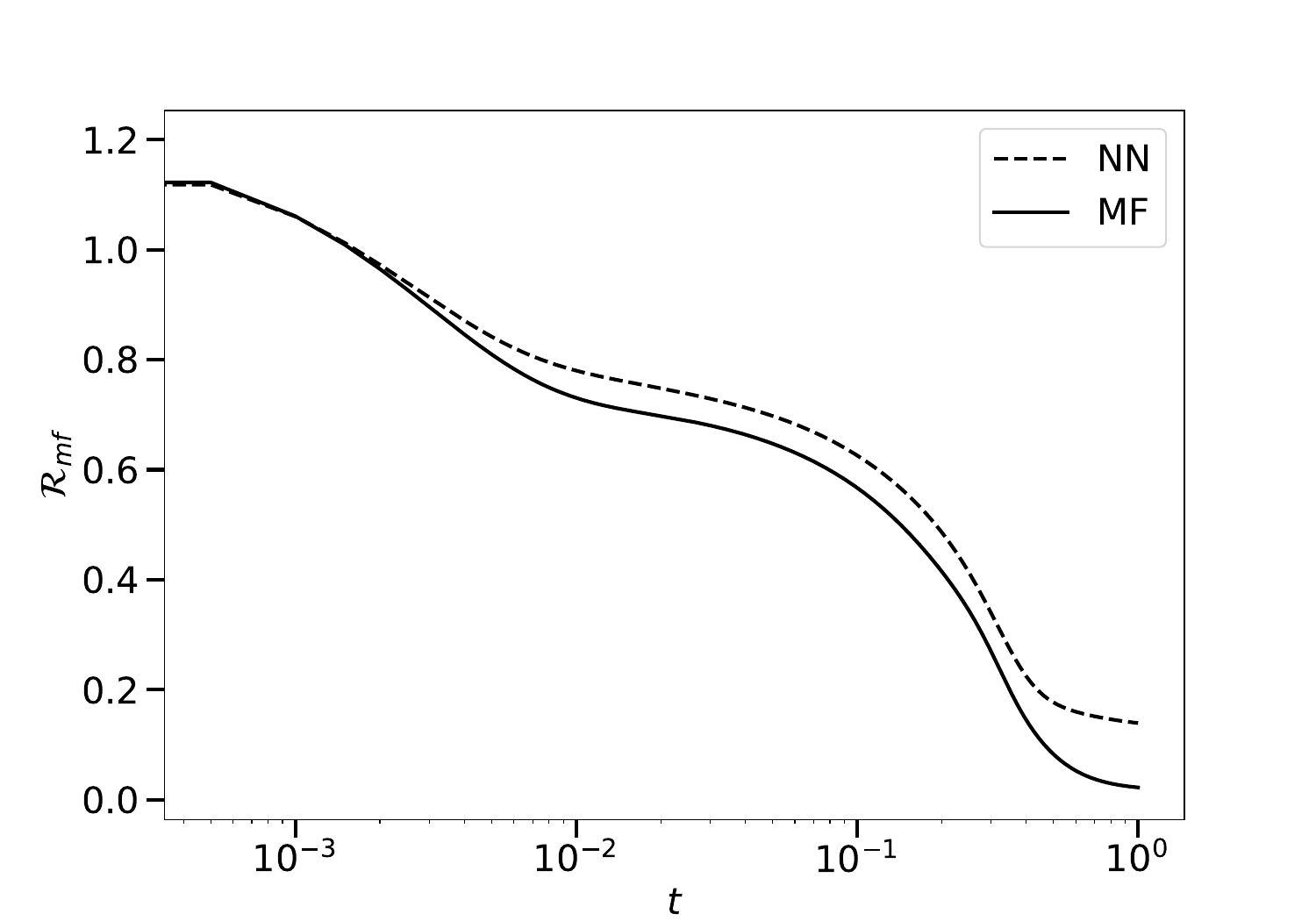}
			\includegraphics[width=\linewidth]{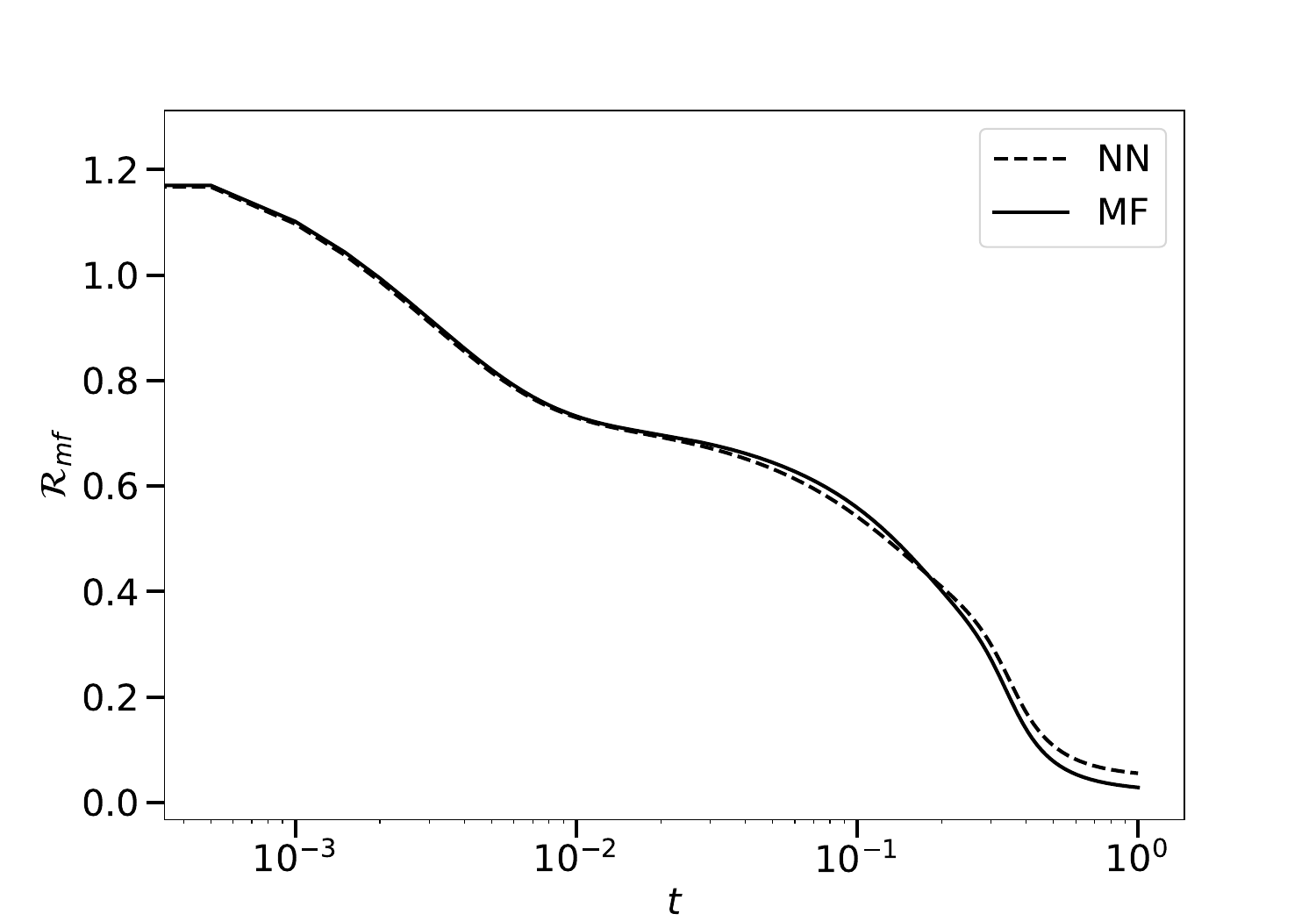}
		\end{subfigure}
		\begin{subfigure}{0.48\linewidth}
			\includegraphics[width=\linewidth]{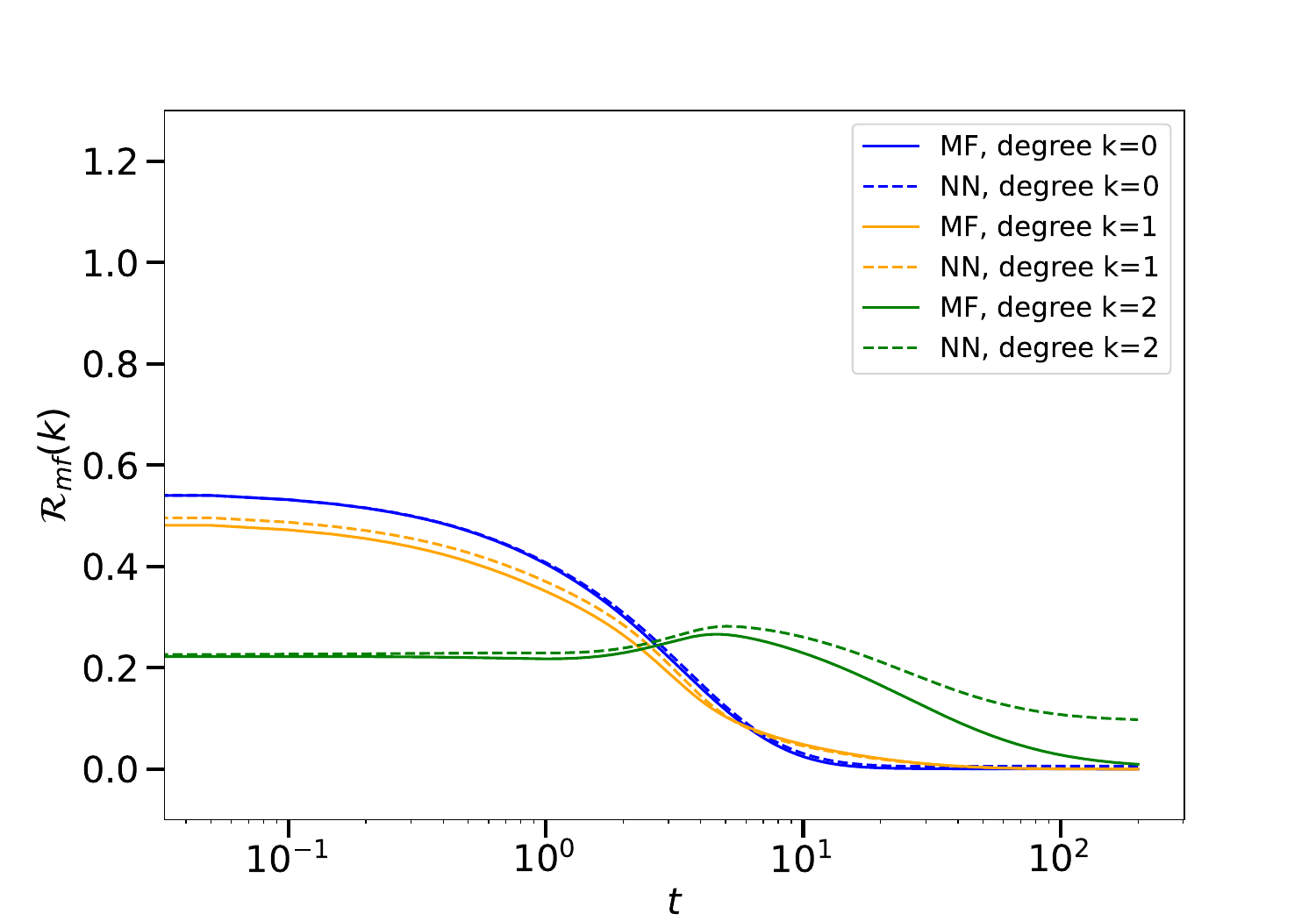}
			\includegraphics[width=\linewidth]{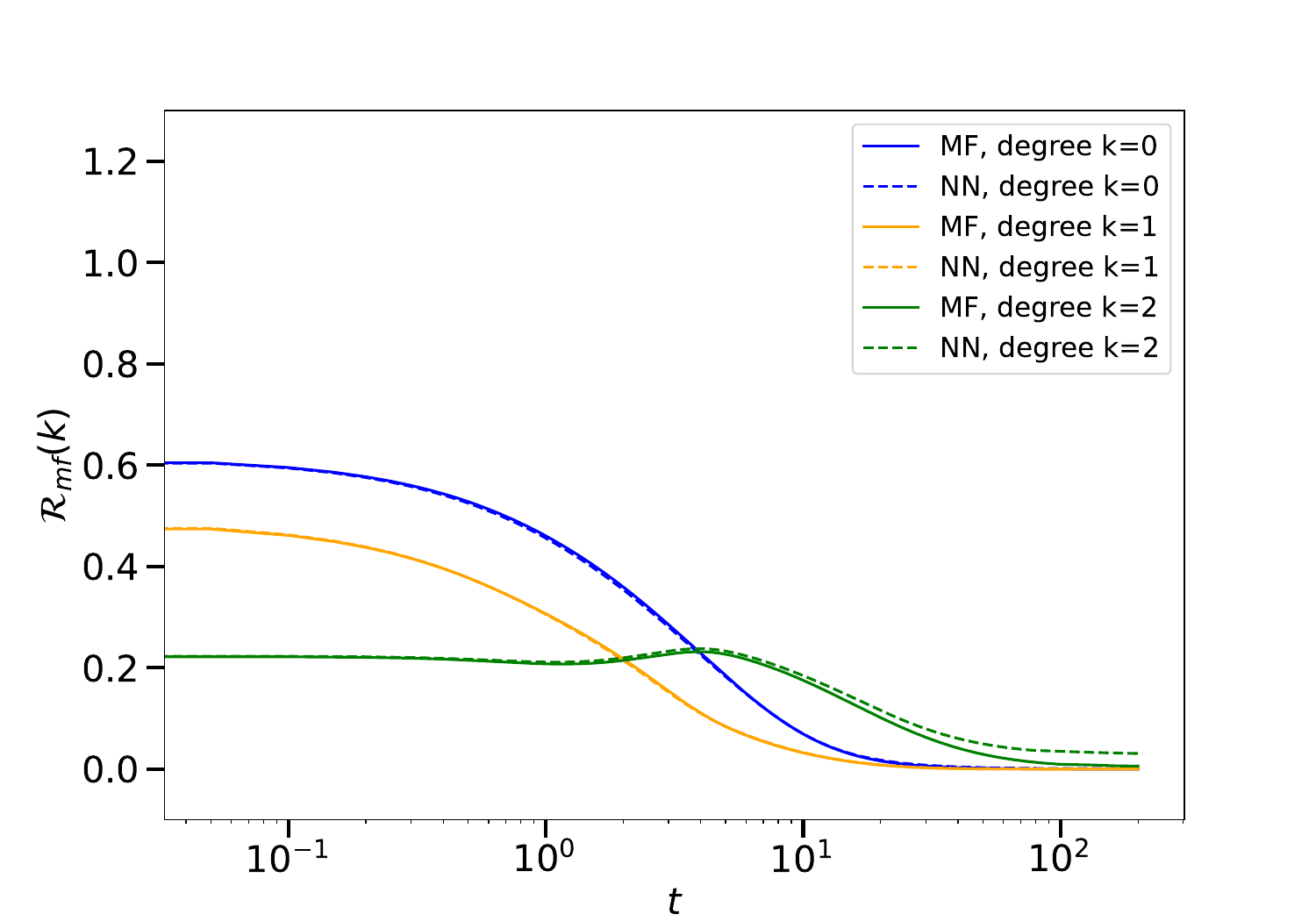}
			\includegraphics[width=\linewidth]{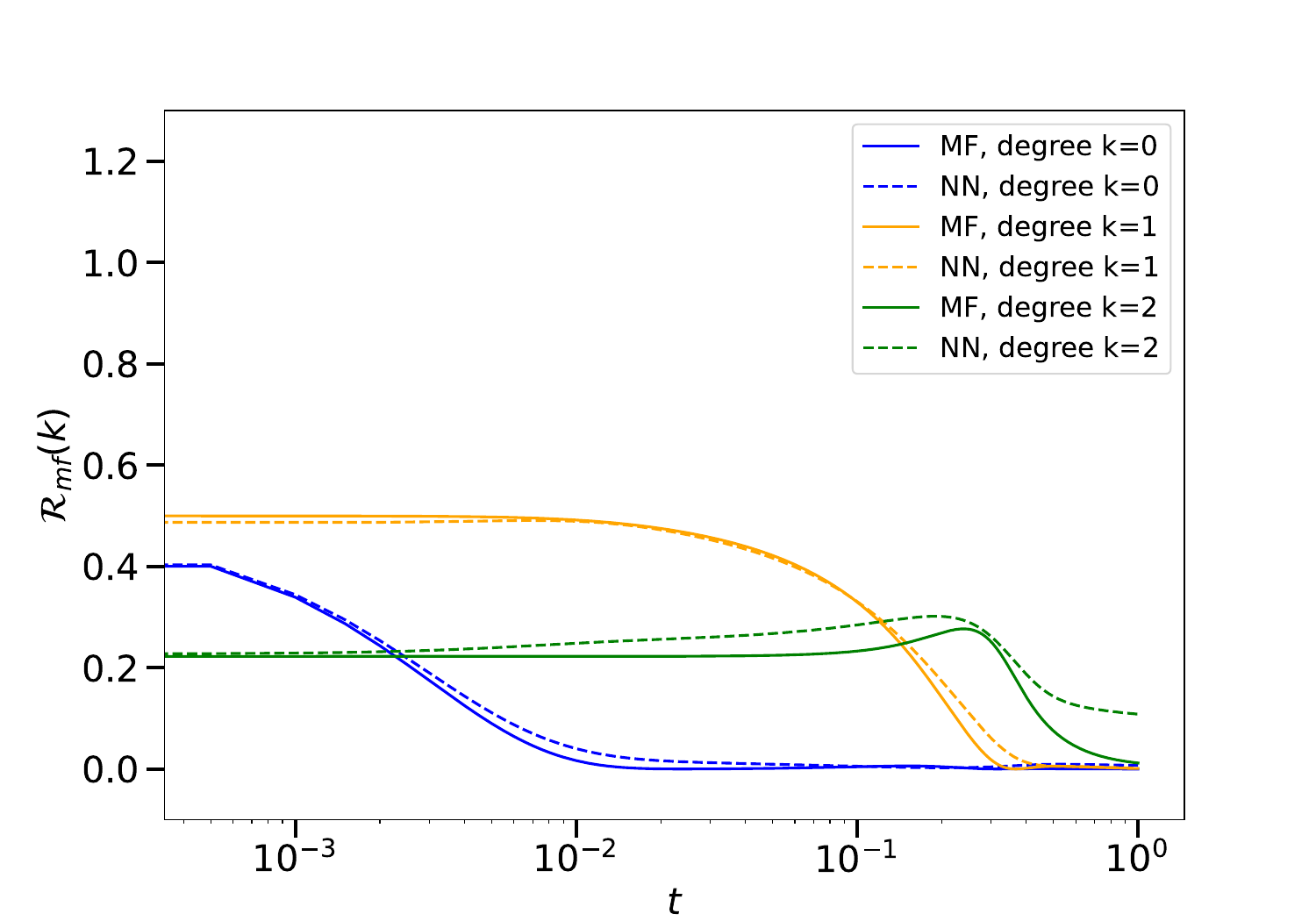}
			\includegraphics[width=\linewidth]{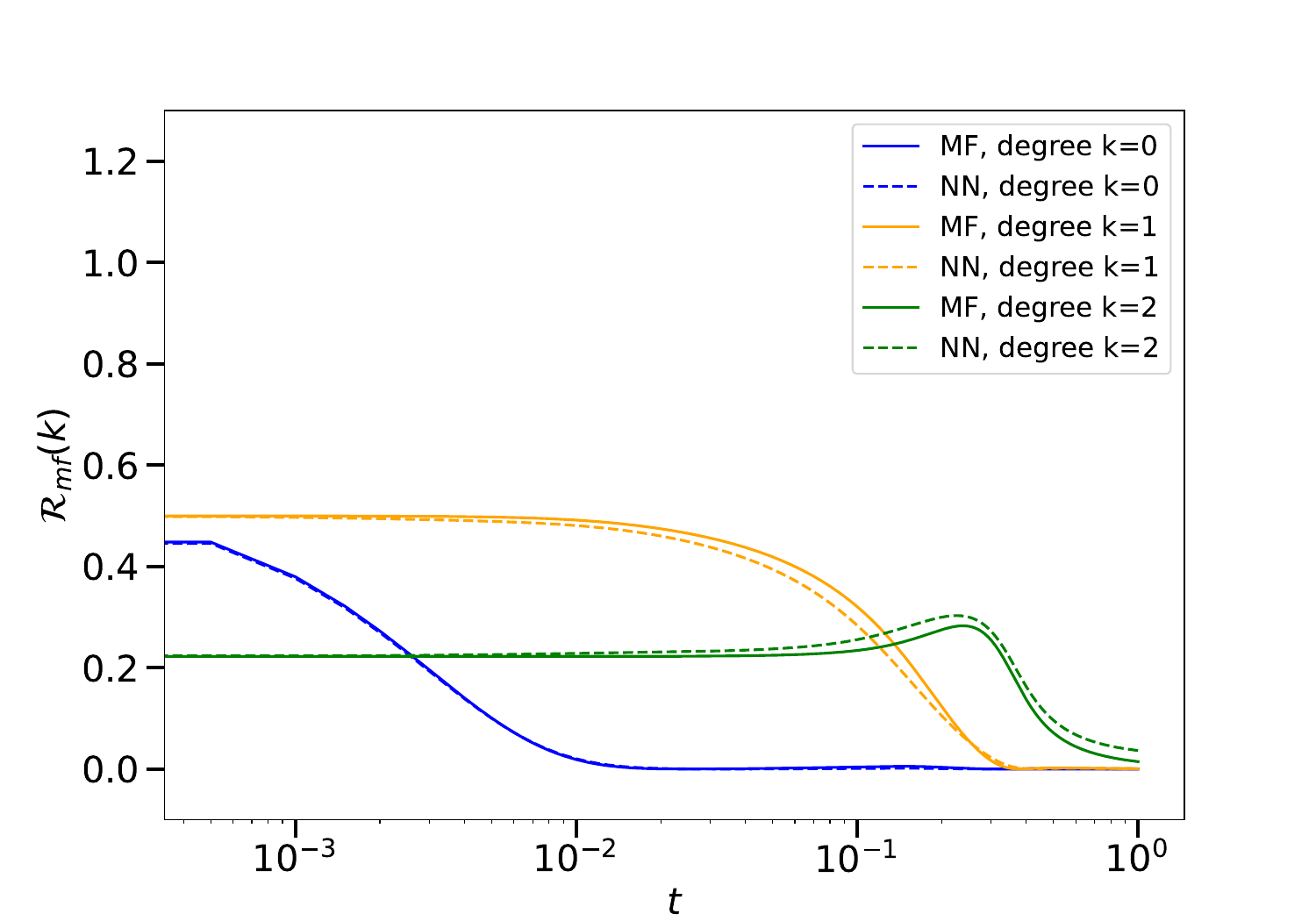}
		\end{subfigure}
		\caption{\modif{Comparison between the simplified  neuron dynamics \eqref{eq:coupled_MF} (MF) and projected gradient descent for the two-layer neural network \eqref{eq:First-NNET} (NN), with the same target function and activation as the simulations in Figure~\ref{fig:simulations}. We use four different combinations of (learning rare ratio, network width):  $(\veps, m) = (1, 10)$ (first row), $(\veps, m) = (1, 50)$ (second row), $(\veps, m) = (10^{-3}, 10)$ (third row), and $(\veps, m) = (10^{-3}, 50)$ (fourth row). Left panel: evolution of the risk $\Risk_{\mf}$ for NN and MF on a logarithmic scale. Right panel: evolution of the first three components of $\Risk_{\mf}$ (constant, linear, and quadratic) for NN and MF.}}
		\label{fig:compare_nn_mf}
	\end{center}
\end{figure}

In Figure \ref{fig:simulations}, we present the result 
of an Euler discretization of Eqs.~\eqref{eq:coupled_MF}
where $\varphi$ is a degree-$2$ polynomial and $\sigma$ is the ReLU activation:
$\sigma(s) = \max(s,0)$,
\begin{align}
	\begin{split}
&\varphi(s) = \mathrm{He}_0(s) -  \mathrm{He}_1(s) - \frac{2}{3}\mathrm{He}_2(s) \, \\
 &\hspace{7mm}= \left(1-\frac{2\sqrt{2}}{6}\right) - s - \frac{2\sqrt{2}}{6}s^2 \, . 
 \label{eq:ExampleSimulations}
 \end{split}
\end{align}
These plots clearly display two of the features emphasized in the introduction:
$(i)$~plateaus separated by periods of  rapid improvement of the risk;
$(ii)$~increasingly long timescales (notice the logarithmic time axis in the second and third row).

In order to examine the incremental learning structure,
we rewrite the risk $\Risk_{\mf}$ of Eq.~\eqref{eq:RmfDef} 
by decomposing $\varphi$ and $\sigma$ in the basis of Hermite polynomials
\begin{align}
\Risk_{\mf}(a,s)
&=\frac{1}{2} \sum_{k \geq 0} \left( \varphi_k - 
 \frac{\sigma_k}{m} \sum_{i = 1}^{m} a_is_i^k \right)^2 \, .\label{eq:RiskHermite}
\end{align}
We observe that, for small  $\veps$, the Hermite coefficients of $\varphi$
 are learned sequentially, in the order of their degree. When $\veps$ is 
 sufficiently small (right plots), this incremental learning happens in well separated phases. 
 The plateaus and waterfalls in the  plots of $\Risk_{\mf}$ 
 correspond to the network learning increasingly higher degree polynomials.
 
In Figure \ref{fig:simulations-values} we plot the evolution of the values of the $a_i$
and $s_i$, for $i\in \{1,\dots,m\}$. 
We observe that the \modif{overall} order 
 of magnitude of the $a_i$'s and the $s_i$'s increases 
 when passing through the different phases of the incremental learning process. \modif{In the mean time, some of the $a_i$'s and $s_i$'s will undergo a sign change during the learning process, which is characterized by a sudden decrease and subsequent rapid increase in its magnitude.}
 
Altogether, the results of Figures \ref{fig:simulations}  and \ref{fig:simulations-values} 
are consistent with the \modif{canonical learning order} up to level $L=2$
as per Definition \ref{def:StandardScenario}.
While we conjecture that incremental learning also occurs for higher-order polynomials, we found this hard to observe in numerical simulations\modif{: we would need to take $\varepsilon$ much smaller than in Figure 2, resulting in prohibitively large simulation costs}.  

First, as predicted in Definition \ref{def:StandardScenario}, the times at which the 
components are learned are closer on a logarithmic scale as the degree increases.
It is therefore increasingly difficult to observe time scales corresponding to 
higher degrees.

Second, we expect there to be a choice of the initialization $(a_{i, \rm{init}}, u_{i, \rm{init}})_{i \in [m]}$,
activation and target function, for which not all the components of $\varphi$ are actually learnt.
We observed empirically that this happens easily for small $m$.

\modif{To conclude this section, in Figure~\ref{fig:compare_nn_mf} we compare the simplified neuron dynamics (MF) of Eq.~\eqref{eq:coupled_MF} and the evolution of projected gradient descent for the original two-layer neural network (NN). From the plots we observe two remarkable phenomena: (1) the evolution of the risk for NN and MF are close to each other during the entire learning process for both large learning rate ratio ($\veps = 1$) and small ($\veps = 10^{-3}$), and their risk curves have the same qualitative behavior even if $m$ is small ($m = 10$); (2) as we increase the value of $m$ from $10$ to $50$, the alignment between the learning curves of NN and MF improves significantly. These observations justify our argument in Section~\ref{sec:reduction-2} that the inter-neuron correlations $r_{ij}$ are well approximated by $s_i s_j$ for wide networks. }

\section{Timescales hierarchy in the gradient flow dynamics}
\label{sec:matched}

We are interested in the behavior of the solution of the ODEs~\eqref{eq:general-dynamics}, 
initialized from $s(\omega,0) = 0$ for all $\omega$ (as per Proposition 
\ref{prop:coupling_W2_bd}). The \modif{canonical learning order} of Definition \ref{def:StandardScenario} 
concerns the behavior of solutions for $\veps\to 0$. This type of questions can be addressed within the theory of
dynamical systems  using 
\emph{singular perturbation theory}  \citep{holmes2019perturbation}. 
Here, `singular' refers to the fact that $\varepsilon$ 
multiplies one of the highest-order derivatives. \modif{In Eq.~\eqref{eq:general-dynamics}, 
$\veps$ multiplies the differential term $\partial_t a(\omega)$, so that the ODE system becomes singular in the limit $\veps \to 0$. In particular, it degenerates to the following system of differential-algebraic equations:
\begin{equation}
	\begin{split}
		V(s(\omega)) = \, & \int a(\nu) U(s(\omega) s(\nu)) \diff \rho(\nu) \, ,\\
		\partial_t s(\omega) = \, &  a(\omega) \left( 1 - s(\omega)^2 \right) \left( V'(s(\omega)) -  \int a(\nu) U'(s(\omega) s(\nu)) s(\nu) \diff \rho(\nu) \right) \, .
	\end{split}
\end{equation}
Due to singularity, the qualitative behavior of the above system is dramatically different 
from that of Eq.~\eqref{eq:general-dynamics} with $\varepsilon$ small but non-zero. 
This is in stark contrast to regular perturbation problems, for which the limiting dynamics 
 will still be a system of differential equations with the same order and 
similar qualitative behavior as the perturbed system.
}

As a side remark, we note that the system \eqref{eq:general-dynamics} can be seen as a 
slow-fast dynamical system, where the $a(\omega)$'s are the fast variables and the $s(\omega)$'s are the
 slow variables \citep{berglund2001perturbation}. Formally, the time derivative of the $a(\omega)$'s
 is multiplied by a factor $(1/\eps)$.
 From a dynamical systems perspective, the present case is made complicated because
 of a bifurcation when the $s(\omega)$'s become non-zero. 

The \modif{canonical learning order} provides a detailed description of this bifurcation.
We will motivate this scenario using a classical, but non-rigorous, 
technique of singular perturbation theory, called the \emph{matched asymptotic expansion}
 \cite[Chapter 2]{holmes2019perturbation}.
This technique decomposes the approximation of the solution in several time 
   scales on which a regular approximation holds. These time scales are traditionally called
    \emph{layers} in the literature; however, we avoid this terminology due to the potential 
    confusion with the layers of the neural network.
    
    We will work mainly using the Hermite representation of the dynamical
     ODEs \eqref{eq:general-dynamics}, which we write down 
     for the reader's convenience:
\begin{equation}
\label{eq:group_sing_hermite}
\begin{split}
\veps \partial_{t} a(\omega) &= \,  \sum_{k=0}^{\infty} \sigma_k s(\omega)^k \left( \varphi_k - \sigma_k \int a(\nu) s(\nu)^k \diff\rho(\nu) \right) \, , \\
\partial_{t} s(\omega) &= \, a(\omega) \left( 1 - s(\omega)^2 \right) \sum_{k=1}^{\infty} k \sigma_k s(\omega)^{k - 1} \left( \varphi_k - \sigma_k \int a(\nu) s(\nu)^k \diff \rho(\nu) \right) \, .
\end{split}
\end{equation}

 \modif{The rest of this section is organized as follows. We first give a brief overview of the method of matched asymptotic expansions and a summary of our main results regarding the learning timescales in Section~\ref{sec:overview_singular_matched}.} Sections \ref{sec:first-layer}-\ref{sec:third-layer} respectively describe the first
 three time scales of the matched asymptotic expansion of \eqref{eq:group_sing_hermite}. This gives, for each time scale, an approximation of the $a(\omega)$, $s(\omega)$. In Appendix~\ref{sec:induced}, we detail how these sections induce an evolution of the risk alternating plateaus and rapid decreases, and support the standing learning scenario of Definition \ref{def:StandardScenario}. Finally, in Section \ref{sec:conjectured}, we 
  conjecture the behavior on longer time scales.

\paragraph{Notations.} We denote $\bfone$ the constant function $\bfone:\omega \in \Omega \mapsto 1 \in \R$. Denote $\langle ., . \rangle_{L^2(\rho)}$ the dot product on $L^2(\rho)$ and $\Vert . \Vert_{L^2(\rho)}$ the associated norm. For $x \in L^2(\rho)$, we denote $x_\perp$ the orthogonal projection of $x$ on the hyperplane $\bfone^\perp$ of $L^2(\rho)$ of functions orthogonal to $\bfone$:
\begin{equation*}
	x_\perp(\omega) = x(\omega) - \int x(\nu) \diff \rho (\nu ) \, .
\end{equation*}
We denote $a_{\text{init}}(\omega) = a(\omega,0)$ and thus $a_{\perp,\text{\init}}$ is the orthogonal projection of $a_{\text{init}}$ on $\bfone^\perp$.



\modif{\subsection{Matched asymptotic expansions}\label{sec:overview_singular_matched}
	
The method of matched asymptotic expansions is a common approach to finding approximate 
solutions of perturbed differential equations. In the present paper, we are mainly interested 
in applying this technique to approximate the solution to the specific singularly perturbed ODE 
system\footnote{Although we keep calling this an ODE system, it is important
to keep in mind that it takes place in an infinite-dimensional space.} 
of Eq.~\eqref{eq:group_sing_hermite}. 
Denoting by $t$ the independent variable and by $\varepsilon$ the perturbation parameter, 
the method of matched asymptotic expansions consists of the following three steps: 
(1)~Divide the domain of $t$ (generally a subinterval of $\R$) to several subdomains, 
which may overlap each other and depend on the perturbation parameter $\veps$; (2)~Within 
each subdomain, find an accurate approximation to the perturbed system. This is usually 
achieved by expanding the perturbed system in powers of $\veps$, and keeping only terms 
that are relevant to the current domain; (3) The approximate solutions obtained in Step 
(2) might not be valid in the overlap of two adjacent subdomains. To resolve this issue, these
 approximate solutions are then combined together through a process called ``matching'' to 
 produce an approximation that is valid on the entire domain.

In our setting, the singularly perturbed system~\eqref{eq:group_sing_hermite} takes the form of
\begin{align*}
	\veps \partial_t a(\omega) = \, & f(a(\omega), s(\omega)), \\
	\partial_t s(\omega) = \, & g(a(\omega), s(\omega)).
\end{align*}
We will carry out explicit calculations for the first three time scales in Sections 
\ref{sec:first-layer}-\ref{sec:third-layer}, respectively. Here is a summary of our main findings:
\begin{itemize}
	\item In Section~\ref{sec:first-layer} we explore the learning of the constant component 
	of the target function, which happens at the timescale $t = \Theta(\veps)$. At the end of this 
	phase, the mean-field risk (see \eqref{eq:risk-general}) evolves to
	\begin{equation}
		\Risk_{\mf,*} = \frac{1}{2 }\sum_{k\geq 1} \varphi_k^2 + O(\varepsilon) \, .
	\end{equation}
	In other words, during this phase, gradient flow learns the constant term in $\varphi$.
	At the end of this time scale we have $a(\omega)= \Theta(1)$ and  $s(\omega)= \Theta(\eps)$.
	\item Then, in Section~\ref{sec:second-layer} we investigate the second time scale 
	$t = t_2\eps^{1/2}$, $t_2 \le c\log(1/\veps)$, during which the $a(\omega)$'s and $s(\omega)$'s increase 
	to a different order in $\eps$. The result of this time scale is mainly technical and needed to understand the transition to the time scale of Section \ref{sec:third-layer}. We also perform the matching procedure to combine the approximate solution within this time scale to the one obtained in Section~\ref{sec:first-layer}.
At the end of this time scale we have $a(\omega)= \Theta(e^{c't_2})$ and  $s(\omega)= \Theta(e^{c't_2}\eps^{1/2})$.
	\item To understand the evolution of the risk relevant to learning the linear component
	of $\varphi$, we introduce a new time scale $t =  \frac{1}{4 |\varphi_1 \sigma_1|} \veps^{1/2} \log \frac{1}{\varepsilon} + \Theta(\veps^{1/2})$ in Section~\ref{sec:third-layer}, and show that the linear component can be learned within this time scale. To be more accurate, at the end of this time scale we have
	\begin{equation}
		\Risk_{\mf,*} = \frac{1}{2 }\sum_{k\geq 2} \varphi_k^2 + O(\varepsilon^{\nicefrac{1}{4}}) \, ,
	\end{equation} 
	and $a(\omega)= \Theta(\eps^{-1/4})$ and  
	$s(\omega)= \Theta(\eps^{1/4})$.
\end{itemize}
Finally, in Section~\ref{sec:conjectured}, we conjecture the behavior of the approximate solutions and induced risks for longer time scales.}
\subsection{First time scale: constant component}
\label{sec:first-layer}

We define a ``fast'' time variable $t_1 = t/\varepsilon$ and replace it in 
Eq.~\eqref{eq:group_sing_hermite}.
We expand the solutions $a(\omega)$ and $s(\omega)$ in powers of $\varepsilon$:
\begin{align}
a(\omega) &= a^{(0)}(\omega) +\varepsilon a^{(1)}(\omega)  + \varepsilon^2 a^{(2)}(\omega)  + \dots \, , \label{eq:expansion-a-layer1}\\
s(\omega) &= s^{(0)}(\omega) + \varepsilon s^{(1)}(\omega)  + \varepsilon^2s^{(2)}(\omega) + \dots \label{eq:expansion-s-layer1}\, , 
\end{align}
where $a^{(0)}(\omega), a^{(1)}(\omega), a^{(2)}(\omega), \dots, s^{(0)}(\omega), s^{(1)}(\omega), s^{(2)}(\omega), \dots$ are implicitly functions
 of $t_1$. They are initialized at 
\begin{align}
&a^{(0)}(\omega,t_1 = 0) = a_{\rm{init}}(\omega) \, , &&a^{(1)}(\omega, t_1 = 0) = 0 \, , &&a^{(2)}(\omega, t_1 = 0) = 0 \, , &&\dots   \label{eq:init-a-layer1}\\
&s^{(0)}(\omega, t_1 = 0) = 0 \, , &&s^{(1)}(\omega, t_1 = 0) = 0 \, , &&s^{(2)}(\omega, t_1 = 0) = 0 \, , &&\dots \label{eq:init-s-layer1}
\end{align}
to be consistent with the initial condition $a(\omega, t_1=0) = a(\omega, t=0) = a_{\rm{init}}(\omega)$ and $s(\omega,t_1=0) = s(\omega,t=0) = 0$.

We substitute the expansion in \eqref{eq:group_sing_hermite}:
\begin{align}
&\partial_{t_1}a^{(0)}(\omega) +\varepsilon \partial_{t_1}a^{(1)}(\omega)  + 
 \dots  \label{eq:substitution-1-layer1}\\
 &\quad=  \sum_{k=0}^{\infty} \sigma_k \left(s^{(0)}(\omega) + \varepsilon s^{(1)}(\omega)  + \dots\right)^k\\
 &\quad\qquad\times\left( \varphi_k - {\sigma_k} \int \left(a^{(0)}(\nu) +\varepsilon a^{(1)}(\nu)  +  \dots\right) \left(s^{(0)}(\nu) + \varepsilon s^{(1)}(\nu)  + \dots\right)^k \diff \rho(\nu) \right) \, ,  \\
&\partial_{t_1}s^{(0)}(\omega) +\varepsilon \partial_{t_1}s^{(1)}(\omega)  + 
\dots  \label{eq:substitution-3-layer1}\\
&\quad=  \varepsilon \left(a^{(0)}(\omega) +\varepsilon a^{(1)}(\omega)  + \dots\right) \left( 1 - \left(s^{(0)}(\omega) + \varepsilon s^{(1)}(\omega)  + \dots\right)^2 \right) \\
& \quad\qquad \times \sum_{k=1}^{\infty} k \sigma_k \left(s^{(0)}(\omega) + \varepsilon s^{(1)}(\omega)  + \dots\right)^{k - 1}  \\
&\quad\qquad\times\left( \varphi_k - {\sigma_k} \int  \left(a^{(0)}(\nu) +\varepsilon a^{(1)}(\nu)  +  \dots\right) \left(s^{(0)}(\nu) + \varepsilon s^{(1)}(\nu)  + \dots\right)^k \diff \rho (\nu) \right) \, . \label{eq:substitution-4-layer1}
\end{align}

The basic assumption of matched asymptotic expansions is that terms of the same order in $\varepsilon$ can be identified (with some limitations that we develop below). For now, let us identify terms of order $1 = \varepsilon^0$:
\begin{align}
\partial_{t_1} a^{(0)}(\omega) &= \sum_{k=0}^{\infty} \sigma_k \left(s^{(0)}(\omega)\right)^k \left(\varphi_k -\sigma_k \int a^{(0)}(\nu) \left(s^{(0)}(\nu)\right)^k \diff\rho(\nu) \right) \, , \label{eq:aux-4}\\
\partial_{t_1} s^{(0)}(\modif{\omega}) &= 0 \label{eq:aux-5} \, .
\end{align}
From \eqref{eq:aux-5} and \eqref{eq:init-s-layer1}, we have $s^{(0)}(\omega) = 0$:
 time $t_1 = O(1) \Leftrightarrow t = O(\varepsilon)$ is too short for the~$s(\omega)$ to be of order $1$. 

Substituting $s^{(0)}(\omega) = 0$ in \eqref{eq:aux-4}, we obtain 
\begin{equation}
\partial_{t_1} a^{(0)}(\omega) = \sigma_0 \left( \varphi_0 - {\sigma_0} \int  a^{(0)}(\nu) \diff\rho(\nu)\right) \, . \label{eq:a-layer1}
\end{equation}
Recall that $\langle . , . \rangle_{L^2(\rho)}$ is the dot product on $L^2(\rho)$, $\bfone$ denotes the constant function $\bfone:\omega \in \Omega \mapsto 1 \in \R$ and $a_\perp$ is the orthogonal projection of $a$ on $\bfone^\perp$. Equation \eqref{eq:a-layer1} can be rewritten as
\begin{align*}
\partial_{t_1} \langle a^{(0)} , \bfone \rangle_{L^2(\rho)} &= \sigma_0 \left(\varphi_0 - {\sigma_0}  \langle a^{(0)} , \bfone \rangle_{L^2(\rho)} \right) \, , \\
\partial_{t_1} a_\perp^{(0)} &= 0 \, ,
\end{align*}
which gives after integration (using \eqref{eq:init-a-layer1}):
\begin{align}
\langle a^{(0)} , \bfone \rangle_{L^2(\rho)} &= e^{-\sigma_0^2 t_1} \langle a_{\rm{init}} , \bfone \rangle_{L^2(\rho)} + \left(1-e^{-\sigma_0^2 t_1}\right) \frac{ \varphi_0}{\sigma_0} \, , \label{eq:aux-26}\\
a_\perp^{(0)} &=  a_{\perp,\rm{init}}\, . \nonumber
\end{align}
At this point, we have determined $a^{(0)}(\omega)$ and $s^{(0)}(\omega)$, and thus $a(\omega) = a^{(0)}(\omega) + O(\varepsilon)$ and $s(\omega) = s^{(0)}(\omega) + O(\varepsilon)$ up to a $O(\varepsilon)$ precision, which is sufficient to obtain a $o(1)$-approximation of the risk $\Risk_{\mf,*}$ (see Section~\ref{sec:induced}). However, note that we could obtain more precise estimates by identifying higher-order terms in \eqref{eq:substitution-1-layer1}-\eqref{eq:substitution-4-layer1}. For instance, identifying the $O(\varepsilon)$ terms in \eqref{eq:substitution-3-layer1}-\eqref{eq:substitution-4-layer1}, we obtain $\partial_{t_1} s^{(1)}(\omega) = a^{(0)}(\omega) \sigma_1 \varphi_1$. This shows that the $s(\omega)$ become non-zero, though only of order $\varepsilon$ on the time scale $t_1 \asymp 1$; the inner-layer weights develop an infinitesimal correlation with the true direction~$u_*$ thanks to the linear component of $\sigma$ and $\varphi$. 

The approximation constructed above should be considered as valid on the time scale $t_1 \asymp 1 \Leftrightarrow t \asymp \varepsilon$. \modif{As $\veps \to 0$, we obtain the following approximation of the risk (see Eq.~\eqref{eq:risk-general} for definition, and Appendix~\ref{sec:induced} for a detailed derivation):
\begin{align*}
	\Risk_{\mf,*} = \frac{1}{2} e^{-2\sigma_0^2t_1}\left(\varphi_0 - {\sigma_0}\left\langle a_{\rm{init}}, \bfone \right\rangle_{L^2(\rho)} \right)^2  + \frac{1}{2 }\sum_{k\geq 1} \varphi_k^2 + O(\varepsilon) \, .
\end{align*}
}
This approximation breaks down when we reach a new time scale, at which the $s(\omega)$ are large enough for the $a(\omega)$ to be affected (at leading order) by the linear part of the functions. We detail the new time scale and its resolution in the next section. 

\subsection{Second time scale: linear component I}
\label{sec:second-layer}

In this section, we seek a second, slower time scale, for which the 
behavior of the asymptotic expansion is different. 

\paragraph{Identification of the scale.} Consider $t_2 = \frac{t}{\varepsilon^\gamma}$,
 where $\gamma < 1$ is to be determined. We rewrite the system \eqref{eq:group_sing_hermite} using $t_2$,
and expand the solutions $a(\omega)$ and $s(\omega)$: 
\begin{align}
a(\omega) &= a^{(0)}(\omega) +\varepsilon^\delta a^{(1)}(\omega)  + \varepsilon^{2\delta} a^{(2)}(\omega)  + \dots \, , \label{eq:expansion-a-layer2}\\
s(\omega) &= \varepsilon^{\delta} s^{(1)}(\omega)  + \varepsilon^{2\delta}s^{(2)}(\omega) + \dots \, , \label{eq:expansion-s-layer2}
\end{align}
\modif{where the exponent $\delta$ is also to be determined.}
(Since within the previous time scale we obtained $s(\omega)=O(\veps)$, it is natural to assume
$s^{(0)}(\omega) = 0$.)
\medskip

Let us pause to comment on our method. 

Similarly to what has been done in the previous time scale, we will substitute the
 expansions~\eqref{eq:expansion-a-layer2}-\eqref{eq:expansion-s-layer2} in the equations 
 \eqref{eq:group_sing_hermite}   in order to compute the different terms in the expansion.
  However, this step also allows us to compute the exponents $\gamma$ and $\delta$, 
  that give respectively the new time scale and the size of the $s(\omega)$'s.

Note that we should have proceeded similarly for the first time scale, by introducing
 a first time variable $t_1 = \frac{t}{\varepsilon^{\gamma'}}$, expanding $a(\omega), s(\omega)$ 
 in powers $1, \varepsilon^{\delta'}, \varepsilon^{2\delta'}, \dots$, and determining
  $\gamma'$ and $\delta'$ a posteriori. This would have led, indeed, to $\gamma' = 1$ and 
  $\delta' = 1$. However, for simplicity, we preferred to fix these values that are natural a priori. 

Finally, note that the expansions \eqref{eq:expansion-a-layer1}-\eqref{eq:expansion-s-layer1} and \eqref{eq:expansion-a-layer2}-\eqref{eq:expansion-s-layer2} are different, because they are valid on different time scales. In fact, the only coherence conditions that we require below is that the expansions match in a joint asymptotic where $t_1 = \frac{t}{\varepsilon} \to \infty$ and $t_2 = \frac{t}{\varepsilon^\gamma}  \to 0$. We thus build different approximations for each one of the time scales, with some matching conditions; this justifies the name of \emph{matched asymptotic expansion}.

\medskip

We now return to our computations and substitute \eqref{eq:expansion-a-layer2}-\eqref{eq:expansion-s-layer2} in \eqref{eq:group_sing_hermite}: 

\begin{align*}
\varepsilon^{1-\gamma} \partial_{t_2} a^{(0)}(\omega) + \dots &= \sum_{k=0}^{\infty} \sigma_k \left(\varepsilon^\delta s^{(1)}(\omega) + \dots \right)^k \\
&\qquad\qquad \times\left(\varphi_k - {\sigma_k} \int \left(a^{(0)}(\nu) + \dots \right) \left(\varepsilon^\delta s^{(1)}(\nu) + \dots \right)^k \diff\rho(\nu)\right) \, , \\
\varepsilon^\delta \partial_{t_2} s^{(1)}(\omega) + \dots &= \varepsilon^\gamma \left(a^{(0)}(\omega) + \dots \right) \left(1-\left(\varepsilon^\delta s^{(1)}(\omega) + \dots\right)^2 \right) \sum_{k=1}^{\infty} k \sigma_k \left(\varepsilon^\delta s^{(1)}(\omega) + \dots \right)^{k-1} \\
&\qquad\qquad \times\left(\varphi_k - {\sigma_k} \int \left(a^{(0)}(\nu) + \dots \right) \left(\varepsilon^\delta s^{(1)}(\nu) + \dots \right)^k \diff \rho(\nu)\right) \, ,
\end{align*}
and thus 
\begin{align}
\varepsilon^{1-\gamma} \partial_{t_2} a^{(0)}(\omega) + O(\varepsilon^{1-\gamma+\delta}) &= \sigma_0 \left(\varphi_0 - \sigma_0 \int  a^{(0)}(\nu) \diff\rho(\nu)\right) \label{eq:aux-6}\\ 
&\qquad-  \varepsilon^\delta {\sigma_0^2}\int  a^{(1)}(\nu) \diff\rho(\nu) + \varepsilon^\delta \sigma_1 \varphi_1 s^{(1)}(\omega) + O(\varepsilon^{2\delta}) \, , \\
\varepsilon^\delta \partial_{t_2} s^{(1)}(\omega) + O(\varepsilon^{2\delta}) &=  \varepsilon^\gamma  \sigma_1 \varphi_1 a^{(0)}(\omega) + O(\varepsilon^{\gamma + \delta}) \, . \label{eq:aux-7}
\end{align}
For the first time scale, we chose $\gamma = \delta = 1$, so that the terms of order $\varepsilon^\delta$ were negligible compared to $\varepsilon^{1-\gamma} \partial_{t_2} a^{(0)}(\omega)$ in~\eqref{eq:aux-6}. This means that the linear components $\sigma_1, \varphi_1$ of the functions had no effect on the $a(\omega)$ at leading order. We are now interested in a new time scale where $\varepsilon^{1-\gamma} \partial_{t_2} a^{(0)}(\omega)$ and $\varepsilon^\delta \sigma_1 \varphi_1 s^{(1)}(\omega)$ are of the same order, i.e., $1-\gamma = \delta$; then the linear components play a role in the dynamics.

Further, for $s^{(1)}(\omega)$ to be non-zero, we need both sides of \eqref{eq:aux-7} to be of the same order, thus $\delta = \gamma$. Putting together, this gives $\gamma = \delta = 1/2$.

\paragraph{Derivation of the ODEs for this time scale.} Let us summarize equations. For $t_2 = \frac{t}{\varepsilon^{\nicefrac{1}{2}}}$ and 
\begin{align*}
a(\omega) &= a^{(0)}(\omega) +\varepsilon^{\nicefrac{1}{2}} a^{(1)}(\omega)  + \dots \, , 
\\
s(\omega) &= \varepsilon^{\nicefrac{1}{2}} s^{(1)}(\omega)  + \dots \, , 
\end{align*}
we have from \eqref{eq:aux-6}-\eqref{eq:aux-7}:
\begin{align}
\varepsilon^{\nicefrac{1}{2}} \partial_{t_2} a^{(0)}(\omega) &= \sigma_0 \left(\varphi_0 - {\sigma_0} \int  a^{(0)}(\nu) \diff\rho(\nu)\right)  \label{eq:aux-8} \\
&\qquad -  \varepsilon^{\nicefrac{1}{2}} {\sigma_0^2} \int a^{(1)}(\nu) \diff\rho(\nu) + \varepsilon^{\nicefrac{1}{2}} \sigma_1 \varphi_1 s^{(1)}(\omega) + O(\varepsilon) \, ,\\
\varepsilon^{\nicefrac{1}{2}} \partial_{t_2} s^{(1)}(\omega) &=  \varepsilon^{\nicefrac{1}{2}}  \sigma_1 \varphi_1 a^{(0)}(\omega) + O(\varepsilon) \, . \label{eq:aux-9}
\end{align}
First, we identify the terms of order $1 = \varepsilon^0$: 
\begin{equation}
\label{eq:aux-12}
0 = \sigma_0 \left(\varphi_0 - {\sigma_0} \int a^{(0)}(\nu) \diff\rho(\nu)\right) \, .
\end{equation}
This means that the trajectory remains in the affine hyperplane defined
by $\varphi_0 = {\sigma_0} \int a^{(0)}(\nu) \diff\rho(\nu)$.
Intuitively, the constant component of $\varphi$ remains fitted by the neural network in this second time scale. 

Second, we identify the terms of order $\varepsilon^{\nicefrac{1}{2}}$ in \eqref{eq:aux-8}-\eqref{eq:aux-9}:
\begin{align}
 \partial_{t_2} a^{(0)}(\omega) &=  -   {\sigma_0^2}\int a^{(1)}(\nu) \diff\rho(\nu) +  \sigma_1 \varphi_1 s^{(1)}(\omega)   \, , \label{eq:aux-10}\\
 \partial_{t_2} s^{(1)}(\omega) &=    \sigma_1 \varphi_1 a^{(0)}(\omega)  \, . \label{eq:aux-11}
\end{align}
\modif{Note that, in Eqs.~\eqref{eq:aux-10}--\eqref{eq:aux-11}, the time derivative of
$a^{(1)}$ does not appear, and therefore the evolution of $a^{(1)}$
is not determined by these equations. In fact, $a^{(1)}$  is best 
interpreted as the Lagrange multiplier associated to the constraint \eqref{eq:aux-12}. 
Namely, this is a free term that can be adjusted so that the solution of
the  system \eqref{eq:aux-10}--\eqref{eq:aux-11} 
satisfies the constraint \eqref{eq:aux-12}. We can check unknown term in \eqref{eq:aux-10}
 leaves the right degree of freedom such that this is the case: we have
 \begin{equation*}
 	0 = \partial_{t_2} \left(\frac{\varphi_0}{\sigma_0}\right) \underset{\eqref{eq:aux-12}}{=} \partial_{t_2} \left( \int  a^{(0)}(\omega) \diff \rho(\omega)\right) \underset{\eqref{eq:aux-10}}{=}  -   {\sigma_0^2}\int a^{(1)}(\nu) \diff\rho(\nu) +  \sigma_1 \varphi_1 \int s^{(1)}(\omega) \diff \rho(\omega) \, .
 \end{equation*}
In this last expression, the first unknown term can always compensate the second term so that the
 constraint is satisfied. The entire evolution of $a^{(1)}$ is determined by higher orders in the 
 expansion.}

 To eliminate this Lagrange multiplier, we use again the compact notations: 
\begin{align}
\partial_{t_2} a^{(0)} &=  -   \sigma_0^2 \langle a^{(1)}, \bfone \rangle_{L^2(\rho)} \bfone +  \sigma_1 \varphi_1 s^{(1)}  \, , \label{eq:aux-13}\\
\partial_{t_2} s^{(1)} &=    \sigma_1 \varphi_1 a^{(0)}  \, , \label{eq:aux-14}
\end{align}
and thus 
\begin{align}
\partial_{t_2} a^{(0)}_\perp &=    \sigma_1 \varphi_1 s^{(1)}_\perp  \, , \label{eq:a_perp_layer2}\\
\partial_{t_2} s^{(1)}_\perp &=    \sigma_1 \varphi_1 a^{(0)}_\perp  \, . \label{eq:s_perp_layer2}
\end{align}

\paragraph{Matching.} The initialization of the ODEs \eqref{eq:aux-13}-\eqref{eq:aux-14} for the second time scale is determined by a classical procedure that matches with the previous time scale. In this paragraph, we denote $\underline{a}, \underline{s}$ the approximation obtained in the first time scale (Section \ref{sec:first-layer}), and $\overline{a}, \overline{s}$ the approximation in the second time scale, described above.

Consider an intermediate time scale $\widetilde{t} = \frac{t}{\varepsilon^\alpha}$, $\nicefrac{1}{2} < \alpha < 1$, and assume $\widetilde{t} \asymp 1$ so that 
\begin{align*}
&t_1 = \frac{t}{\varepsilon} = \frac{\widetilde{t}}{\varepsilon^{1-\alpha}} \to \infty \, , &&t_2 = \frac{t}{\varepsilon^{\nicefrac{1}{2}}} = {\varepsilon^{\alpha-\nicefrac{1}{2}}}{\widetilde{t}} \to 0 \, .
\end{align*}
In this intermediate regime, we want the approximations provided on the first and the second time scales to match: $\underline{a}(\widetilde{t})$ and $\overline{a}(\widetilde{t})$ (resp.~$\underline{s}(\widetilde{t})$ and $\overline{s}(\widetilde{t})$) should match to leading order.

From the first time scale approximation, 
\begin{align}
\underline{a} &= \underline{a}^{(0)} + O(\varepsilon) \label{eq:aux-15} \\
&=  \langle \underline{a}^{(0)}, \bfone \rangle_{L^2(\rho)} \bfone + \underline{a}^{(0)}_\perp + O(\varepsilon) \\
&=  \left[e^{-\sigma_0^2 t_1} \langle a_{\rm{init}}, \bfone \rangle_{L^2(\rho)} + \left(1 - e^{-\sigma_0^2 t_1} \right) \frac{ \varphi_0}{\sigma_0} \right] \bfone + a_{\perp,\rm{init}} + O(\varepsilon)  \\
&=  \left[e^{-\sigma_0^2 \nicefrac{\widetilde{t}}{\varepsilon^{1-\alpha}}} \langle a_{\rm{init}}, \bfone \rangle_{L^2(\rho)} + \left(1 - e^{-\sigma_0^2 \nicefrac{\widetilde{t}}{\varepsilon^{1-\alpha}}} \right) \frac{ \varphi_0}{\sigma_0} \right] \bfone + a_{\perp,\rm{init}} + O(\varepsilon) \\
&= \frac{\varphi_0}{\sigma_0} \bfone +  a_{\perp,\rm{init}} + o(1) \, . \label{eq:aux-16}
\end{align}
From the second time scale approximation, 
\begin{align}
\overline{a} &= \overline{a}^{(0)}(t_2) + O(\varepsilon^{\nicefrac{1}{2}}) = \overline{a}^{(0)}({\varepsilon^{\alpha-\nicefrac{1}{2}}}{\widetilde{t}}) + O(\varepsilon^{\nicefrac{1}{2}}) \label{eq:aux-17}\\
&= \overline{a}^{(0)}(0) + o(1) \, . \label{eq:aux-18}
\end{align}

By matching, Equations \eqref{eq:aux-16} and \eqref{eq:aux-18} should be coherent. Thus the ODE for the
 second time scale should be initialized from $\overline{a}^{(0)}(0) = \frac{\varphi_0}{\sigma_0} \bfone +  a_{\perp,\rm{init}}$.

Similarly, the matching procedure gives that the ODE for the second time scale should be initialized from $\overline{s}^{(1)} = 0$.

\paragraph{Solution.} As we are done with the matching procedure, we now consider the solution in the second time scale only, that we denote again by $a$, $s$ as in \eqref{eq:aux-13}, \eqref{eq:aux-14}. The matching procedure motivates us to consider the solution of
\eqref{eq:a_perp_layer2}-\eqref{eq:s_perp_layer2} initialized at $a_\perp^{(0)}(0) = a_{\perp,\rm{init}}$, $s_\perp^{(1)} =0$. This gives 
\begin{align*}
&a_\perp^{(0)} = \cosh\left(\varphi_1 \sigma_1 t_2\right)  a_{\perp,\rm{init}} \, ,  &&s_\perp^{(1)} = \sinh\left(\varphi_1 \sigma_1 t_2\right)  a_{\perp,\rm{init}} \, .
\end{align*}
To conclude, we note that $\langle a^{(0)}, \bfone \rangle_{L^2(\rho)} = \frac{\varphi_0}{\sigma_0}$ is constrained by \eqref{eq:aux-12}. Further, from \eqref{eq:aux-11},
\begin{equation*}
\partial_{t_2} \langle s^{(1)}, \bfone \rangle_{L^2(\rho)} = \sigma_1 \varphi_1 \langle a^{(0)}, \bfone \rangle_{L^2(\rho)} = \sigma_1 \varphi_1\frac{\varphi_0}{\sigma_0},
\end{equation*}
thus $\langle s^{(1)}, \bfone \rangle_{L^2(\rho)} = \sigma_1 \varphi_1\frac{\varphi_0}{\sigma_0} t_2$. 

Putting together, these equations give:
\begin{align}
\label{eq:layer2-approx}
&a^{(0)} = \frac{\varphi_0}{\sigma_0} \bfone + \cosh\left(\varphi_1 \sigma_1 t_2 \right) a_{\perp,\rm{init}} \, , &&s^{(1)} = \sigma_1 \varphi_1 \frac{\varphi_0}{\sigma_0} t_2 \bfone + \sinh\left(\varphi_1 \sigma_1 t_2 \right) a_{\perp,\rm{init}} \, . 
\end{align}

We observe that $a^{(0)}$ and $s^{(1)}$ diverge as $t_2 \to \infty$. This 
implies that our approximation on the second time scale must break down at a certain point.
 Indeed, we analyzed this time scale under the assumption that both $a^{(0)}$ and $s^{(1)}$ are of order $1$. However,
since $a^{(0)}$ and $s^{(1)}$ diverge exponentially as $t_2\to\infty$, 
as per Eq.~\eqref{eq:layer2-approx}, this assumption breaks down when $t_2\asymp \log(1/\varepsilon)$.

More precisely, in \eqref{eq:aux-8} (resp.~\eqref{eq:aux-9}), the $O(\varepsilon)$ term includes a term of the form
\begin{align*}
	&-\varepsilon s^{(1)}(\omega) \sigma_1^2 \int  a^{(0)}(\nu) s^{(1)}(\nu) \diff\rho(\nu) &&\left(\text{resp.~}-\varepsilon a^{(0)}(\omega) \sigma_1^2 \int  a^{(0)}(\nu) s^{(1)}(\nu) \diff\rho(\nu)\right) \, .
\end{align*}
 When $a^{(0)}$ and $s^{(1)}$ become of order $\varepsilon^{-\nicefrac{1}{4}}$, this term becomes of order $\varepsilon^{\nicefrac{1}{4}}$, which is then of the same order as the term $\varepsilon^{\nicefrac{1}{2}} \sigma_1 \varphi_1 s^{(1)}(\omega)$ in \eqref{eq:aux-8} (resp.~the term $\varepsilon^{\nicefrac{1}{2}}  \sigma_1 \varphi_1 a^{(0)}(\omega)$ in \eqref{eq:aux-9}). At this point, these terms can not be neglected anymore. From~\eqref{eq:layer2-approx}, we have 
\begin{align*}
&a^{(0)} \sim \frac{e^{|\varphi_1 \sigma_1| t_2}}{2} a_{\perp,\rm{init}} \, , &&s^{(1)}  \sim \sign(\varphi_1 \sigma_1) \frac{e^{|\varphi_1 \sigma_1| t_2}}{2} a_{\perp,\rm{init}} \, , &&t_2 \to \infty \, .
\end{align*}
Therefore, $a^{(0)}$ and $s^{(1)}$ become of order $\varepsilon^{-\nicefrac{1}{4}}$ at the time $t_2 \sim \frac{1}{4 |\sigma_1 \varphi_1|} \log \frac{1}{\varepsilon}$, at which the approximation on the second time scale breaks down. We thus introduce a new time scale centered at this critical point. 

\subsection{Third time scale: linear component II}
\label{sec:third-layer}

We now introduce the time $t_3 = t_2 - \frac{1}{4 |\varphi_1 \sigma_1|} \log \frac{1}{\varepsilon}$. As $t_3$ is only a translation from $t_2$, the ODEs in terms of $t_3$ are the same as the ones in term of $t_2$.
However, in this time scale, $a$ and $\varepsilon^{\nicefrac{1}{2}}s$ have diverged. In coherence with the discussion above, we seek expansions of the form 
\begin{align}
a &= \varepsilon^{-\nicefrac{1}{4}} a^{(-1)} + a^{(0)} + \varepsilon^{\nicefrac{1}{4}} a^{(1)} + \dots \, , \label{eq:expansion-a-layer3} \\
s &= \varepsilon^{\nicefrac{1}{4}} s^{(1)} + \varepsilon^{\nicefrac{1}{2}} s^{(2)} + \dots \, . \label{eq:expansion-s-layer3}
\end{align}
Similarly to the second time scale, we substitute \eqref{eq:expansion-a-layer3}-\eqref{eq:expansion-s-layer3} in \eqref{eq:group_sing_hermite} and obtain 
\begin{align*}
\varepsilon^{\nicefrac{1}{4}} \partial_{t_3} a^{(-1)}(\omega) &= - \varepsilon^{-\nicefrac{1}{4}} {\sigma_0^2} \int  a^{(-1)}(\nu) \diff\rho(\nu) + \sigma_0 \left( \varphi_0 - {\sigma_0} \int  a^{(0)}(\nu) \diff\rho(\nu)\right) \\
&\hspace{5mm} - \varepsilon^{\nicefrac{1}{4}} {\sigma_0^2} \int a^{(1)}(\nu)\diff\rho(\nu) + \varepsilon^{\nicefrac{1}{4}} \sigma_1 \left(\varphi_1 - {\sigma_1} \int  a^{(-1)}(\nu) s^{(1)}(\nu) \diff\rho(\nu)\right) s^{(1)}(\omega) + O(\varepsilon^{\nicefrac{1}{2}}) \, , \\
\varepsilon^{\nicefrac{1}{4}} \partial_{t_3} s^{(1)}(\omega) &= \varepsilon^{\nicefrac{1}{4}} \sigma_1 \left(\varphi_1 - {\sigma_1} \int  a^{(-1)}(\nu) s^{(1)}(\nu) \diff \rho(\nu)\right) a^{(-1)}(\omega) + O(\varepsilon^{\nicefrac{1}{2}}) \, . 
\end{align*}
First, we identify the terms of order $\varepsilon^{-\nicefrac{1}{4}}$:
\begin{equation}
0 = - {\sigma_0^2} \int  a^{(-1)}(\nu) \diff\rho(\nu) = - {\sigma_0^2} \left\langle a^{(-1)}, \bfone \right\rangle_{L^2(\rho)} \, . 
\label{eq:aux-19}
\end{equation}
This means that $a$ has no component diverging in $\varepsilon$ in the direction of $\bfone$.

Second, we identify the terms of order $1=\varepsilon^{0}$:
\begin{equation}
\label{eq:aux-20}
0 = \sigma_0 \left(\varphi_0 - {\sigma_0} \int  a^{(0)}(\nu) \diff\rho(\nu)\right) = \sigma_0 \left(\varphi_0 - \sigma_0 \left\langle a^{(0)}, \bfone \right\rangle_{L^2(\rho)}\right) \, .
\end{equation}
Put together with \eqref{eq:aux-19}, this equation ensures that the constant component of $\varphi$ remains learned on this third time scale. 

Third, we identify the terms of order $\varepsilon^{\nicefrac{1}{4}}$:
\begin{align}
\label{eq:aux-21}
\begin{split}
 \partial_{t_3} a^{(-1)}(\omega) &=  -  {\sigma_0^2}\int a^{(1)}(\nu) \diff\rho(\nu) +  \sigma_1 \left(\varphi_1 - {\sigma_1} \int  a^{(-1)}(\nu) s^{(1)}(\nu)\diff \rho(\nu)\right) s^{(1)}(\omega)  \, , \\
\partial_{t_3} s^{(1)}(\omega) &= \sigma_1 \left(\varphi_1 - {\sigma_1} \int  a^{(-1)}(\nu) s^{(1)}(\nu) \diff\rho(\nu)\right) a^{(-1)}(\omega)  \, . 
\end{split}
\end{align}
Again, the term $-  {\sigma_0^2}\int a^{(1)}(\nu) \diff\rho(\nu) $ is best interpreted as the Lagrange multiplier associated to the constraints \eqref{eq:aux-19}, \eqref{eq:aux-20}. Using the compact notations, 
\begin{align*}
\int  a^{(-1)}(\nu) s^{(1)}(\nu)\diff \rho(\nu) &= \left\langle a^{(-1)}, s^{(1)} \right\rangle_{L^2(\rho)} = {\left\langle a^{(-1)}, \bfone \right\rangle_{L^2(\rho)}\left\langle \bfone, s^{(1)} \right\rangle_{L^2(\rho)}} + \left\langle a^{(-1)}_\perp, s^{(1)}_\perp \right\rangle_{L^2(\rho)} \\&= \left\langle a^{(-1)}_\perp, s^{(1)}_\perp \right\rangle_{L^2(\rho)} \, ,
\end{align*}
where in the last equality we use \eqref{eq:aux-19}. Thus we can rewrite \eqref{eq:aux-21} as
\begin{align}
\label{eq:aux-22}
\begin{split}
\partial_{t_3} a^{(-1)} &=  -  {\sigma_0^2} \langle  a^{(1)}, \bfone \rangle_{L^2(\rho)} \bfone +  \sigma_1 \left(\varphi_1 - {\sigma_1} \left\langle a^{(-1)}_\perp, s^{(1)}_\perp \right\rangle_{L^2(\rho)}\right) s^{(1)}  \, , \\
\partial_{t_3} s^{(1)} &= \sigma_1 \left(\varphi_1 - {\sigma_1} \left\langle a^{(-1)}_\perp, s^{(1)}_\perp \right\rangle_{L^2(\rho)}\right) a^{(-1)}  \, ,
\end{split}
\end{align}
and thus 
\begin{align}
	\label{eq:ode-third-time-scale}
\begin{split}
\partial_{t_3} a^{(-1)}_\perp &=    \sigma_1 \left(\varphi_1 - {\sigma_1} \left\langle a^{(-1)}_\perp, s^{(1)}_\perp \right\rangle_{L^2(\rho)}\right) s^{(1)}_\perp  \, , \\
\partial_{t_3} s^{(1)}_\perp &= \sigma_1 \left(\varphi_1 - {\sigma_1} \left\langle a^{(-1)}_\perp, s^{(1)}_\perp \right\rangle_{L^2(\rho)}\right) a^{(-1)}_\perp  \, .
\end{split}
\end{align}
In Appendix \ref{sec:solution-ode-third-time-scale}, we solve this system of ODEs and determine the initial condition by matching with the previous layer. The result is that
\begin{align}
	\label{eq:aux-23-main}
	\begin{split}
		&a^{(-1)} = a^{(-1)}_\perp = \lambda a_{\perp,\rm{init}} \, , \\
		&s^{(1)} = s^{(1)}_\perp = \sign(\sigma_1 \varphi_1) \lambda a_{\perp,\rm{init}} \, ,
	\end{split}
\end{align}
where $\lambda = \lambda(t_3)$ is the function 
\begin{equation}
	\label{eq:aux-27-main}
	\lambda(t_3) = \frac{|\varphi_1|^{\nicefrac{1}{2}}}{\left({|\sigma_1|} \left\Vert a_{\perp,\rm{init}}\right\Vert_{L^2(\rho)}^2 + 4 |\varphi_1| e^{-2\vert \sigma_1 \varphi_1 \vert t_3} \right)^{\nicefrac{1}{2}}} \, .
\end{equation}
This solution finishes to describe how the linear part of the function $\varphi$ is learned. \modif{Plugging it into the equations for $a^{(-1)}$ and $s^{(1)}$, we get
\begin{align*}
	\sigma_1 \int  a^{(-1)}(\nu) s^{(1)}(\nu)\diff \rho(\nu) = \sigma_1 \left\langle a^{(-1)}_\perp, s^{(1)}_\perp \right\rangle_{L^2(\rho)} = \frac{\varphi_1 |\sigma_1| \left\Vert a_{\perp,\rm{init}}\right\Vert_{L^2(\rho)}^2 }{{|\sigma_1|} \left\Vert a_{\perp,\rm{init}}\right\Vert_{L^2(\rho)}^2 + 4 |\varphi_1| e^{-2\vert \sigma_1 \varphi_1 \vert t_3}},
\end{align*}
which converges to $\varphi_1$ as $t_3 \to \infty$. Consequently, we obtain the following approximation for $\Risk_{\mf,*}$ within this time scale (again, see Appendix \ref{sec:induced} for details):
\begin{align*}
	\Risk_{\mf,*} = \frac{1}{2}  \varphi_1^2\left(1 - \frac{1}{1 + \frac{4 |\varphi_1|}{|\sigma_1| \left\Vert a_{\perp,\rm{init}}\right\Vert_{L^2(\rho)}^2} e^{-2\vert \sigma_1 \varphi_1 \vert t_3} }   \right)^2 + \frac{1}{2 }\sum_{k\geq 2} \varphi_k^2 + O(\varepsilon^{\nicefrac{1}{4}}) \, , \quad \veps \to 0.
	\end{align*}
}

\subsection{Conjectured behavior for larger time scales}
\label{sec:conjectured}

The analysis of the previous sections naturally suggests the existence
of a sequence of cutoffs. At each time scale, a new polynomial component of $\varphi$ is learned
within a window that is much shorter than the time elapsed before that phase started. 
Along this sequence, we 
expect $s$ and $a$ to grow to increasingly larger scales in $\varepsilon$ (but $s$ remains 
$o(1)$ while $a$ diverges). 

More precisely, we assume that during the $l$-th phase, the network learns 
the degree-$l$ component 
$\varphi_l$, and various quantities satisfy the following
scaling behavior:
\begin{align}
a = O(\varepsilon^{-\omega_l}), \;\;\;
s = O(\varepsilon^{\beta_l}), \;\;\; t = O(\varepsilon^{\mu_l})\, ,
\end{align}
where $\omega_l > 0$ is an increasing sequence and $\beta_l, \mu_l>0$ are decreasing sequences.
Further, while learning of this component takes place when $t = O(\varepsilon^{\mu_l})$,
the actual evolution of the risk (and of the neural network) take place on much shorter scales,
namely:
\begin{align}
 \Delta t = O(\varepsilon^{\nu_l})\, ,
\end{align}
where $\nu_l$ is also decreasing, with $\nu_l>\mu_l$. 
The goal of this section is to provide heuristic arguments to conjecture the values of 
$\omega_l$, $\beta_l$, $\mu_l$ and $\nu_l$. We will base this conjecture on a rigorous analysis of
a simplified model.

The simplified model is motivated by the expectation (supported by the heuristics and simulations
in the previous sections) 
that learning each component happens independently from the details of the
evolution on previous time scales.
In the simplified model, the activation function
  $\sigma (x)$ is proportional to the $l$-th Hermite polynomial, namely $\sigma (x) = \sigma_l \mathrm{He}_l (x)$. 
  This is the component of $\sigma$ that we expect to be relevant on the $l$-th time scale.
  The gradient flow equations \eqref{eq:group_sing_hermite} then read:
\begin{equation}\label{eq:l-th_single}
	\begin{split}
		\veps \partial_{t} a(\omega) &= \, \sigma_l s(\omega)^l \left( \varphi_l - \sigma_l \int a(\nu) s(\nu)^l \diff\rho(\nu) \right) \, , \\
		\partial_{t} s(\omega) &= \, a(\omega) \left( 1 - s(\omega)^2 \right) l \sigma_l s(\omega)^{l - 1} \left( \varphi_l - \sigma_l \int a(\nu) s(\nu)^l \diff \rho(\nu) \right) \, .
	\end{split}
\end{equation}
with corresponding risk component
\begin{equation*}
	\Risk_l = \frac{1}{2} \left( \varphi_l - \sigma_l \int a(\nu) s(\nu)^l \diff \rho(\nu) \right)^2.
\end{equation*}
We capture the effect of learning dynamics on the previous time scales
by the overall magnitude of the $a (\omega)$'s and $s (\omega)$'s at initialization. 
Namely, we choose the scale of initialization of the simplified model to be given by  
the end of the $(l - 1)$-th time scale,
 i.e., $a (\omega) \asymp \veps^{- \omega_{l - 1}}$ and $s (\omega) \asymp \veps^{\beta_{l - 1}}$. Further,
 in order for the $(l-1)$-th component to be learned, namely
 \begin{equation}
 	\int a(\nu) s(\nu)^{l-1} \diff \rho(\nu) \approx \frac{\varphi_{l - 1}}{\sigma_{l - 1}},
 \end{equation}
 we require $\omega_{l - 1} = (l - 1) \beta_{l - 1}$ so that $\int a(\nu) s(\nu)^{l-1} \diff \rho(\nu) = \Theta (1)$. Analogously, we 
  assume  $\omega_{l} = l \beta_{l}$.

 Based on this consideration, we introduce the rescaled variables
\begin{equation*}
	\ta (\omega) = \veps^{\omega_l} a (\omega), \ \ts (\omega) = \veps^{- \beta_l} s (\omega), \ \text{where} \ 
	\ta (\omega, 0) \asymp \veps^{\omega_l - \omega_{l - 1}}, \ts(\omega, 0) \asymp \veps^{\beta_{l - 1} - \beta_l}.
\end{equation*}
Rewriting Eq.~\eqref{eq:l-th_single} in terms of $\ta(\omega)$'s and $\ts(\omega)$'s, and using $\omega_l = l \beta_l$, we get that
\begin{equation}
	\begin{split}
		\veps^{1 - 2 l \beta_l} \partial_t \ta(\omega) = \, & \sigma_l \ts(\omega)^l \left( \varphi_l - \sigma_l \int \ta(\nu) \ts(\nu)^l \diff \rho(\nu) \right) \\
		\veps^{2 \beta_l} \partial_t \ts(\omega) = \, & l \sigma_l \ta(\omega) \ts(\omega)^{l - 1} \left( 1 - \veps^{2 \beta_l} \ts(\omega)^2 \right) \left( \varphi_l - \sigma_l \int \ta(\nu) \ts(\nu)^l \diff \rho(\nu) \right).
	\end{split}
\end{equation}
In order for the $\ta(\omega)$'s and $\ts(\omega)$'s to be learned simultaneously, we need 
$1 - 2 l \beta_l = 2 \beta_l$, which implies $\beta_l = 1 / 2 (l + 1)$. Making a further change of 
the time variable $t = \veps^{\nu_l} \tau$, where $\nu_l = 2 \beta_l = 1 / (l+1)$, it follows that
\begin{equation}\label{eq:l-th_simp}
	\begin{split}
		\partial_\tau \ta(\omega) = \, & \sigma_l \ts(\omega)^l \left( \varphi_l - \sigma_l \int \ta(\nu) \ts(\nu)^l \diff \rho(\nu) \right) \\
		\partial_\tau \ts(\omega) = \, & l \sigma_l \ta(\omega) \ts(\omega)^{l - 1} \left( 1 - \veps^{2 \beta_l} \ts(\omega)^2 \right) \left( \varphi_l - \sigma_l \int \ta(\nu) \ts(\nu)^l \diff \rho(\nu) \right).
	\end{split}
\end{equation}
Moreover, rewriting the risk in terms of the rescaled variables
$\ta, \ts$, $\Risk_l (\tau) = \Risk_l(\ta(\tau),\ts(\tau))$ satisfies the ODE:
\begin{equation}\label{eq:lth_risk_ODE}
	\partial_\tau \Risk_l = - 2 \sigma_l^2 \Risk_l \cdot \int \ts(\omega)^{2 (l - 1)} \left( l^2 \ta(\omega)^2 \left( 1 - \veps^{2 \beta_l} \ts(\omega)^2 \right) + \ts(\omega)^2 \right) \diff \rho(\omega).
\end{equation}
Note that with our choice of $\beta_l$ and $\omega_l$, we have $\omega_l - \omega_{l - 1} = \beta_{l - 1} - \beta_l = 1 / 2 l (l+1)$. This means that the 
$\ta(\omega)$'s and $\ts(\omega)$'s are initialized at the same scale, namely 
\begin{align}
\ta(\omega, 0), \ts(\omega,0) = \Theta(\veps^{1/2l(l+1)})\, . \label{eq:BarABarS}
\end{align}
The theorem below describes quantitatively the dynamics of the 
simplified model for small $\veps$, and determines the value of $\mu_l$ (recall that $\nu_l = 1 / (l+1)$):
\begin{thm}[Evolution of the simplified gradient flow]\label{thm:evolution_single}
	Assume $l \ge 2$  and let $(\ta(\omega, \tau), \ts(\omega, \tau))_{\tau \ge 0}$ be the unique
	solution  of the ODE system \eqref{eq:l-th_simp}, initialized as 
	per Eq.~\eqref{eq:BarABarS} (note in particular that $\sigma_l \varphi_l \ta(\omega, 0) \ts(\omega,0)^l \asymp \veps^{1/2 l}$). Then the followings hold:
	\begin{itemize}
		\item [$(a)$] Let us denote
		\begin{equation}
			A = \left\{ \omega: \sigma_l \varphi_l \liminf_{\veps \to 0} \veps^{-1 / 2 l} \ta(\omega, 0) \ts(\omega,0)^l > 0 \right\}
		\end{equation}
	    and assume $\rho(A) > 0$. For $\Delta \in (0, \varphi_l^2 / 2)$, define
	    \begin{equation}
	    	\tau (\Delta) = \inf \{ \tau \ge 0: \Risk_l(\ta(\tau),\ts(\tau)) \le \Delta \}.
	    \end{equation}
	    Then, for any fixed $\Delta$ we have $\tau(\Delta) = \Theta (\veps^{- (l - 1) / 2 l (l + 1)})$ as $\veps \to 0$. Further, if $\rho$ is a discrete probability measure, then there exists $\tau_*(\eps) = \Theta(\veps^{- (l - 1) / 2 l (l + 1)})$ and,
		for any $\Delta>0$ a constant $c_*(\Delta)>0$ independent of $\eps$ such that
		\begin{align}
		\tau \le \tau_*(\eps)-c_*(\Delta)   & \Rightarrow \;\;
		\liminf_{\eps\to 0}\Risk_l(\ta(\tau),\ts(\tau)) \ge \frac{1}{2}\varphi_l^2-\Delta\, ,\\
		\tau \ge \tau_*(\eps)+c_*(\Delta)  & \Rightarrow \;\;
		\limsup_{\eps\to 0}\Risk_l(\ta(\tau),\ts(\tau)) \le \Delta\, ,
		\end{align}
	    namely the $l$-th component is learnt in an $O(1)$ time window around $\tau_*(\eps) = \Theta(\veps^{- (l - 1) / 2 l (l + 1)})$.
	    
		\item [$(b)$] Similarly, we denote
		\begin{equation}
			B = \left\{ \omega: \sigma_l \varphi_l \limsup_{\veps \to 0} \veps^{-1 / 2 l} \ta(\omega, 0) \ts(\omega,0)^l < 0, \ \text{and} \ \liminf_{\veps \to 0} (\ts (\omega, 0)^2 / \ta (\omega, 0)^2) > l \right\}.
		\end{equation}
		If $\rho(B) > 0$, then the same 
		claims as in $(a)$ hold.
		\item [$(c)$] If neither of the conditions at points $(a)$, $(b)$ holds, and
		\begin{equation}
			\sigma_l \varphi_l \limsup_{\veps \to 0} \veps^{-1 / 2 l} \ta(\omega, 0) \ts(\omega,0)^l < 0, \quad \limsup_{\veps \to 0} (\ts (\omega, 0)^2 / \ta (\omega, 0)^2) < l
		\end{equation}
		for almost every $\omega \in \Omega$. Then, for such $\omega \in \Omega$ and each $\Delta>0$, there exists a constant $C_*(\omega, \Delta)>0$ such that
		 \begin{align}
		 \tau \ge C_* (\omega, \Delta) \veps^{- (l-1) / 2l (l+1)} \;\;\Rightarrow\;\; 
		 \vert \ts(\omega, \tau) \vert \le \Delta \veps^{1 / 2l (l+1)}\, ,
		 \end{align}
		 meaning that $\ts (\omega, \tau)$ converges to $0$ eventually.
		\end{itemize}
	We further note that 
	$\tau = \Theta(\veps^{- (l - 1) / 2 l (l + 1)}) \Longleftrightarrow t = \Theta (\veps^{\mu_l})$ with $\mu_l = 1/2l$, 
	and  $\tau = O(1) \Longleftrightarrow t = O(\veps^{\nu_l})$ with $\nu_l = 1/(l+1)$.
\end{thm}
The proof of Theorem~\ref{thm:evolution_single} is deferred to Appendix~\ref{sec:proof_evo_sin}. 

\begin{remark}
	Under the conditions of cases $(a)$ and $(b)$, we see that the degree-$l$ component of the target function is learnt within an $O(\veps^{1 / (l+1)})$ time window around $t_* (l, \veps) \asymp \veps^{1/ 2 l}$, which is consistent with the timescales conjectured in Definition~\ref{def:StandardScenario}.
\end{remark}
\begin{remark}
Case $(c)$ corresponds to $s (\omega) / s(\omega, 0)$ becoming 
close to $0$ in time $t = O (\veps^{\mu_l})$, and staying at $0$. In other words, 
 the neurons become orthogonal to the target direction and play no role in 
 learning higher-degree components any longer. 
 
 Informally, case $(c)$ couples the learning of different polynomial components.
 It can happen that
 the learning phase $l-1$ induces 
 an effective initialization $(\ta (\omega, 0), \ \ts (\omega, 0))$ within the domain of
 case $(c)$. 
 
 We expect this not to be the case for suitable choices of initialization (or equivalently $\rP_A$), $\varphi$, and $\sigma$. Establishing this would amount to establishing that
 the \modif{canonical learning order} holds.  
\end{remark}

\section{Stochastic gradient descent and finite sample size}\label{sec:gf_sgd}

So far we focused on analyzing the projected gradient flow (GF) dynamics 
with respect to the population risk, as defined in Eqs.~\eqref{eq:GF-1}-\eqref{eq:GF-2}.
In this section, we extract the implications of our analysis of GF on online projected 
stochastic gradient descent, which is a projected version of the SGD dynamics~\eqref{eq:sgd_MF}. 

For simplicity of notation, we denote by $z = (y, x)\in \reals\times\reals^d$
a datapoint and by $\theta_i = (a_i,u_i)\in\reals\times\S^{d-1}$ the parameters of neuron $i$.
For $z = (y, x)$ and $\rho^{(m)} = (1/m) \sum_{i=1}^{m} \delta_{\theta_i} = (1/m) \sum_{i=1}^{m} 
\delta_{(a_i, u_i)}$, we define
\begin{align*}
	\hat{F}_i (\rho^{(m)}; z) = \, & \left( y - \frac{1}{m} \sum_{j = 1}^{m} a_j \sigma (\langle u_j, x \rangle) \right) \sigma (\langle u_i, x \rangle), \\
	\hat{G}_i (\rho^{(m)}; z) = \, & a_i \left( y - \frac{1}{m} \sum_{j=1}^{m} a_j \sigma (\langle u_j, x \rangle) \right) \sigma' (\langle u_i, x \rangle) x.
\end{align*}

The projected SGD dynamics is specified as follows:
\begin{equation}\label{eq:Psgd_MF}
	\begin{split}
		\overline{a}_i (k + 1) = \, & \overline{a}_i (k) + \eps^{-1} \eta \hat{F}_i (\overline{\rho}^{(m)} (k); z_{k + 1}) \\
		\overline{u}_i (k + 1) = \, & \operatorname{Proj}_{\S^{d - 1}} \left( \overline{u}_i (k) + \eta \hat{G}_i (\overline{\rho}^{(m)} (k); z_{k + 1}) \right),
	\end{split}
\end{equation}
where for $u \in \R^d$ and compact $S \subset \R^d$, $\operatorname{Proj}_{S} (u) :=
 \argmin_{s \in S} \norm{s - u}_2$, and  $\orho^{(m)} := (1/m) \sum_{i=1}^{m} \delta_{\otheta_i}$. Note that the $(\oa_i, \ou_i)$'s here are different from the $(\oa, \os)$'s in Section~\ref{sec:matched}.

We prove that, for small $\eta$, the projected SGD of Eq.~\eqref{eq:Psgd_MF} is close to the gradient flow
of Eqs.~\eqref{eq:GF-1}-\eqref{eq:GF-2}. 
Throughout  this section, we make the following assumptions similar to those assumed in Section~\ref{sec:LargeNet}:
\begin{description}\label{ass:MF_limit_more}
	\item [A1.] $\rho_0$ is supported on $[-M_1, M_1] \times \S^{d - 1}$. Hence, $\vert a_i (0) \vert \le M_1$ for all $i \in [m]$.
	\item [A2.] The activation function is bounded: $\norm{\sigma}_{\infty} \le M_2$. Additionally, define for $u, u' \in \R^d$:
	\begin{align}\label{eq:new_def_V}
		V ( \langle u_*, u \rangle; \norm{u_*}_2, \norm{u}_2 ) = \, & \E \left[ \varphi (\langle u_*, x \rangle) \sigma (\langle u, x \rangle) \right] , \\ 
		\label{eq:new_def_U}
		U (\langle u, u' \rangle; \norm{u}_2, \norm{u'}_2) = \, & \E \left[ \sigma (\langle u, x \rangle) \sigma (\langle u', x \rangle) \right].
	\end{align}
    We then require the functions $V$ and $U$ to be bounded and differentiable, with uniformly bounded and Lipschitz continuous gradients for all $\norm{u}_2, \norm{u'}_2 \le 2$:
    \begin{align}\label{eq:grad_bound_V}
    	& \norm{\nabla_u V}_{2} \le M_2, \ \norm{\nabla_u V - \nabla_{u'} V }_{2} \le M_2 \norm{u - u'}_2, \\
    	\label{eq:grad_bound_U}
    	& \norm{\nabla_{(u, u')} U}_{2} \le M_2, \ \norm{\nabla_{(u, u')} U - \nabla_{(u_1, u_1')} U}_2 \le M_2 \left( \norm{u - u_1}_2 + \norm{u' - u_1'}_2 \right).
    \end{align}
    Similar to Remark~\ref{rem:ass_A2_suff}, we can show that a sufficient condition for Eq.s~\eqref{eq:grad_bound_V} and \eqref{eq:grad_bound_U} is
	\begin{equation*}
		\sup \left\{ \norm{\sigma'}_{L^2}, \, \norm{\sigma''}_{L^2} \right\} \le M_2', \quad \ \sup \left\{ \norm{\varphi}_{L^2}, \, \norm{\varphi'}_{L^2}, \, \norm{\varphi''}_{L^2} \right\} \le M_2',
	\end{equation*}
    where the constant $M_2'$ depends uniquely on $M_2$.
	\item [A3.] Assume $(x, y) \sim \P$, then we require that $y \in [- M_3, M_3]$ almost surely. Moreover, we assume that for all $\norm{u}_2 \le 2$, both $\sigma(\langle u, x \rangle)$ and $\sigma' (\langle u, x \rangle) (x - \langle u, x \rangle u)$ are $M_3$-sub-Gaussian.
\end{description}

The following theorem upper bounds the distance between gradient flow 
and projected stochastic gradient descent dynamics.
\begin{thm}[Difference between GF and Projected SGD]\label{thm:diff_gf_psgd}
Let $\theta_i(t) = (a_i(t), u_i(t))$ be the solution of the GF
ordinary differential equations~\eqref{eq:GF-1}-\eqref{eq:GF-2}.
There exists a constant $M$ that only depends on the $M_i$'s from Assumptions A1-A3, such that for 
any $T, z \ge 0$ and
	\begin{equation*}
		\eta \le \frac{1}{(d + \log m + z^2) M \exp((1 + 1 / \veps) M T (1 + T / \veps)^2)},
	\end{equation*}
	the following holds with probability at least $1 - \exp(- z^2)$:
	\begin{align}\label{eq:a_bd_psgd}
		\sup_{k \in [0, T/\eta] \cap \mathbb{N}} \max_{i \in [m]} \left\vert \overline{a}_i (k) \right\vert \le \, & M(1 + T / \veps), \\
		\label{eq:theta_diff_gf_psgd}
		\sup_{k \in [0, T/\eta] \cap \mathbb{N}} \max_{i \in [m]} \norm{\theta_i (k \eta) - \overline{\theta}_i (k)}_2 \le \, & \left( \sqrt{d + \log m} + z \right) \\
		& \ \ \times M \exp \left( \left( 1 + \frac{1}{\veps} \right) M T \left(1 + \frac{T}{\veps} \right)^2 \right) \sqrt{\eta}, \\
		\label{eq:risk_diff_gf_psgd}
		\sup_{k \in [0, T/\eta] \cap \mathbb{N}} \left\vert \Risk ( \oa (k),\ou (k)  ) - \Risk ( a(k \eta),u(k\eta) ) \right\vert \le \, & \left( \sqrt{d + \log m} + z \right) \\
		& \ \ \times M \exp \left( \left( 1 + \frac{1}{\veps} \right) M T \left(1 + \frac{T}{\veps} \right)^2 \right) \sqrt{\eta}.
	\end{align}
\end{thm}
The proof is presented in Appendix~\ref{sec:coupling_bd} and follows the same scheme
 as in that of Theorem 1 part (B) in \citep{mei2019mean}. The main difference with respect to
 that theorem is here we are interested in projected SGD (and GF) instead of plain SGD (and GF), hence an additional step of approximation is required, and the $a_i$'s and $u_i$'s need to be treated separately.
We next draw implications of the last result on learning by
online SGD within the \modif{canonical learning order}.
%
\begin{thm}\label{thm:MainLearning}
Fix any $\delta > 0$. Assume $\varphi,\sigma$ and the initialization $\rP_A$ be such that the \modif{canonical learning order}
of Definition \ref{def:StandardScenario} holds up to level $L$ for some $L \ge 2$, and that
\begin{equation}
	\sum_{k \ge L+1} \varphi_k^2 \le \frac{\delta}{2}.
\end{equation}
Then, there exist constants $\veps_*=\veps_*(\delta)$, $T_0 = T_0 (\delta)$, $T = T(\eps, \delta) = T_0 (\delta) \veps^{1/(2L)}$ and $M = M(\eps, \delta)$ that  depend on $\eps, \delta$
 (together with $\varphi,\sigma$ and $\rP_A$)
 such that the following happens. Assume $\eps\le \eps_*(\delta)$ and
 $m, d, z$ are such that $d \ge M$, $m \ge \max(M, z)$, and the step size 
 $\eta$ and number of samples (equivalently, number of steps) $n$
satisfy
\begin{align} 
\eta & = \frac{1}{M (d+\log m+ z)}\, ,\\
n & = MT (d+\log m+ z)\, .
\end{align}
Then, with probability at least $1-e^{-z}$,
 the projected gradient descent algorithm of Eq.~\eqref{eq:Psgd_MF}
achieves population risk smaller than $\delta$: 
\begin{align} 
\P\Big(\Risk ( \oa (n),\ou (n) ) \le \delta\Big)\ge 1-e^{-z}\, .
\end{align}
\end{thm}
The proof of Theorem~\ref{thm:MainLearning} is deferred to Appendix~\ref{sec:proof_MainLearning}.
\begin{remark}
Within the lazy or neural tangent regime, learning the projection
of the target function $\varphi(\< u_*,x\>)$ onto polynomials of degree $\ell$
requires $n\gg d^{\ell}$ samples, and $m\gg d^{\ell-1}$ neurons 
\citep{ghorbani2021linearized,mei2022generalization,montanari2022interpolation}. 

In contrast, Theorem \ref{thm:MainLearning} shows that, within the \modif{canonical learning order}, $O(d)$ samples and $O(1)$ neurons are sufficient. Further as per Theorem 
\ref{thm:diff_gf_psgd}, the learning dynamics is accurately described by the 
GF analyzed in the previous sections. 
\end{remark}

\modif{
\section{Discussion}
\label{sec:discussion}

We conclude by discussing some of our findings as well as potential extensions of our work.
As mentioned in the introduction, our initial motivation was to understand 
certain ubiquitous phenomena in the learning dynamics of multi-layer neural networks.
A particularly striking phenomenon that we could reproduce in the 
present mathematical setting is the coexistence of plateaus in which the risk barely changes 
and sudden drops. 

In the next paragraphs, we will briefly emphasize results or future directions that
were not anticipated at the beginning of this work.

\paragraph{Implicit bias in function space.} We provided evidence towards
the canonical learning order of Definition \ref{def:StandardScenario}.
According to this scenario, the target function
$\varphi$ is learnt according to its decomposition into Hermite polynomials, with lower degree
components learnt first.
This theory applies to online SGD via Theorem \ref{thm:diff_gf_psgd} and 
Theorem \ref{thm:MainLearning}. In this setting, the number of SGD steps correspond to the number of samples.
Therefore, at a small sample size, SGD will fit a low degree polynomial approximation
of the target function, with the degree increasing with samples.

A similar phenomenon is observed with (rotationally invariant) kernel methods \cite{mei2022generalization},
with one important difference. Here the number of samples always scale linearly in the degree,
while for kernel methods, different polynomial degree correspond to different scalings with the dimension.

\paragraph{Implicit bias in parameter space.} 
Our analysis tracks the evolution of the weights as well.
As explained in Section \ref{sec:matched}, in order for the degree-$k$ component of
 the target function to be well approximated (in the $d,m\to\infty$ limit),
 it is sufficient that $\sigma_k \int a(\nu) s(\nu)^k \rho(\de\nu) 
= \varphi_k$. Here $\nu$ is an abstract neuron index, $a(\nu)$ is the second-layer weight
and $s(\nu)$ is the projection of the first layer weight along the
target direction $u_*$. 

Naively, one would expect that, in order for learning to take place, first 
layer weights should be well aligned with $u_*$, i.e. $s(\nu)$ should concentrate close to one.
However this is not the only way to satisfy the constraints  
$\sigma_k \int a(\nu) s(\nu)^k d\rho(\nu) = \varphi_k$. Indeed,
our analysis in Section~\ref{sec:matched} indicates that gradient
flow satisfies this constraint with  $s = \Theta(\eps^{\beta_k})$ and $a = \Theta(\eps^{-\omega_k})$
with $\beta_k = 1/2(k+1)$, $\omega_k=k/2(k+1)$
 (so that $\sigma_k \int a(\nu) s(\nu)^k \rho(\de\nu)$ will be of order one) as $\varepsilon \to 0$.
 In other words, the alignment is small, and second layer weights are large.
 (In general, weights on multiple scales coexist.)

\paragraph{The role of the learning rate $\eps$.} 
The initialization of parameters and relative step-sizes play a key role 
in modern (non-convex) machine learning. The combination of the two scalings (initialization and relative
stepsize) affects the learning dynamics.
In order to clarify this point, we can consider a general parametrization
(we keep $\|u_i\|_2=1$)
 \begin{equation*}
f(x;a,u) = \frac{1}{m^{\gamma}} \sum_{i=1}^{m} a_i \sigma(\langle u_i, x \rangle)=: 
 \sum_{i=1}^{m} c_i \sigma(\langle u_i, x \rangle),\;
\end{equation*}
and gradient flow dynamics 
\begin{align*}
\eps\partial_t  a_i &= - m \partial_{a_i} \Risk(a,u) \, ,   \\
\partial_t u_i &= - m (I_d - u_i u_i^\top) \nabla_{u_i} \Risk(a,u) \, . 
\end{align*}
(Note that the learning rate in the second equation can be set to $1$ without loss of generality, 
by rescaling the time axis.)
Rewriting this in terms of the coefficients $c_i$, so that the function 
representation is kept fixed, we have
\begin{align*}
s \partial_t  c_i &= - m \partial_{c_i} \tRisk(c,u) \, ,  \;\; s=\eps m^{2\gamma}\, ,
\end{align*}
while the second equation remains unchanged. This parametrization
allows us to compare various scalings in a uniform fashion.
\begin{itemize}
\item  Mean field scaling \citep{mei2018landscape,chizat2018global}: 
 $s=\Theta(m^2)$, $|c_i(0)| = \Theta(m^{-1})$.
 \item In this paper: $s= \eps m^2$, $|c_i(0)| = \Theta(m^{-1})$, $\eps\to 0$ after $m\to\infty$.
\item Classical scaling \citep{lecun2002efficient,he2015delving}: 
 $s=\Theta(1)$, $|c_i(0)| = \Theta(m^{-1/2})$.
\end{itemize}
As mentioned already, mean field scaling can exhibit better feature learning properties. 
In particular, the class of functions studied in the present paper can require
much larger sample size
to learn under the classical scaling \citep{oymak2020toward,ghorbani2021linearized,yehudai2019power}.
The choice of initialization in this paper is the same as 
in the mean field literature, with the difference that the relative
learning rate $s$ is a factor $\eps$ smaller, hence making it --in a sense-- slightly closer
to the the classical scaling. It would be interesting to explore other scalings
as well.

We also note that, while the limit of small $\eps$ is interesting,
setting directly $\eps=0$ leads to a  singular behavior\footnote{No matter how we rescale time,
in this case learning takes place instantly,  up to a certain critical degree.}. Formally, 
setting $\eps=0$ corresponds to keeping second layer weights equal to their optimal values:
a correct analysis  of this case requires to account for the role of stepsize 
and not just use the gradient flow approximation.

\paragraph{More complex network models.} 
The choice of the neural network model in this paper was mainly dictated by 
the desire to avoid inessential technicalities.  It would be
important to move towards more realistic models.

First, we used projected gradient descent to constrain the weights' norms $\|u_i\|=1$.
While this is a common theoretical device in studying
 single-index models \citep{arous2021online, bietti2022learning}, we believe that
techniques developed here can be extended to the more general case.
  Analogously, we could add biases to the network architecture and hence replace Eq.~\eqref{eq:First-NNET}
  by
  \begin{equation}
f(x;a,u, b) = \frac{1}{m} \sum_{i=1}^{m} a_i \sigma(\langle u_i, x \rangle+b_i), \qquad 
\ a_1, b_1,\cdots, a_m, b_m \in \R, \ u_1, \cdots, u_m \in \reals^{d},
\end{equation}
 With this change, the limiting mean-field dynamics
  will be an autonomous ODE system of $(a_i(t), b_i(t), s_i(t), r_i (t))_{i=1}^{m}$ where 
  $r_i (t) = \norm{u_i (t)}_2$.  
   We expect that its evolution will be qualitatively similar to
   that of the simplified dynamics considered in the paper.
   
Second, the single-index model studied here is a simple example of target function 
which requires feature learning. An obvious generalization
is to consider multi-index models, as already discussed in Remark \ref{rmk:MultiIndex}. 

Finally, it would be interesting to generalize our analysis to classification losses.
    }
     
\section*{Acknowledgments}

This work was supported by the NSF through award DMS-2031883, the Simons Foundation through
Award 814639 for the Collaboration on the Theoretical Foundations of Deep Learning, the NSF grant
CCF-2006489 and the ONR grant N00014-18-1-2729, and a grant from Eric and Wendy Schmidt
at the Institute for Advanced Studies. Part of this work was carried out while Andrea Montanari
was on partial leave from Stanford and a Chief Scientist at Ndata Inc dba Project N. The present
research is unrelated to AM’s activity while on leave.

\newpage
\bibliographystyle{plainnat}
\bibliography{bibliography}	

\newpage
\appendix

\modif{
\section{Proof of Proposition \ref{propo:Approx}}
\label{app:Approx}

By standard approximation theory arguments \citep{pinkus1999approximation}, it is sufficient to show
that there exists an integrable function $a_d\in L^1(\S^{d-1},\mu_0)$ such that
\begin{align}
\lim_{d\to\infty} \E\big\{\big(\int a_d(u)\,  \sigma(\<u,x\>) \, \mu_0(\de u)-
 \varphi(\<u_*,x\>)\big)^2\big\} = 0 \, .\label{eq:Representation}
\end{align}
 (We denote by $\mu_0$ 
the uniform probability measure over $\S^{d-1}$.)

Denote by $P_{d,k}$ the Gegenbauer polynomial of order $d$ and degree $k$
(see, e.g., \cite{mei2022generalization}).
Namely, $(P_{d,k}:k\ge 0)$ form an orthogonal system with respect to the measure with density
$\propto (1-t^2)^{(d-3)}$, $t\in [-1,1]$. Recall that for fixed 
$v,w$ of norm $1$, the polynomials $P_{d,j}(\<v,u\>), P_{d,k}(\<w,u\>)$ are spherical harmonics satisfying
\begin{align}
\int P_{d,j}(\<v,u\>)P_{d,k}(\<w,u\>) \, \mu_0(\de u)= \delta_{kj} P_{d,k}(\<v,w\>)
\, .\label{eq:Projection}
\end{align}
Also,  $P_{d,k}(1) = B_{d,k}$ is the dimension of the space of spherical harmonics of degree $k$,
whence $(P_{d,k}( \cdot )/B_{d,k}^{1/2}:\, k\ge 0)$ form an orthonormal set.
We will denote by $c_{d,k}(\sigma)$ the $k$-th coefficient of the expansion of $\sigma( \, . \, \sqrt{d})$ in this basis,
and similarly for $\varphi(\, . \, \sqrt{d})$, with coefficients $c_{d,k}(\varphi)$,
namely
\begin{align*}
\sigma(t\sqrt{d}) &= \sum_{k=0}^{\infty}\frac{c_{d,k}(\sigma)}{B_{d,k}^{1/2}} \, P_{d,k}(t)\, ,\\
\varphi(t\sqrt{d}) &= \sum_{k=0}^{\infty}\frac{c_{d,k}(\varphi)}{B_{d,k}^{1/2}} \, P_{d,k}(t)\, .
\end{align*}
As shown for instance in \cite{mei2022generalization}, $\lim_{d\to\infty}c_{d,k}(\sigma)= c_k(\sigma)$
is the $k$-th Hermite coefficient of $\sigma$ and similarly for $c_{d,k}(\varphi)$.
In particular, $c_{d,k}(\sigma)\neq 0$ for all $d$ large enough.
For $N$ a large integer let 
\begin{align*}
a_d(u) = \sum_{k=0}^N \frac{c_{d,k}(\varphi)}{c_{d,k}(\sigma)}P_{d,k}(\<u,u_*\>)\,. 
\end{align*}
By Eq.~\eqref{eq:Projection}, we have,  for $\|z\|=\sqrt{d}$,
\begin{align*}
\int a_d(u)\,  \sigma(\<u,z\>) \, \mu_0(\de u)& = 
 \sum_{k=0}^N \frac{c_{d,k}(\varphi)}{c_{d,k}(\sigma)} \frac{c_{d,k}(\sigma)}{B_{d,k}^{1/2}} 
 P_{d,k}(\<u_*,z\>/\sqrt{d})\\
 & = \sum_{k=0}^N \frac{c_{d,k}(\varphi)}{B_{d,k}^{1/2}} 
 P_{d,k}(\<u_*,z\>/\sqrt{d})\, .
\end{align*}
Denoting by $z$ a uniform random vector on the sphere of radius $\sqrt{d}$,
and $r =\|x\|_2/\sqrt{d}$, we have
\begin{align*}
\E\big\{\big(\int a_d(u)\,  \sigma(\<u,x\>) \, \mu_0(\de u)-&
 \varphi(\<u_*,x\>)\big)^2\big\}  = \E\big\{\big(\int a_d(u)\,  \sigma(r\<u,z\>) \, \mu_0(\de u)-
 \varphi(r\<u_*,z\>)\big)^2\big\} \\
 & \stackrel{(*)}{\le} \E\big\{\big(\int a_d(u)\,  \sigma(\<u,z\>) \, \mu_0(\de u)-
 \varphi(\<u_*,z\>)\big)^2\big\}
 +\frac{C_NL^2}{d}\\
 & \le  \E\big\{ \varphi_{>N}(\<u_*,z\>)^2\big\}
 +\frac{C_NL^2}{d}\, ,
 \end{align*}
 where in $(*)$ we used concentration of $\chi$-squared random variables, Lipschitzness of $\sigma$ and $\varphi$, and that
 $\varphi_{>N}(t\sqrt{d} )$ is the projection of $\varphi(t \sqrt{d} )$ 
 orthogonal to polynomials of degree at most $N$ (with respect to the measure with density proportional to $(1-t^2)^{(d-3)/2}$ on $[-1, 1]$).
 Therefore
 \begin{align*}
 \limsup_{d\to\infty} \E\big\{\big(\int a_d(u)\,  \sigma(\<u,x\>) \, \mu_0(\de u)-
 \varphi(\<u_*,x\>)\big)^2\big\}\le
 \sum_{k=N+1}^{\infty}c_k(\varphi)^2\,.
\end{align*}
The claim \eqref{eq:Representation} follows by taking $N\to\infty$.}

\section{Appendix to Section \ref{sec:LargeNet}}

\subsection{Proof of Proposition \ref{prop:dynamics}}
\label{sec:proof-prop-dynamics}


When $x \sim \sN(0,I_d)$ and $u,u' \in \S^{d-1}$, $\begin{pmatrix}
\langle u,x\rangle \\ \langle u', x \rangle
\end{pmatrix} \sim \sN\left(0, \begin{pmatrix}
1 & \langle u , u' \rangle \\
\langle u , u' \rangle & 1
\end{pmatrix}\right)$. Thus
\begin{align}
\Risk (a, u) &= \frac{1}{2} \E \left(\varphi(\langle u_*, x \rangle) - \frac{1}{m} \sum_{i=1}^{m} a_i \sigma(\langle u_i, x \rangle)\right)^2 \nonumber\\
&=  \frac{1}{2} \E \left[ \varphi(\langle u_*, x \rangle)^2 \right]- \frac{1}{m} \sum_{i = 1}^m a_i \E \left[\varphi(\langle u_*, x \rangle)  \sigma(\langle u_i, x \rangle) \right] + \frac{1}{2} \frac{1}{m^2} \sum_{i, j = 1}^m a_i a_j \E\left[ \sigma(\langle u_i, x \rangle) \sigma(\langle u_j, x \rangle)\right] \nonumber\\
&=  \frac{1}{2} \Vert \varphi \Vert^2_{L^2} - \frac{1}{m} \sum_{i = 1}^m a_i V(\langle u_* , u_i \rangle) + \frac{1}{2} \frac{1}{m^2} \sum_{i, j = 1}^m a_i a_j U(\langle u_i , u_j \rangle) \label{eq:aux-risk} \\
&= \frac{1}{2} \Vert \varphi \Vert^2_{L^2} - \frac{1}{m} \sum_{i = 1}^m a_i V(s_i) + \frac{1}{2} \frac{1}{m^2} \sum_{i, j = 1}^m a_i a_j U(r_{ij}) \, . \nonumber
\end{align}
This proves \eqref{eq:risk}. Equation \eqref{eq:dynamics-1} follows directly:
\begin{align*}
\veps \partial_t a_i &= - m \partial_{a_i} \Risk (a, u)  =  V(s_i) - \frac{1}{m} \sum_{j=1}^{m} a_j U (r_{ij}) \, .
\end{align*}
To obtain equations \eqref{eq:dynamics-2}-\eqref{eq:dynamics-4}, we now take gradients in \eqref{eq:aux-risk}:
\begin{align*}
\partial_t u_i &= - m  (I_d - u_i u_i^\top) \nabla_{u_i} \Risk (a, u) \\
&= a_i \left( I_d - u_i u_i^\top \right) \left( V' (\langle u_*, u_i \rangle) u_* - \frac{1}{m} \sum_{j=1}^{m} a_j U' (\langle u_i, u_j \rangle) u_j \right) \\
&= a_i \left( V' (\langle u_*, u_i \rangle) (u_* - u_i u_i^\top u_*) - \frac{1}{m} \sum_{j=1}^{m} a_j U' (\langle u_i, u_j \rangle) (u_j - u_i u_i^\top  u_j) \right) \\
&= a_i \left( V' (s_i) (u_* - s_i u_i) - \frac{1}{m} \sum_{j=1}^{m} a_j U' (r_{ij}) (u_j - r_{ij} u_i) \right) \, .
\end{align*}
Thus
\begin{align*}
\partial_t s_i &= \langle u_*, \partial_t u_i \rangle \\
&= a_i \left( V' (s_i) (\langle u_*, u_* \rangle - s_i \langle u_*, u_i \rangle) - \frac{1}{m} \sum_{j=1}^{m} a_j U' (r_{ij}) (\langle u_*, u_j \rangle  - r_{ij} \langle u_*, u_i\rangle) \right) \\
&= a_i \left( V' (s_i) (1 - s_i^2) - \frac{1}{m} \sum_{j=1}^{m} a_j U' (r_{ij}) (s_j  - r_{ij} s_i) \right) \, .
\end{align*}
This gives \eqref{eq:dynamics-2}. Finally, we perform a similar computation to compute 
$
\partial_t r_{ij} = \langle \partial_t u_i, u_j \rangle + \langle u_i, \partial_t u_j \rangle 
$. We compute only the first term, as the second term can be obtained by inverting $i$ and $j$: 
\begin{align*}
\langle \partial_t u_i, u_j \rangle 
&= a_i \left( V' (s_i) (\langle u_j, u_* \rangle - s_i \langle u_j, u_i \rangle) - \frac{1}{m} \sum_{p=1}^{m} a_p U' (r_{ip}) (\langle u_j, u_p \rangle  - r_{ip} \langle u_j, u_i\rangle) \right) \\
&= a_i \left( V' (s_i) (s_j - s_i r_{ij}) - \frac{1}{m} \sum_{p=1}^{m} a_p U' (r_{ip}) (r_{jp}  - r_{ip} r_{ij}) \right) \, .
\end{align*}
Adding the symmetric term $\langle  u_i, \partial_t u_j \rangle$, we obtain \eqref{eq:dynamics-3}-\eqref{eq:dynamics-4}.

\subsection{Proof of Corollary~\ref{coro:high_dim_lim}}\label{sec:proof_only_coro}
First, note that in the proof of Lemma~\ref{lem:a_bound_r_perp}, we obtain the following a priori estimate on the magnitude of the $a_i^0$'s:
\begin{equation}
	\sup_{1 \le i \le m} \left\vert a_i^0 (t) \right\vert \le M \left( 1 + \frac{t}{\veps} \right), \ \forall t \ge 0,
\end{equation}
where $M$ only depends on the $M_i$'s in Assumptions A1-A3. Using a similar argument as that in the proof of Proposition~\ref{prop:coupling_W2_bd}, we obtain that for any $t \in [0,T]$ and $i \in [m]$,
\begin{align*}
	\left\vert \partial_t (a_i - a_i^0) \right\vert \le\, & \frac{M}{\veps} \left( \left\vert s_i - s_i^0 \right\vert + \frac{1}{m} \sum_{j=1}^{m} \left\vert a_j - a_j^0 \right\vert \right) + \frac{M (1 + t / \veps)}{\veps} \cdot \frac{1}{m} \sum_{j=1}^{m} \left\vert r_{ij} - r_{ij}^0 \right\vert, \\
	\left\vert \partial_t (s_i - s_i^0) \right\vert \le\, & M (1 + t / \veps) \cdot \left( \left\vert a_i - a_i^0 \right\vert + \frac{1}{m} \sum_{j=1}^{m} \left\vert a_j - a_j^0 \right\vert \right) \\
	& + M(1 + t / \veps)^2 \cdot \left( \left\vert s_i - s_i^0 \right\vert + \frac{1}{m} \sum_{j=1}^{m} \left( \left\vert s_j - s_j^0 \right\vert + \left\vert r_{ij} - r_{ij}^0 \right\vert \right) \right),
\end{align*}
and for $1 \le i \neq j \le m$,
\begin{align*}
	\left\vert \partial_t (r_{ij} - r_{ij}^0) \right\vert \le\, & M (1 + t / \veps) \left( \left\vert a_i - a_i^0 \right\vert + \left\vert a_j - a_j^0 \right\vert + \left\vert s_i - s_i^0 \right\vert + \left\vert s_j - s_j^0 \right\vert + \frac{1}{m} \sum_{p=1}^{m} \left\vert a_p - a_p^0 \right\vert \right) \\
	& + M(1 + t / \veps)^2 \cdot \left( \left\vert r_{ij} - r_{ij}^0 \right\vert + \frac{1}{m} \sum_{p=1}^{m} \left( \left\vert r_{ip} - r_{ip}^0 \right\vert + \left\vert r_{jp} - r_{jp}^0 \right\vert \right) \right).
\end{align*}
Therefore, we deduce that
\begin{align*}
	\partial_t \sum_{i=1}^{m} (a_i - a_i^0)^2 \le\, & \frac{M}{\veps} \sum_{i=1}^{m} (s_i - s_i^0)^2 + \frac{M (1 + t / \veps)}{\veps} \left( \sum_{i=1}^{m} (a_i - a_i^0)^2 + \frac{1}{m} \sum_{i, j =1}^{m} (r_{ij} - r_{ij}^{0})^2 \right), \\
	\partial_t \sum_{i=1}^{m} (s_i - s_i^0)^2 \le\, & M(1 + t / \veps) \sum_{i=1}^{m} (a_i - a_i^0)^2 + M(1 + t / \veps)^2 \left( \sum_{i=1}^{m} (s_i - s_i^0)^2 + \frac{1}{m} \sum_{i, j =1}^{m} (r_{ij} - r_{ij}^{0})^2 \right), \\
	\partial_t \left( \frac{1}{m} \sum_{i, j =1}^{m} (r_{ij} - r_{ij}^{0})^2 \right) \le\, & M(1 + t/ \veps) \left( \sum_{i=1}^{m} (a_i - a_i^0)^2 + \sum_{i=1}^{m} (s_i - s_i^0)^2 \right) + M(1 + t/\veps)^2 \cdot \frac{1}{m} \sum_{i, j =1}^{m} (r_{ij} - r_{ij}^{0})^2.
\end{align*}
Defining
\begin{equation*}
	G(t) = \sum_{i=1}^{m} (a_i (t) - a_i^0 (t))^2 + \sum_{i=1}^{m} (s_i (t) - s_i^0 (t))^2 + \frac{1}{m} \sum_{i, j =1}^{m} (r_{ij} (t) - r_{ij}^{0} (t) )^2,
\end{equation*}
then we know that $G'(t) \le (M (1 + t)^2 / \veps^2) G(t)$. Applying Gr\"{o}nwall's inequality yields
\begin{equation*}
	G(t) \le G(0) \exp \left( \int_{0}^{t} (M (1 + s)^2 / \veps^2) {\rm d} s \right) \le G(0) \exp \left( M t (1 + t)^2 / \veps^2 \right), \ \forall t \in [0, T].
\end{equation*}
Since $\{ \langle u_i (0), u_* \rangle \}_{i \in [m]} \sim_{\iid} \normal (0, 1/d)$ and for any $i \in [m]$, $\{ \langle u_i (0), u_j(0) \rangle \}_{j \neq i} \sim_{\iid} \normal (0, 1/d)$. Using standard concentration inequalities, we know that
\begin{equation}
	G(0) = \sum_{i=1}^{m} \langle u_i (0), u_* \rangle^2 + \frac{1}{m} \sum_{i \neq j} \langle u_i (0), u_j (0) \rangle^2 \le C \frac{m}{d}
\end{equation}
with probability at least $1 - \exp(C' m)$, where $C$ and $C'$ are both absolute constants. Therefore,
\begin{align}
	\sup_{t \in [0,T]}\|a(t)-a^0(t)\|_2 \le\, & C \sqrt{\frac{m}{d}} \exp \left( M T (1 + T)^2 / \veps^2 \right) , \\
	\sup_{t \in [0,T]}\|s(t)-s^0(t)\|_2 \le\, & C \sqrt{\frac{m}{d}} \exp \left( M T (1 + T)^2 / \veps^2 \right) , \\
	\sup_{t \in [0,T]}\|R(t)-R^0(t)\|_{\rm F} \le\, & C \frac{m}{\sqrt{d}} \exp \left( M T (1 + T)^2 / \veps^2 \right) .
\end{align}
Next we upper bound the risk difference, by direct calculation,
\begin{align*}
	& \big|\Risk(a(t),u(t))- \Risk_{\red} (a^0(t), s^0(t), R^0(t))\big| \\
	=\, & \big|\Risk_{\red}(a(t),s(t), R(t))- \Risk_{\red} (a^0(t), s^0(t), R^0(t))\big| \\
	\le\, & \frac{1}{m} \left\vert a^\top V(s) - (a^0)^\top V(s^0) \right\vert + \frac{1}{2 m^2} \left\vert a^\top U(R) a - (a^0)^\top U(R^0) a^0 \right\vert \\
	\le\, & \frac{1}{m} \left( \sqrt{m} \norm{V}_{\infty} \norm{a - a^0}_2 + \norm{a^0}_2 \norm{V'}_{\infty} \norm{s - s^0}_2 \right) \\
	& + \frac{1}{2 m^2} \left( \norm{U(R) - U(R^0)}_{\op} \norm{a^0}_2^2 + 2 \norm{U(R)}_{\op} \norm{a^0}_2 \norm{a - a^0}_2 \right) \\
	\le\, & \frac{M(1 + t / \veps)}{\sqrt{m}} \left( \norm{a - a^0}_2 + \norm{s - s^0}_2 \right) + \frac{M(1 + t / \veps)^2}{2m} \norm{R - R^0}_{\rm F} \\
	\le\, & \frac{M}{\sqrt{d}} \exp \left( M t (1 + t)^2 / \veps^2 \right)
\end{align*}
with probability at least $1 - \exp(-C' m)$, where the constant $M$ only depends on the $M_i$'s from Assumptions A1-A3. The conclusion now follows from taking the supremum over all $t \in [0, T]$. This completes the proof of Corollary~\ref{coro:high_dim_lim}.

\subsection{Proof of Proposition \ref{prop:coupling_W2_bd}}\label{sec:proof_MF}

We consider $r_{ij}^\perp = r_{ij} - s_i s_j = \langle u_i, u_j \rangle - \langle u_i, u_* \rangle \langle u_*, u_j \rangle$, the dot product between $u_i$ and $u_j$ that is out of the relevant subspace spanned by $u_*$. We show that these variables satisfy the ODEs 
\begin{equation}\label{eq:dynamics_perp}
\begin{split}
\partial_t r_{ij}^\perp = \, & - a_i \left( V'(s_i) \cdot s_i r_{ij}^\perp + \frac{1}{m} \sum_{p=1}^{m} a_p U'(r_{ip}) (r_{jp}^\perp - r_{ip} r_{ij}^\perp) \right) \\
\, & - a_j \left( V'(s_j) \cdot s_j r_{ij}^\perp + \frac{1}{m} \sum_{p=1}^{m} a_p U'(r_{jp}) (r_{ip}^\perp - r_{jp} r_{ij}^\perp) \right) \, .
\end{split}
\end{equation}
By definition of $r_{ij}^{\perp}$, we readily see that
\begin{equation*}
	\partial_t r_{ij}^{\perp} = \partial_t r_{ij} - s_i \partial_t s_j - s_j \partial_t s_i.
\end{equation*}
Plugging in Eq.s~\eqref{eq:dynamics-2} to \eqref{eq:dynamics-4} gives that
\begin{align*}
	\partial_t r_{ij}^{\perp} = \, & a_i \left( V'(s_i) (s_j s_i^2 - s_i r_{ij}) - \frac{1}{m} \sum_{p=1}^{m} a_p U'(r_{ip}) (r_{jp} - s_j s_p - r_{ip} r_{ij} + s_i s_j r_{ip}) \right) \\
	& + a_j \left(  V'(s_j) (s_i s_j^2 - s_j r_{ij}) - \frac{1}{m} \sum_{p=1}^{m} a_p U'(r_{jp}) (r_{ip} - s_i s_p - r_{jp} r_{ij} + s_i s_j r_{jp}) \right) \\
	=\, & - a_i \left( V'(s_i) s_i r_{ij}^{\perp} + \frac{1}{m} \sum_{p=1}^{m} a_p U'(r_{ip}) (r_{jp}^{\perp} - r_{ip} r_{ij}^{\perp} ) \right) \\
	& - a_j \left( V'(s_j) s_j r_{ij}^{\perp} + \frac{1}{m} \sum_{p=1}^{m} a_p U'(r_{jp}) (r_{ip}^{\perp} - r_{jp} r_{ij}^{\perp} ) \right).
\end{align*}
This proves Eq.~\eqref{eq:dynamics_perp}.

\begin{lem}\label{lem:a_bound_r_perp}
	If Assumptions A1-A3 hold, then we have for any fixed $T > 0$:
	\begin{equation*}
		\sup_{t \in [0, T]} \sum_{i, j = 1}^{m} r_{ij}^{\perp} (t)^2 \le m \exp \left( M T \left( 1 + T \right)^2 / \veps^2 \right).
	\end{equation*}
\end{lem}

\begin{proof}
	To begin with, using Eq.~\eqref{eq:dynamics_perp}, we obtain that
	\begin{align*}
	\partial_t \left( \sum_{i, j = 1}^{m} r_{ij}^{\perp} (t)^2 \right) &= 2 \sum_{i, j = 1}^{m} r_{ij}^{\perp} \times \partial_t r_{ij}^{\perp} \\
	&= - 4 \sum_{i, j = 1}^{m} a_i r_{ij}^{\perp} \left( V'(s_i) \cdot s_i r_{ij}^\perp + \frac{1}{m} \sum_{p=1}^{m} a_p U'(r_{ip}) (r_{jp}^\perp - r_{ip} r_{ij}^\perp) \right).
	\end{align*}
    Using the ODEs for the $a_i$'s, we obtain that
    \begin{align*}
    	\left\vert \partial_t a_i \right\vert =\, & \frac{1}{\veps} \left\vert V(s_i) - \frac{1}{m} \sum_{j=1}^{m} a_j U(r_{ij}) \right\vert = \frac{1}{\veps} \left\vert \E \left[ \varphi(\langle u_*, x \rangle) \sigma (\langle u_i, x \rangle) \right] - \frac{1}{m} \sum_{j=1}^{m} a_j \E \left[ \sigma (\langle u_i, x \rangle) \sigma (\langle u_j, x \rangle ) \right] \right\vert \\
    	=\, & \frac{1}{\veps} \left\vert \E \left[ \sigma (\langle u_i, x \rangle) \left( y - f(x; a, u) \right) \right] \right\vert \le \frac{1}{\veps} \E \left[ \sigma(\langle u_i, x \rangle)^2 \right]^{1/2} \E \left[ \left( y - f(x; a, u) \right)^2 \right]^{1/2} \\
    	\stackrel{(i)}{\le}\, & \frac{1}{\veps} M_2 \sqrt{2 \Risk(a(0), u(0))} \le \frac{M}{\veps},
    \end{align*}
    where $(i)$ follows from our assumptions and the fact that $\Risk(a(t), u(t)) \le \Risk(a(0), u(0))$, since $\partial_t \Risk (a, u) \le 0$ by gradient flow equations. Moreover, the constant $M$ only depends on the $M_i$'s. Since $\vert a_i (0) \vert \le M_1$ for all $i \in [m]$, we know that $\vert a_i (t) \vert \le M (1 + t / \veps)$ for all $t \ge 0$, thus leading to the following estimate:
	\begin{align*}
	\partial_t \left( \sum_{i, j = 1}^{m} r_{ij}^{\perp} (t)^2 \right) \le \, & 4 \sum_{i, j = 1}^{m} M \left( 1 + \frac{t}{\veps} \right) \left\vert r_{ij}^{\perp} \right\vert \left( \norm{V'}_{\infty} \left\vert r_{ij}^{\perp} \right\vert + M \left( 1 + \frac{t}{\veps} \right) \norm{U'}_{\infty} \left( \frac{1}{m} \sum_{p=1}^{m} \left\vert r_{jp}^{\perp} \right\vert + \left\vert r_{ij}^{\perp} \right\vert \right) \right) \\
	\le \, & M \left( 1 + \frac{t}{\veps} \right)^2 \left( \sum_{i, j = 1}^{m} r_{ij}^{\perp} (t)^2 + \frac{1}{m} \sum_{i, j, p = 1}^{m} \left\vert r_{ij}^{\perp} (t) \right\vert \cdot \left\vert r_{jp}^{\perp} (t) \right\vert \right) \\
	\le \, & M \left( 1 + \frac{t}{\veps} \right)^2 \left( \sum_{i, j = 1}^{m} r_{ij}^{\perp} (t)^2 + \frac{1}{2 m} \sum_{i, j, p = 1}^{m} \left( r_{ij}^{\perp} (t)^2 + r_{jp}^{\perp} (t)^2 \right) \right) \\
	\le \, & M \left( 1 + \frac{t}{\veps} \right)^2 \sum_{i, j = 1}^{m} r_{ij}^{\perp} (t)^2,
	\end{align*}
	where the constant $M$ only depends on the $M_i$'s in our assumptions. At initialization, we know that $\sum_{i, j = 1}^{m} r_{ij}^{\perp} (0)^2 = m$. Applying Gr\"{o}nwall's inequality yields that
	\begin{equation*}
	\sum_{i, j = 1}^{m} r_{ij}^{\perp} (t)^2 \le m \exp \left( \int_{0}^{t} M \left( 1 + \frac{s}{\veps} \right)^2 \mathrm{d} s \right), \quad \forall t \in [0, T],
	\end{equation*}
	which further implies that
	\begin{equation*}
	\sup_{t \in [0, T]} \sum_{i, j = 1}^{m} r_{ij}^{\perp} (t)^2 \le m \exp \left( \int_{0}^{T} M \left( 1 + \frac{t}{\veps} \right)^2 \mathrm{d} t \right) \le m \exp \left( M T \left( 1 + T \right)^2 / \veps^2 \right) .
	\end{equation*}
	This completes the proof.
\end{proof}

	We show that
	\begin{equation}\label{eq:W2_bd_target}
	\sup_{t \in [0, T]} \sum_{i=1}^{m} \left( \left( a_i (t) - a_i^{\mf} (t) \right)^2 + \left( s_i (t) - s_i^{\mf} (t) \right)^2 \right) \le C(T).
	\end{equation}
	To this end, we define $S(t) = \sum_{i=1}^{m} \left( \left( a_i (t) - a_i^{\mf} (t) \right)^2 + \left( s_i (t) - s_i^{\mf} (t) \right)^2 \right)$. By our assumption, $S(0) = 0$. Moreover, using the same technique as in the proof of Lemma~\ref{lem:a_bound_r_perp}, we know that $\vert a_i^{\mf} (t) \vert \le M (1 + t / \veps)$ for all $i \in [m]$. According to Eq.s~\eqref{eq:dynamics-1}-\eqref{eq:dynamics-4} and Eq.~\eqref{eq:coupled_MF}, we deduce that
	\begin{align*}
	& \left\vert \partial_t (a_i - a_i^{\mf} ) \right\vert = \, \frac{1}{\veps} \cdot \left\vert V(s_i) - V(s_i^{\mf}) - \frac{1}{m} \sum_{j=1}^{m} \left( a_j U(r_{ij}) - a_j^{\mf} U (s_i^{\mf} s_j^{\mf}) \right) \right\vert \\
	\le\, & \frac{1}{\veps} \cdot \left( \left\vert V(s_i) - V(s_i^{\mf}) \right\vert + \frac{1}{m} \sum_{j=1}^{m} \left\vert a_j U(r_{ij}) - a_j U(s_i s_j) \right\vert + \frac{1}{m} \sum_{j=1}^{m} \left\vert a_j U(s_i s_j) - a_j^{\mf} U (s_i^{\mf} s_j^{\mf}) \right\vert \right) \\
	\le\, & \frac{1}{\veps} \cdot \left( \norm{V'}_{\infty} \vert s_i - s_i^{\mf} \vert + \frac{M (1 + t / \veps )}{m} \norm{U'}_{\infty} \sum_{j=1}^{m} \left\vert r_{ij}^{\perp} \right\vert \right) \\
	& + \frac{1}{m \veps} \sum_{j=1}^{m} \left( \norm{U}_{\infty} \vert a_j - a_j^{\mf} \vert + M \left( 1 + \frac{t}{\veps} \right) \norm{U'}_{\infty} \left( \vert s_i - s_i^{\mf} \vert + \vert s_j - s_j^{\mf} \vert \right) \right) \\
	\le\, & \frac{M}{\veps} \left( 1 + \frac{t}{\veps} \right) \cdot \left( \vert s_i - s_i^{\mf} \vert + \frac{1}{m} \sum_{j=1}^{m} \left( \vert s_j - s_j^{\mf} \vert + \vert a_j - a_j^{\mf} \vert + \left\vert r_{ij}^{\perp} \right\vert \right) \right),
	\end{align*}
	thus leading to the following estimate:
	\begin{align*}
	\sum_{i = 1}^{m} (a_i - a_i^{\mf}) \cdot \partial_t (a_i - a_i^{\mf}) \le\, & \frac{M}{\veps} \left( 1 + \frac{t}{\veps} \right) \cdot \frac{1}{m} \sum_{i = 1}^{m} \vert a_i - a_i^{\mf} \vert \cdot \sum_{j = 1}^{m} \left( \vert a_j - a_j^{\mf} \vert + \vert s_j - s_j^{\mf} \vert \right) \\
	& + \frac{M}{\veps} \left( 1 + \frac{t}{\veps} \right) \cdot \left( \sum_{i = 1}^{m} \vert a_i - a_i^{\mf} \vert \vert s_i - s_i^{\mf} \vert + \frac{1}{m} \sum_{i, j = 1}^{m} \vert a_i - a_i^{\mf} \vert \left\vert r_{ij}^{\perp} \right\vert \right) \\
	\stackrel{(i)}{\le}\, & \frac{M}{\veps} \left( 1 + \frac{t}{\veps} \right) \cdot \left( \sum_{i = 1}^{m} (a_i - a_i^{\mf})^2 + \sum_{i=1}^{m} (s_i - s_i^{\mf})^2 + \frac{1}{m} \sum_{i, j = 1}^{m} \left( r_{ij}^{\perp} \right)^2 \right) \\
	\stackrel{(ii)}{\le}\, & \frac{M}{\veps} \left( 1 + \frac{t}{\veps} \right) \cdot \left( \sum_{i = 1}^{m} (a_i - a_i^{\mf})^2 + \sum_{i=1}^{m} (s_i - s_i^{\mf})^2 + \exp \left( Mt (1 + t)^2 / \veps^2 \right) \right),
	\end{align*}
	where in $(i)$ we use the Cauchy-Schwarz inequality and the inequality of arithmetic and geometric means, and $(ii)$ follows from the conclusion of Lemma~\ref{lem:a_bound_r_perp}. Similarly, we obtain that
	\begin{align*}
	\left\vert \partial_t (s_i - s_i^{\mf}) \right\vert \le\, & \norm{V'}_{\infty} \vert a_i - a_i^{\mf} \vert + M \left( 1 + \frac{t}{\veps} \right) \left( \norm{V''}_{\infty} + 2 \norm{V'}_{\infty} \right) \vert s_i - s_i^{\mf} \vert \\
	& + \frac{1}{m} \left\vert a_i \sum_{j=1}^{m} a_j U'(r_{ij}) (s_j - r_{ij} s_i) - a_i \sum_{j = 1}^{m} a_j U'(s_i s_j) (1 - s_i^2) s_j \right\vert \\
	& + \frac{1}{m} \left\vert a_i \sum_{j = 1}^{m} a_j U'(s_i s_j) (1 - s_i^2) s_j - a_i^{\mf} \sum_{j = 1}^{m} a_j^{\mf} U'(s_i^{\mf} s_j^{\mf}) (1 - (s_i^{\mf})^2) s_j^{\mf} \right\vert \\
	\le\, & M \left( 1 + \frac{t}{\veps} \right) \cdot \left( \vert a_i - a_i^{\mf} \vert + \vert s_i - s_i^{\mf} \vert \right) + \frac{2}{m} \cdot M \left( 1 + \frac{t}{\veps} \right) \left( \norm{U'}_{\infty} + \norm{U''}_{\infty} \right) \sum_{j = 1}^{m} \vert a_j \vert \left\vert r_{ij}^{\perp} \right\vert \\
	& + M \left( 1 + \frac{t}{\veps} \right)^2 \cdot \frac{1}{m} \sum_{j=1}^{m} \left( \vert a_j - a_j^{\mf} \vert + \vert s_j - s_j^{\mf} \vert \right) \\
	\le\, & M \left( 1 + \frac{t}{\veps} \right)^2 \cdot \left( \vert a_i - a_i^{\mf} \vert + \vert s_i - s_i^{\mf} \vert + \frac{1}{m} \sum_{j=1}^{m} \left( \vert s_j - s_j^{\mf} \vert + \vert a_j - a_j^{\mf} \vert + \left\vert r_{ij}^{\perp} \right\vert \right) \right),
	\end{align*}
	which further implies that
	\begin{equation*}
	\sum_{i=1}^{m} (s_i - s_i^{\mf}) \cdot \partial_t (s_i - s_i^{\mf}) \le M \left( 1 + \frac{t}{\veps} \right)^2 \cdot \left( \sum_{i = 1}^{m} (a_i - a_i^{\mf})^2 + \sum_{i=1}^{m} (s_i - s_i^{\mf})^2 + \exp \left( Mt (1 + t)^2 / \veps^2 \right) \right).
	\end{equation*}
	Combining the above estimates, we finally deduce that
	\begin{align*}
	S'(t) =\, & 2 \sum_{i = 1}^{m} \left( (a_i - a_i^{\mf}) \cdot \partial_t (a_i - a_i^{\mf}) + (s_i - s_i^{\mf}) \cdot \partial_t (s_i - s_i^{\mf}) \right) \\
	\le\, & \frac{M (1 + t) (1 + t / \veps)}{\veps} \cdot \left( \sum_{i = 1}^{m} (a_i - a_i^{\mf})^2 + \sum_{i=1}^{m} (s_i - s_i^{\mf})^2 + \exp \left( Mt (1 + t)^2 / \veps^2 \right) \right) \\
	=\, & \frac{M (1 + t) (1 + t / \veps)}{\veps} \cdot \left( S(t) + \exp \left( Mt (1 + t)^2 / \veps^2 \right) \right).
	\end{align*}
	Applying Gr\"{o}nwall's inequality immediately implies
	\begin{equation}
		S(t) \le \exp \left( Mt (1 + t)^2 / \veps^2 + \int_{0}^{t} \frac{M (1 + s) (1 + s / \veps)}{\veps} \mathrm{d} s \right) \le \exp \left( Mt (1 + t)^2 / \veps^2 \right),
	\end{equation}
	 which further leads to Eq.~\eqref{eq:W2_bd_target} and concludes the proof of Proposition~\ref{prop:coupling_W2_bd}. The ``consequently'' part can be shown via direct calculation, but we include it here for the sake of completeness. By definition, for any $t \in [0, T]$ we have
	 \begin{align*}
	 	& \left\vert \Risk_{\red} \left( a(t), s(t), R(t) \right) - 
	 	\Risk_{\mf} \left( \amf(t), \smf(t)\right) \right\vert \\
	 	\le\, & \frac{1}{m} \left\vert \sum_{i=1}^{m} a_i V(s_i) - \sum_{i=1}^{m} a_i^{\mf} V(s_i^{\mf}) \right\vert + \frac{1}{2 m^2} \left\vert \sum_{i, j = 1}^{m} a_i a_j U(r_{ij}) - \sum_{i, j=1}^{m} a_i^{\mf} a_j^{\mf} U(s_i^{\mf} s_j^{\mf}) \right\vert \\
	 	\le\, & \frac{1}{m} \sum_{i=1}^{m} \left( \norm{V}_{\infty} \left\vert a_i - a_i^{\mf} \right\vert + M(1 + t / \veps) \norm{V'}_{\infty} \left\vert s_i - s_i^{\mf} \right\vert \right) + \frac{1}{2 m^2} M (1 + t / \veps)^2 \norm{U'}_{\infty} \sum_{i, j = 1}^{m} \left\vert r_{ij}^{\perp} \right\vert \\
	 	& + \frac{1}{2 m^2} \sum_{i, j=1}^{m} \left( M(1 + t / \veps) \norm{U}_{\infty} \left( \left\vert a_i - a_i^{\mf} \right\vert + \left\vert a_j - a_j^{\mf} \right\vert  \right) + M(1 + t / \veps)^2 \norm{U'}_{\infty} \left( \left\vert s_i - s_i^{\mf} \right\vert + \left\vert s_j - s_j^{\mf} \right\vert  \right) \right) \\
	 	\le\, & M (1 + t / \veps) \frac{1}{m} \sum_{i=1}^{m} \left\vert a_i - a_i^{\mf} \right\vert + M (1 + t / \veps)^2 \frac{1}{m} \sum_{i=1}^{m} \left\vert s_i - s_i^{\mf} \right\vert + M (1 + t / \veps)^2 \frac{1}{m^2} \sum_{i, j=1}^{m} \left\vert r_{ij}^{\perp} \right\vert \\
	 	\le\, & M (1 + t / \veps)^2 \cdot \left( \sqrt{\frac{1}{m} \sum_{i=1}^{m} \left( \left( a_i (t) - a_i^{\mf} (t) \right)^2 + \left( s_i (t) - s_i^{\mf} (t) \right)^2 \right) } + \sqrt{\frac{1}{m^2} \sum_{i, j=1}^{m} \left\vert r_{ij}^{\perp} (t) \right\vert^2 } \right) \\
	 	\le\, & \frac{1}{\sqrt{m}} M (1 + t / \veps)^2 \exp \left( M t (1 + t)^2 / \veps^2 \right).
	 \end{align*}
     Therefore,
     \begin{equation}
     	\sup_{t \in [0, T]} \left\vert \Risk_{\red} \left( a(t), s(t), R(t) \right) - 
     	\Risk_{\mf} \left( \amf(t), \smf(t)\right) \right\vert \le \frac{M \exp \left( M T (1 + T)^2 / \veps^2 \right)}{\sqrt{m}},
     \end{equation}
     as desired.

\subsection{Derivation of the mean field dynamics~\eqref{eq:MF_dynamics_rho_bar}}\label{sec:derive_MF}
For any bounded continuous $f \in C_b (\R^2)$, we have
\begin{align*}
& \int_{\R^2} f (a, s) \partial_t \rho_t (\mathrm{d} a, \mathrm{d} s) = \partial_t \left( \int_{\R^2} f (a, s) \rho_t (\mathrm{d} a, \mathrm{d} s) \right) = \partial_t \left( \frac{1}{m} \sum_{i=1}^{m} f(a_i^{\mf} (t), s_i^{\mf} (t)) \right) \\
=\, & \frac{1}{m} \sum_{i=1}^{m} \left( \partial_a f (a_i^{\mf} (t), s_i^{\mf} (t)) \cdot \partial_t a_i^{\mf} (t) + \partial_s f (a_i^{\mf} (t), s_i^{\mf} (t)) \cdot \partial_t s_i^{\mf} (t) \right) \\
\stackrel{(i)}{=}\, & \frac{1}{m} \sum_{i=1}^{m} \left( \partial_a f (a_i^{\mf} (t), s_i^{\mf} (t)) \cdot \Psi_a \left(a_i^{\mf} (t), s_i^{\mf}(t); \rho_t \right) + \partial_s f (a_i^{\mf} (t), s_i^{\mf} (t)) \cdot \Psi_s \left(a_i^{\mf} (t), s_i^{\mf}(t); \rho_t \right) \right) \\
=\, & \int_{\R^2} \left( \partial_a f (a, s) \cdot \Psi_a (a, s; \rho_t) + \partial_s f (a, s) \cdot \Psi_s (a, s; \rho_t) \right) \rho_t (\mathrm{d} a, \mathrm{d} s) \\
\stackrel{(ii)}{=}\, & - \int_{\R^2} f(a, s) \cdot \left( \partial_a \left( \rho_t \Psi_a (a, s; \rho_t) \right) + \partial_s \left( \rho_t \Psi_s (a, s; \rho_t) \right) \right) (\mathrm{d} a, \mathrm{d} s),
\end{align*}
where $(i)$ follows from the ODE satisfied by the $(a_i^{\mf} (t), s_i^{\mf} (t))$'s, and in $(ii)$ we use integration by parts. We thus obtain that
\begin{equation*}
\partial_t \rho_t = - \left( \partial_a \left( \rho_t \Psi_a (a, s; \rho_t) \right) + \partial_s \left( \rho_t \Psi_s (a, s; \rho_t) \right) \right) = - \nabla \cdot \left( \rho_t \Psi (a, s; \rho_t) \right),
\end{equation*}
which recovers Eq.~\eqref{eq:MF_dynamics_rho_bar}.

\subsection{Details of the alternative mean field approach}
\label{sec:details}

Let
\begin{equation}
	\label{eq:def-rho}
	\overline{\rho}_t = \frac{1}{m} \sum_{i=1}^m \delta_{(a_i (t), u_i (t))} \, ,
\end{equation} 
where $(a_i(t), u_i(t))_{1 \leq i \leq m}$ is the solution of \eqref{eq:GF-1}--\eqref{eq:GF-2}. $\overline{\rho}_t$ is a measure on $\R \times \S^{d-1}$ solving the continuity PDE 
\begin{align}\label{eq:MF_dynamics_rho_u}
	\begin{split}
	\partial_t \overline{\rho}_t (a, u) &= - \nabla \cdot \left( \overline{\rho}_t \overline{\Psi} \left( a, u; \overline{\rho}_t \right) \right) \\
	&= - \left( \partial_a \left( \overline{\rho}_t \overline{\Psi}_a \left( a, u; \overline{\rho}_t \right) \right) + \partial_u \left( \overline{\rho}_t \overline{\Psi}_u \left( a, u; \overline{\rho}_t \right) \right) \right) \, ,
	\end{split}
\end{align}
where $\overline{\Psi} = (\overline{\Psi}_a, \overline{\Psi}_u)$ is given by 
\begin{align*}
	\overline{\Psi}_a \left( a, u; \overline{\rho} \right) &= \varepsilon^{-1} \left(V(\langle u , u_* \rangle) - \int_{\R \times \S^{d-1}} a_1 U(\langle u , u_1 \rangle) \overline{\rho}(\diff a_1, \diff u_1) \right) \, , \\
	\overline{\Psi}_u \left( a, u; \overline{\rho} \right) &= a\left(I_d - u u^\top\right) \left(V'(\langle u , u_* \rangle)u_* - \int_{\R \times \S^{d-1}} a_1 U'(\langle u , u_1 \rangle)u_1 \overline{\rho}(\diff a_1, \diff u_1) \right) \, . 
\end{align*}
A remarkable property of the equation \eqref{eq:MF_dynamics_rho_u} is that it preserves invariance to rotations orthogonal to $u_*$. Indeed, assume that $\overline{\rho}$ is invariant to rotations orthogonal to $u_*$. In this case, we show that $\overline{\Psi}_a \left( a, u; \overline{\rho} \right)$ and $\langle u_*, \overline{\Psi}_u \left( a, u; \overline{\rho} \right) \rangle$ depend only on $s := \langle u, u_* \rangle$ and $s_1 := \langle u_1, u_* \rangle$. Let $u^\perp$ (resp.~$u_1^\perp$) denote the component of $u$ (resp.~$u_1$) orthogonal to $u_*$. Let $R$ denote a random uniform rotation orthogonal to $u_*$. By the rotation invariance of $\overline{\rho}$, 
\begin{align*}
	\overline{\Psi}_a \left( a, u; \overline{\rho} \right) &= \varepsilon^{-1} \left(V(\langle u , u_* \rangle) - \int_{\R \times \S^{d-1}} a_1 \E_R \left[U(\langle u , Ru_1 \rangle)\right] \overline{\rho}(\diff a_1, \diff u_1) \right) \\ 
	&= \varepsilon^{-1} \left(V(s) - \int_{\R \times \S^{d-1}} a_1 \E_R \left[U(ss_1 + \langle u^\perp , Ru_1^\perp \rangle)\right] \overline{\rho}(\diff a_1, \diff u_1) \right) \, .
\end{align*}
The random variable $B^{(d)} = \left\langle \frac{u^\perp}{\Vert u^\perp \Vert} , R\frac{u_1^\perp}{\Vert u_1^\perp \Vert} \right\rangle$ is a one dimensional projection of a random variable uniform on the unit sphere of the hyperplane orthogonal to $u_*$; thus it has the density $p_{B^{(d)}}(b) \propto (1-b^2)^{d/2-2}$ (see, e.g., \cite[Lemma 4.17]{frye2012spherical}). Denote 
\begin{equation*}
	U^{(d)}(s,s_1) = \E_{B^{(d)}} \left[U\left(ss_1 + (1-s^2)^{1/2}(1-s_1^2)^{1/2} B^{(d)}\right)\right] \, ,
\end{equation*}
then we have
\begin{equation}
	\label{eq:psi-a}
		\overline{\Psi}_a \left( a, u; \overline{\rho} \right) = \varepsilon^{-1} \left(V(s) - \int_{\R \times \S^{d-1}} a_1 U^{(d)}(s,s_1) \overline{\rho}(\diff a_1, \diff u_1) \right) \, .
\end{equation}
Further, we compute
\begin{align*}
	&\left\langle u_*, \overline{\Psi}_u \left( a, u; \overline{\rho} \right) \right\rangle = 	\Bigg\langle u_*, a\left(I_d - u u^\top\right) \\
	&\hspace{3cm} \left(V'(s)u_* - \int_{\R \times \S^{d-1}} a_1 \E_R\left[ U'(ss_1 + \langle u^\perp , Ru_1^\perp\rangle )\left(s_1 u_* + Ru_1^\perp \right) \right] \overline{\rho}(\diff a_1, \diff u_1) \right) \Bigg\rangle \\
	&= a \left[(1-s^2)V'(s) -\int_{\R \times \S^{d-1}} a_1 \E_R\left[ U'(ss_1 + \langle u^\perp , Ru_1^\perp \rangle )\langle u_*, (I_d - u u^\top)(s_1 u_* + Ru_1^\perp) \rangle \right] \overline{\rho}(\diff a_1, \diff u_1) \right] \, .
\end{align*}
In the equation above, we have $\langle u_*, (I_d - u u^\top) s_1 u_* \rangle = s_1 (1-s^2)$ and as $\langle u_*, Ru_1^\perp \rangle = 0$ a.s., we have 
\begin{equation*}
	\left\langle u_*, (I_d - u u^\top) Ru_1^\perp\right\rangle = - \langle u,u_* \rangle \langle u, Ru_1^\perp \rangle = -s\langle u^\perp, Ru_1^\perp \rangle \, .
\end{equation*}
Thus we obtain 
\begin{align*}
	&\left\langle u_*, \overline{\Psi}_u \left( a, u; \overline{\rho} \right) \right\rangle \\
	&= a (1-s^2) \left[V'(s) -\int_{\R \times \S^{d-1}} a_1 \E_R\left[ U'(ss_1 + \langle u^\perp , Ru_1^\perp \rangle )\left(s_1 - \frac{s}{1-s^2} \langle u^\perp , Ru_1^\perp \rangle\right) \right] \overline{\rho}(\diff a_1, \diff u_1) \right] \, .
\end{align*}
Note that 
\begin{equation*}
	\partial_s U^{(d)}(s,s_1) = \E_{B^{(d)}}\left[U'\left(ss_1 + (1-s^2)^{1/2}(1-s_1^2)^{1/2} B^{(d)}\right)\left(s_1 - \frac{s}{(1-s^2)^{1/2}}(1-s_1^2)^{1/2}B^{(d)}\right)\right]
\end{equation*}
and thus we have
\begin{align}
	\label{eq:psi-u}
	\left\langle u_*, \overline{\Psi}_u \left( a, u; \overline{\rho} \right) \right\rangle = a (1-s^2) \left[V'(s) -\int_{\R \times \S^{d-1}} a_1 \left(\partial_s U^{(d)}\right)(s,s_1) \overline{\rho}(\diff a_1, \diff u_1) \right] \, .
\end{align}
Of course, a discrete measure of the form \eqref{eq:def-rho} can not be invariant to rotations orthogonal to $u_*$. However, if the $u_i$ are initialized uniformly on the unit sphere, then the measure $\overline{\rho}_0$ converges to a measure with the rotation invariance as $m \to \infty$. One can then apply the results of \cite{mei2019mean} to control the deviations from this limit. Let us thus assume that $\overline{\rho}_0$ satisfies the rotation invariance. Define the map $\varphi(a,u) = (a, \langle u, u_* \rangle)$. Then, from \eqref{eq:psi-a}, \eqref{eq:psi-u}, the push-forward $\rho_t$ of $\overline{\rho}_t$ through the map $\varphi$ satisfies the continuity equation 
\begin{align*}
	\partial_t \rho_t (a, s) &= - \nabla \cdot \left( \rho_t \Psi^{(d)} \left( a, s; \rho_t \right) \right) \\
	&= - \left( \partial_a \left( \rho_t \Psi^{(d)}_a \left( a, s; \rho_t \right) \right) + \partial_s \left( \rho_t \Psi^{(d)}_s \left( a, s; \rho_t \right) \right) \right),
\end{align*}
where  $\Psi^{(d)} = (\Psi^{(d)}_a, \Psi^{(d)}_s)$ is given by
\begin{align*}
	\Psi^{(d)}_a (a, s; \rho) = \, & \veps^{-1}\cdot \left( V(s) - \int_{\R^2} a_1 U^{(d)}(s, s_1) \rho (\mathrm{d} a_1, \mathrm{d} s_1) \right)\, , \\
	\Psi^{(d)}_s (a, s; \rho) = \, & a(1 - s^2) \cdot \left( V'(s) - \int_{\R^2} a_1 \partial_s U^{(d)}(s, s_1) \rho(\mathrm{d} a_1, \mathrm{d} s_1) \right) \, .
\end{align*}
When $d \to \infty$, $p_{B^{(d)}}(b)\diff b \propto (1-b^2)^{d/2-2} \diff b$ converges weakly to the Dirac mass $ \delta_0(\diff b)$. As a consequence, 
\begin{align*}
	&U^{(d)}(s,s_1) \xrightarrow[d \to \infty]{} U(ss_1) \, , &&\partial_sU^{(d)}(s,s_1) \xrightarrow[d \to \infty]{} U'(ss_1)s_1 \, .
\end{align*}
As a consequence, in the limit $d \to \infty$, we recover the equations \eqref{eq:MF_dynamics_rho_bar}--\eqref{eq:psi-s}. Moreover, if $\overline{\rho}_0 = \rP_A\otimes \Unif(\S^{d-1}) $, then $\rho_0$ converges weakly to $\rP_A \otimes \delta_0(\diff s)$ as $d \to \infty$. 

\modif{
\subsection{Proof of Proposition \ref{prop:risk_infimum}}
\label{sec:proof_risk_infimum}

	First, note that the potential functions $U$ and $V$ admit the following expansion:
	\begin{align*}
		U(s) = \, \sum_{k=0}^{\infty} \sigma_k^2 s^k, \ 
		V(s) = \, \sum_{k=0}^{\infty} \varphi_k \sigma_k s^k.
	\end{align*}
	As a consequence, we deduce that
	\begin{align*}
		\Risk_{\mf,*}(\rho) = \, & \frac{1}{2} \sum_{k=0}^{\infty} \varphi_k^2 - \sum_{k=0}^{\infty} \varphi_k \sigma_k \int a(\omega) s(\omega)^k  \diff \rho(\omega) + \frac{1}{2} \sum_{k=0}^{\infty} \sigma_k^2 \int a(\omega_1)a(\omega_2) s(\omega_1)^k s(\omega_2)^k  \diff  \rho(\omega_1)  \diff \rho(\omega_2) \\
		= \, & \frac{1}{2} \sum_{k=0}^{\infty} \left( \varphi_k^2 - 2 \varphi_k \sigma_k \int a(\omega) s(\omega)^k  \diff \rho(\omega) + \sigma_k^2 \left( \int a(\omega) s(\omega)^k  \diff  \rho(\omega) \right)^2 \right) \\
		= \, & \frac{1}{2} \sum_{k=0}^{\infty} \left( \varphi_k - \sigma_k \int a(\omega) s(\omega)^k  \diff \rho(\omega) \right)^2.
	\end{align*}
	Now we show that the above risk can be arbitrarily small. We will choose $\rho$ to be the Lebesgue measure on $[-1, 1]$ and $a \in L^2 [-1, 1]$ so that $\int a(\omega) s(\omega)^k  \diff \rho(\omega) = \int_{-1}^{1} a(s) s^k \diff s$. Now, we define the following set of sequences
	\begin{equation*}
		W = \left\{ \left( \sigma_k \int_{-1}^{1} a(s) s^k \diff s \right)_{k \ge 0} \, \middle\vert \, a \in L^2 [-1, 1] \right\} \, .
	\end{equation*}
	Since $a \in L^2 [-1, 1]$ and $(\sigma_k)_{k \ge 0} \in \ell^2$, we know that $W \subset \ell^2$, i.e., $W$ is a linear subspace of $\ell^2$. Now it suffices to show that $W$ is dense in $\ell^2$, which is equivalent to $W^{\perp} = \{ 0 \}$, namely
	\begin{equation*}
		v \in \ell^2, \ v \perp W \implies v = 0.
	\end{equation*}
	Fix any such $v$ and take $\mu \in W$ such that for all $k$, $\mu_k = \sigma_k \int_{-1}^{1} a(s) s^k \diff s$ for some $a \in L^2 [-1, 1]$. We then have
	\begin{align*}
		0 = \langle v, \mu \rangle = \sum_{k=0}^{\infty} v_k \mu_k = \sum_{k=0}^{\infty} v_k \sigma_k \int_{-1}^{1} a(s) s^k \diff s = \int_{-1}^{1} a(s) \cdot \left( \sum_{k=0}^{\infty} v_k \sigma_k s^k \right) \diff s,
	\end{align*}
	where the last step follows from dominated convergence theorem. Indeed, by H\"{o}lder's inequality,
	\begin{equation*}
		\sum_{k=0}^{\infty} |v_k \sigma_k| \le \left( \sum_{k=0}^{\infty} v_k^2 \right)^{1/2} \left( \sum_{k=0}^{\infty} \sigma_k^2 \right)^{1/2} < \infty.
	\end{equation*}
	As a consequence, the function series $\sum_{k=0}^{n} \sigma_k v_k s^k$ uniformly absolutely converges to the continuous function $f(s) = \sum_{k=0}^{\infty} \sigma_k v_k s^k$ on $[-1, 1]$.
	The above argument then implies that for any $a \in L^2 [-1, 1]$, $\int_{-1}^{1} a(s) f(s) \diff s = 0$, which further implies that $f(s) \equiv 0$. Therefore, $\sigma_k v_k = 0$ for all $k \ge 0$. Since $\sigma_k \neq 0$ for all $k \ge 0$, we must have $v_k = 0$ for all $k \ge 0$, i.e., $v = 0$. This completes the proof of the density of $W$ in $\ell^2$, and thus the proof of the Proposition.
}

\section{Calculations for the analysis of mean-field gradient flow}

\subsection{Solution of Eq. \eqref{eq:ode-third-time-scale}}
\label{sec:solution-ode-third-time-scale}

In order to solve the system \eqref{eq:ode-third-time-scale}, we start from  an associated one-dimensional ODE.
\begin{lem}
	The solution $\lambda = \lambda(t_3)$ of the ODE 
	\begin{equation}
		\partial_{t_3} \lambda = \vert\sigma_1 \vert \left(|\varphi_1| - {|\sigma_1|} \left\Vert a_{\perp,\rm{init}}\right\Vert_{L^2(\rho)}^2 \lambda^2 \right) \lambda \label{eq:bernoulli-differential-eq}
	\end{equation}
	with initial condition $\lambda(0)$ is  
	\begin{equation}
		\label{eq:bernoulli-solution}
		\lambda = \frac{|\varphi_1|^{\nicefrac{1}{2}}}{\left({|\sigma_1|}\left\Vert a_{\perp,\rm{init}}\right\Vert_{L^2(\rho)}^2 + \left(|\varphi_1| \lambda(0)^{-2} - {|\sigma_1|}\left\Vert a_{\perp,\rm{init}}\right\Vert_{L^2(\rho)}^2\right) e^{-2\vert \sigma_1 \varphi_1 \vert t_3} \right)^{\nicefrac{1}{2}}} \, .
	\end{equation}
\end{lem}

\begin{proof}
	For simplicity, denote $\alpha = \vert\sigma_1 \vert$, $\beta = |\varphi_1|$ and $\gamma = {|\sigma_1|} \left\Vert a_{\perp,\rm{init}}\right\Vert_{L^2(\rho)}^2$. Then 
	\begin{equation*}
		\partial_{t_3} \lambda = \alpha \left(\beta - \gamma \lambda^2\right) \lambda \, .
	\end{equation*}
	This is Bernoulli differential equation (see, e.g., \cite{bernoulli}). In this situation, the classical trick is to reduce the problem to a linear inhomogeneous first-order equation by considering
	\begin{align*}
		\partial_{t_3} \left(\lambda^{-2}\right) &= -2 \left(\partial_{t_3} \lambda \right) \lambda^{-3} = -2 \alpha \left(\beta - \gamma \lambda^2\right) \lambda^{-2} \\
		&= 2 \alpha(\gamma - \beta \lambda^{-2}) \, .
	\end{align*}
	Solving this linear inhomogeneous first-order equation gives 
\begin{equation*}
		\lambda^{-2} = \frac{\gamma}{\beta} + \left(\lambda(0)^{-2}-\frac{\gamma}{\beta}\right)\, e^{-2\alpha \beta t_3} \, , 
\end{equation*}
	and thus 
	\begin{equation*}
		\lambda = \frac{\beta^{\nicefrac{1}{2}}}{\left(\gamma + \left(\beta \lambda(0)^{-2} - \gamma \right) e^{-2\alpha\beta t_3}\right)^{\nicefrac{1}{2}}} \, ,
	\end{equation*}
	which is the claimed result. 
\end{proof}

Let $\lambda = \lambda(t_3)$ be a solution of \eqref{eq:bernoulli-differential-eq} and consider 
\begin{align}
	\label{eq:aux-23}
	\begin{split}
		&a^{(-1)} = a^{(-1)}_\perp = \lambda a_{\perp,\rm{init}} \, , \\
		&s^{(1)} = s^{(1)}_\perp = \sign(\sigma_1 \varphi_1) \lambda a_{\perp,\rm{init}} \, .
	\end{split}
\end{align}
Then $a^{(-1)}, s^{(1)}$ are solutions of the constrained ODE system \eqref{eq:aux-19}, \eqref{eq:aux-22}. Indeed, 
\begin{equation*}
	\left\langle a^{(-1)}, \bfone \right\rangle_{L^2(\rho)} = \lambda \left\langle a_{\perp,\rm{init}}, \bfone \right\rangle_{L^2(\rho)} = 0 \, , 
\end{equation*}
thus the constraint \eqref{eq:aux-19} is satisfied. Further
\begin{align*}
	\partial_{t_3} a_{\perp}^{(-1)} &= \left(\partial_{t_3} \lambda \right) a_{\perp,\rm{init}} \\
	&= |\sigma_1| \left(|\varphi_1| - {|\sigma_1|} \Vert a_{\perp,\rm{init}} \Vert_{L^2(\rho)}^2 \lambda^2  \right) \lambda a_{\perp,\rm{init}} \\
	&= \sign(\sigma_1 \varphi_1)\sigma_1 \left(\varphi_1 - \sign(\sigma_1 \varphi_1) {\sigma_1} \Vert a_{\perp,\rm{init}} \Vert_{L^2(\rho)}^2 \lambda^2  \right) \lambda a_{\perp,\rm{init}} \\ 
	&= \sigma_1 \left(\varphi_1 - {\sigma_1} \left\langle a^{(-1)}_\perp, s^{(1)}_\perp \right\rangle_{L^2(\rho)}\right) s^{(1)}_\perp \, .
\end{align*}
A similar computation shows that the differential equation for $s^{(1)}$ is also satisfied. This concludes that \eqref{eq:aux-23} is a valid candidate to solve the third time scale. 

\paragraph{Matching.} To determine the value of the initialization $\lambda(0)$ we perform a 
matching procedure with the previous time scale. In this paragraph, we denote 
$\underline{a}, \underline{s}$ the approximation obtained in the second time scale
 (Section \ref{sec:second-layer}), and $\overline{a}, \overline{s}$ the approximation in the 
 third time scale (Section~\ref{sec:third-layer} and above).

Consider an intermediate time scale $\widetilde{t} = t_2 - c\log \frac{1}{\varepsilon}$ with $0 < c < \frac{1}{4|\sigma_1 \varphi_1|}$. Assume $\widetilde{t} \asymp 1$. Then
\begin{align*}
	t_2 &= \widetilde{t} + c \log \frac{1}{\varepsilon} \xrightarrow[\varepsilon\to 0]{} +\infty \, , \\
	t_3 &=\widetilde{t} - \left(\frac{1}{4|\sigma_1 \varphi_1|}-c\right) \log \frac{1}{\varepsilon} \xrightarrow[\varepsilon\to 0]{} -\infty \, .
\end{align*}
From the approximation \eqref{eq:layer2-approx} on the second time scale,
\begin{align}
	\underline{a} &= \underline{a}^{(0)} + O(\varepsilon^{\nicefrac{1}{2}}) \nonumber \\
	&= \frac{\varphi_0}{\sigma_0} \bfone + \cosh\left(\varphi_1\sigma_1 t_2\right) {a}_{\perp,\rm{init}} + O(\varepsilon^{\nicefrac{1}{2}}) \nonumber \\
	&= \cosh\left(\varphi_1\sigma_1 
	\left(\widetilde{t}+c \log \frac{1}{\varepsilon}\right)\right) {a}_{\perp,\rm{init}} + O(1) \nonumber \\
	&= \frac{1}{2} e^{\vert \varphi_1 \sigma_1 \vert \widetilde{t}} \varepsilon^{-c \vert \varphi_1 \sigma_1 \vert } {a}_{\perp,\rm{init}} + O(1) \, . \label{eq:aux-24}
\end{align}
From the approximation on the third time scale, 
\begin{align*}
	\overline{a} &= \varepsilon^{-\nicefrac{1}{4}} \overline{a}^{(-1)} + O(1) \\
	&= \varepsilon^{-\nicefrac{1}{4}} \lambda {a}_{\perp,\rm{init}} +  O(1) \, .
\end{align*}
Note that as $t_3 \to -\infty$, from \eqref{eq:bernoulli-solution}, 
\begin{equation*}
	\lambda \sim \frac{|\varphi_1|^{\nicefrac{1}{2}}}{\left(|\varphi_1| \lambda(0)^{-2} - {|\sigma_1|} \left\Vert a_{\perp,\rm{init}}\right\Vert_{L^2(\rho)}^2 \right)^{\nicefrac{1}{2}}} e^{\vert \sigma_1 \varphi_1 \vert t_3} \, .
\end{equation*}
Thus 
\begin{align}
	\overline{a} &\sim \varepsilon^{-\nicefrac{1}{4}}\frac{|\varphi_1|^{\nicefrac{1}{2}}}{\left(|\varphi_1| \lambda(0)^{-2} - {|\sigma_1|} \left\Vert a_{\perp,\rm{init}}\right\Vert_{L^2(\rho)}^2 \right)^{\nicefrac{1}{2}}} e^{\vert \sigma_1 \varphi_1 \vert \left(\widetilde{t} - \left(\frac{1}{4|\sigma_1 \varphi_1|}-c\right) \log \frac{1}{\varepsilon}\right)}  a_{\perp,\rm{init}} \nonumber\\
	&\sim \frac{|\varphi_1|^{\nicefrac{1}{2}}}{\left(|\varphi_1| \lambda(0)^{-2} - {|\sigma_1|} \left\Vert a_{\perp,\rm{init}}\right\Vert_{L^2(\rho)}^2 \right)^{\nicefrac{1}{2}}} e^{\vert \sigma_1 \varphi_1 \vert \widetilde{t}} \varepsilon^{-c|\sigma_1 \varphi_1|} a_{\perp,\rm{init}} \, . \label{eq:aux-25}
\end{align}
By matching, Equations \eqref{eq:aux-24} and \eqref{eq:aux-25} should be coherent. This gives
\begin{equation*}
	\frac{1}{2} = \frac{|\varphi_1|^{\nicefrac{1}{2}}}{\left(|\varphi_1| \lambda(0)^{-2} - {|\sigma_1|} \left\Vert a_{\perp,\rm{init}}\right\Vert_{L^2(\rho)}^2 \right)^{\nicefrac{1}{2}}} \, ,
\end{equation*}
and thus 
\begin{equation}
	\label{eq:aux-27}
	\lambda(t_3) = \frac{|\varphi_1|^{\nicefrac{1}{2}}}{\left({|\sigma_1|} \left\Vert a_{\perp,\rm{init}}\right\Vert_{L^2(\rho)}^2 + 4 |\varphi_1| e^{-2\vert \sigma_1 \varphi_1 \vert t_3} \right)^{\nicefrac{1}{2}}} \, .
\end{equation}
One could check similarly that $s^{(1)}$ also satisfies the matching conditions under the same constraint, and thus that \eqref{eq:aux-23} are indeed the solutions of the third time scale. 

\subsection{Induced approximation of the risk}
\label{sec:induced}

In this section, we show that the behavior of $a$ and $s$ derived in Sections 
\ref{sec:first-layer}--\ref{sec:third-layer} leads to an evolution of
the risk alternating plateaus and rapid decreases, 
in agreement with the \modif{canonical learning order} 
of Definition \ref{def:StandardScenario}. For the convenience of the reader, we recall the expression \eqref{eq:risk-general} of the risk
\begin{align*}
	\Risk_{\mf,*} &= \frac{1}{2} \Vert \varphi \Vert^2 - \int a(\omega) V(s(\omega))  \diff \rho(\omega) + \frac{1}{2} \int a(\omega_1)a(\omega_2) U(s(\omega_1)s(\omega_2))  \diff  \rho(\omega_1)  \diff \rho(\omega_2) \\
	&= \frac{1}{2} \sum_{k \geq 0} \left(\varphi_k - \sigma_k \int a(\omega) s(\omega)^k \diff\rho(\omega)\right)^2 \, .
\end{align*}

\paragraph{First time scale $t_1 = \frac{t}{\varepsilon}$ (Section \ref{sec:first-layer}).} 
On this time scale, we have $a = O(1)$ and $s = O(\varepsilon)$. Thus for all
$k \geq 1$, $\int a(\omega) s(\omega)^k \diff\rho(\omega) = O(\varepsilon)$ whence
$\left(\varphi_k - \sigma_k \int a(\omega) s(\omega)^k \diff\rho(\omega)\right)^2 = \varphi_k^2 + O(\varepsilon)$.

Further, using \eqref{eq:aux-26},
\begin{align*}
	\left(\varphi_0 - \sigma_0 \int a(\omega)  \diff\rho(\omega)\right)^2 &= \left(\varphi_0 - {\sigma_0}\left\langle a, \bfone \right\rangle_{L^2(\rho)} \right)^2 = \left(\varphi_0 - {\sigma_0}\left\langle a^{(0)}, \bfone \right\rangle_{L^2(\rho)} +O(\varepsilon) \right)^2 \\
	&= e^{-2\sigma_0^2t_1}\left(\varphi_0 - {\sigma_0}\left\langle a_{\rm{init}}, \bfone \right\rangle_{L^2(\rho)} \right)^2 + O(\varepsilon) \, .
\end{align*}
Thus as $\varepsilon \to 0$,
\begin{align*}
	\Risk_{\mf,*} = \frac{1}{2} e^{-2\sigma_0^2t_1}\left(\varphi_0 - {\sigma_0}\left\langle a_{\rm{init}}, \bfone \right\rangle_{L^2(\rho)} \right)^2  + \frac{1}{2 }\sum_{k\geq 1} \varphi_k^2 + O(\varepsilon) \, .
\end{align*}
This describes, in a more detailed form, the first transition in Definition \ref{def:StandardScenario}.

\paragraph{Second time scale $t_2 = \frac{t}{\varepsilon^{\nicefrac{1}{2}}}$ (Section \ref{sec:second-layer}).} On this time scale, we have $a = O(1)$ and $s = O(\varepsilon^{\nicefrac{1}{2}})$. Thus for all $k \geq 1$, $\left( \varphi_k - \sigma_k \int a(\omega) s(\omega)^k \diff\rho(\omega) \right)^2 = \varphi_k^2 + O(\varepsilon^{\nicefrac{1}{2}})$.

Further, using \eqref{eq:aux-12}, 
\begin{align*}
	\left(\varphi_0 - \sigma_0 \int a(\omega)  \diff\rho(\omega) \right)^2 = \left(\varphi_0 - \sigma_0 \int a^{(0)}(\omega)  \diff\rho(\omega) + O(\varepsilon^{\nicefrac{1}{2}}) \right)^2 = O(\varepsilon^{\nicefrac{1}{2}}) \, .
\end{align*}
Thus as $\varepsilon \to 0$,
\begin{align*}
	\Risk_{\mf,*} =  \frac{1}{2 }\sum_{k\geq 1} \varphi_k^2 + O(\varepsilon^{\nicefrac{1}{2}}) \, .
\end{align*}
This second time scale does not induce any transition of the risk $\Risk_{\mf,*}$
(but was necessary to understand the divergence of $a$ and $\varepsilon^{-\nicefrac{1}{2}} s$).

\paragraph{Third time scale $t_3 =  \frac{t}{\varepsilon^{\nicefrac{1}{2}}} - \frac{1}{4|\sigma_1 \varphi_1|} \log \frac{1}{\varepsilon}$ (Section \ref{sec:third-layer}).} On this time scale, we have $a = O(\varepsilon^{-\nicefrac{1}{4}})$ and $s = O(\varepsilon^{\nicefrac{1}{4}})$. Thus for all $k \geq 2$, $\left( \varphi_k - \sigma_k \int a(\omega) s(\omega)^k \diff\rho(\omega) \right)^2 = \varphi_k^2 + O(\varepsilon^{\nicefrac{1}{4}})$. 

Further, using \eqref{eq:aux-19}, \eqref{eq:aux-20},
\begin{align*}
	\left(\varphi_0 - \sigma_0 \int a(\omega)  \diff\rho(\omega) \right)^2 &= \left(\varphi_0 - {\sigma_0}\varepsilon^{-\nicefrac{1}{4}}\int a^{(-1)}(\omega)\diff\rho(\omega)- {\sigma_0}\int a^{(0)}(\omega)\diff\rho(\omega) + O(\varepsilon^{\nicefrac{1}{4}}) \right)^2 \\
	&= O(\varepsilon^{\nicefrac{1}{4}}) \, .
\end{align*}
Finally, using \eqref{eq:aux-23-main}, \eqref{eq:aux-27-main}, 
\begin{align*}
	\left(\varphi_1 - {\sigma_1} \int  a(\omega) s(\omega) \diff\rho(\omega) \right)^2 &= \left(\varphi_1 - {\sigma_1} \left\langle a, s \right\rangle_{L^2(\rho)} \right)^2 = \left(\varphi_1 - {\sigma_1}\left\langle a^{(-1)}, s^{(1)} \right\rangle_{L^2(\rho)} + O(\varepsilon^{\nicefrac{1}{4}}) \right)^2 \\ 
	&\stackrel{(a)}{=} \left(\varphi_1 - {\sigma_1}\sign(\sigma_1 \varphi_1) \lambda^2 \Vert a_{\perp,\rm{init}} \Vert^2_{L^2(\rho)} \right)^2 + O(\varepsilon^{\nicefrac{1}{4}}) \\
	&\stackrel{(b)}{=} \left(\varphi_1 - \frac{\sigma_1 \sign(\sigma_1 \varphi_1) |\varphi_1| \Vert a_{\perp,\rm{init}} \Vert_{L^2(\rho)}^2}{\left({|\sigma_1|} \left\Vert a_{\perp,\rm{init}}\right\Vert_{L^2(\rho)}^2 + 4 |\varphi_1| e^{-2\vert \sigma_1 \varphi_1 \vert t_3} \right)}   \right)^2 + O(\varepsilon^{\nicefrac{1}{4}}) \\
	&= \varphi_1^2\left(1 - \frac{1}{1 + \frac{4 |\varphi_1|}{|\sigma_1| \left\Vert a_{\perp,\rm{init}}\right\Vert_{L^2(\rho)}^2} e^{-2\vert \sigma_1 \varphi_1 \vert t_3} }   \right)^2 + O(\varepsilon^{\nicefrac{1}{4}})\, ,
\end{align*}
where in $(a)$ we used \eqref{eq:aux-23-main} and in $(b)$ \eqref{eq:aux-27-main}. 
Thus as $\varepsilon \to 0$,
\begin{align*}
	\Risk_{\mf,*} = \frac{1}{2}  \varphi_1^2\left(1 - \frac{1}{1 + \frac{4 |\varphi_1|}{|\sigma_1| \left\Vert a_{\perp,\rm{init}}\right\Vert_{L^2(\rho)}^2} e^{-2\vert \sigma_1 \varphi_1 \vert t_3} }   \right)^2 + \frac{1}{2 }\sum_{k\geq 2} \varphi_k^2 + O(\varepsilon^{\nicefrac{1}{4}}) \, .
\end{align*}
This describes, in a more detailed form, the second transition in Definition \ref{def:StandardScenario}.

\subsection{Proof of Theorem~\ref{thm:evolution_single}}\label{sec:proof_evo_sin}
Throughout the proof, we will use the shorthand $\Risk_l (\tau)$ to represent $\Risk_l (\ta(\tau), \ts(\tau))$.
First, note that according to the ODE satisfied by $\Risk_l$ (Eq.~\eqref{eq:lth_risk_ODE}), we know that $\Risk_l$ must be non-increasing, thus for small enough $\veps > 0$,
\begin{align*}
	\left\vert \varphi_l - \sigma_l \int \ta(\nu, \tau) \ts(\nu, \tau)^l \diff \rho(\nu) \right\vert \le\, & \left\vert \varphi_l - \sigma_l \int \ta(\nu, 0) \ts(\nu, 0)^l \diff \rho(\nu) \right\vert \\
	\le\, & \left\vert \varphi_l \right\vert + O (\veps^{1/ 2l}) \le 2 \left\vert \varphi_l \right\vert, \ \forall \tau \ge 0.
\end{align*}
Hence, we obtain the estimates:
\begin{align*}
	\partial_{\tau} \left\vert \ta (\omega) \right\vert \le \left\vert \partial_{\tau} \tilde{a} (\omega) \right\vert \le 2 \vert \sigma_l \vert \vert \varphi_l \vert \left\vert \ts (\omega) \right\vert^l, \ \partial_\tau \left\vert \ts (\omega) \right\vert \le \left\vert \partial_{\tau} \ts (\omega) \right\vert \le 2 l \vert \sigma_l \vert \vert \varphi_l \vert \left\vert \ta (\omega) \right\vert \left\vert \ts (\omega) \right\vert^{l-1}.
\end{align*}
According to the comparison theorem for system of ODEs, we know that $\vert \ta(\omega, \tau) \vert \le \hat{a} (\omega, \tau)$, $\vert \ts(\omega, \tau) \vert \le \hat{s} (\omega, \tau)$ for all $\tau \ge 0$ where
\begin{equation*}
	\hat{a} (\omega, 0) = \max \left\{ \vert \ta(\omega, 0) \vert, \ \vert \ts(\omega, 0) \vert \right\}, \ \hat{s} (\omega, 0) = l^{1/2} \hat{a} (\omega, 0) = l^{1/2} \max \left\{ \vert \ta(\omega, 0) \vert, \ \vert \ts(\omega, 0) \vert \right\},
\end{equation*}
and
\begin{equation}
	\partial_\tau \hat{a} (\omega) = 2 \vert \sigma_l \vert \vert \varphi_l \vert \hat{s} (\omega)^l, \ \partial_\tau \hat{s} (\omega) = 2 l \vert \sigma_l \vert \vert \varphi_l \vert \hat{a} (\omega) \hat{s} (\omega)^{l-1}.
\end{equation}
The above system of ODEs can be solved analytically via integration. First, we note that
\begin{equation*}
	\partial_\tau (\hat{s} (\omega)^2) = 4 l \vert \sigma_l \vert \vert \varphi_l \vert \hat{a} (\omega) \hat{s} (\omega)^l = l \partial_\tau ( \hat{a} (\omega)^2 ),
\end{equation*}
which implies that (further note $\hat{s} (\omega, 0)^2 = l \hat{a} (\omega, 0)^2$)
\begin{equation}
	\hat{s} (\omega, \tau) = l^{1/2} \hat{a} (\omega, \tau), \ \forall \tau \ge 0.
\end{equation}
The ODE system then reduces to $\partial_{\tau} \hat{a} (\omega) = 2 l^{l/2} \vert \sigma_l \vert \vert \varphi_l \vert \hat{a} (\omega)^l$, which admits the solution
\begin{equation}
	\hat{a} (\omega, \tau) = \left( \hat{a} (\omega, 0)^{-l+1} - 2 l^{l/2} (l-1) \vert \sigma_l \vert \vert \varphi_l \vert \tau \right)^{-1/(l-1)}.
\end{equation}
Since $\hat{a} (\omega, 0) = \Theta (\veps^{1 / 2 l (l+1)})$, we know that $\hat{a} (\omega, \tau), \hat{s} (\omega, \tau) = o(1)$ until $\tau = \Theta (\veps^{-(l-1)/2l(l+1)}) - O(1) = \Theta (\veps^{-(l-1)/2l(l+1)})$, which means that $\ta (\omega, \tau), \ts (\omega, \tau) = o(1)$ until $\tau = \Omega (\veps^{-(l-1)/2l(l+1)})$. As a consequence,
\begin{equation*}
	\Risk_l (\tau) = \frac{1}{2} \left( \varphi_l - \sigma_l \int \ta(\nu, \tau) \ts(\nu, \tau)^l \diff \rho (\nu) \right)^2 = \frac{1}{2} \varphi_l^2 - o(1)
\end{equation*}
until $\tau = \Omega (\veps^{-(l-1)/2l(l+1)})$. This means that the learning of the $l$-th component will not begin until $\tau = \Omega (\veps^{-(l-1)/2l(l+1)})$, namely $\tau(\Delta) = \Omega (\veps^{-(l-1)/2l(l+1)})$ for any fixed $\Delta > 0$. Note that the above argument applies to all of the settings in the theorem statement.

Next, we show that for any fixed $\Delta > 0$, $\tau(\Delta) = O (\veps^{-(l-1)/2l(l+1)})$, which means that the $l$-th component can be learnt in $O (\veps^{-(l-1)/2l(l+1)})$ time. To prove our claim by contradiction, assume that there exists $\Delta > 0$ and a sequence $\veps_k \downarrow 0$, such that
\begin{equation}\label{eq:lth_contra_ass}
	\lim_{k \to \infty} \frac{\tau (\Delta)}{\veps_k^{-(l-1)/2l(l+1)}} = + \infty.
\end{equation}
By definition of $\tau(\Delta)$, we know that $\forall \tau \le \tau (\Delta)$,
\begin{equation*}
	\left\vert \varphi_l - \sigma_l \int \ta(\nu, \tau) \ts(\nu, \tau)^l \diff \rho (\nu) \right\vert \ge\, \sqrt{2 \Delta}.
\end{equation*}
Now, assume the condition of setting (a) holds and denote
\begin{equation}
	A_{\veps_0, \eta} = \left\{ \omega: \forall \veps < \veps_0, \ \min ( \vert \ta(\omega, 0) \vert, \vert \ts(\omega,0) \vert ) > \eta \veps^{1/2l(l+1)}, \ \sigma_l \varphi_l \ta(\omega, 0) \ts(\omega,0)^l > 0 \right\}.
\end{equation}
Then by definition and our assumption that $\ta (\omega, 0)$ is of the same order as $\ts(\omega, 0)$, we know that $A = \cup_{\veps_0 > 0, \eta > 0} A_{\veps_0, \eta}$. Since $\rho (A) > 0$, there exists $\veps_0, \eta > 0$ such that $\rho (A_{\veps_0, \eta}) > 0$. Note that here we can choose $\veps_0$ and $\eta$ to be arbitrarily small since the set $A_{\veps_0, \eta}$ is non-increasing in $\veps_0$ and $\eta$. For $\omega \in A_{\veps_0, \eta}$ and $\tau \le \tau(\Delta)$, we have
\begin{equation*}
	\begin{split}
		\partial_\tau ( \ta(\omega)^2 ) = \, & 2 \sigma_l \ta(\omega) \ts(\omega)^l \left( \varphi_l - \sigma_l \int \ta(\nu) \ts(\nu)^l \diff \rho(\nu) \right) \ge 2 \sqrt{2 \Delta} \left\vert \sigma_l \ta(\omega) \ts(\omega)^l \right\vert \\
		\partial_\tau ( \ts(\omega)^2 ) = \, & 2 l \sigma_l \ta(\omega) \ts(\omega)^l \left( 1 - \veps^{2 \beta_l} \ts(\omega)^2 \right) \left( \varphi_l - \sigma_l \int \ta(\nu) \ts(\nu)^l \diff \rho(\nu) \right) \\
		\ge\, & 2 l \sqrt{2 \Delta} \left\vert \sigma_l \ta(\omega) \ts(\omega)^l \right\vert \left( 1 - \veps^{2 \beta_l} \ts(\omega)^2 \right).
	\end{split}
\end{equation*}
Moreover, we know that at initialization, $\vert \ta(\omega, 0) \vert, \vert \ts(\omega,0) \vert > \eta \veps^{1/2l(l+1)}$. Using the ODE comparison theorem and a similar argument as that in proving $\tau(\Delta) = \Omega (\veps^{-(l-1)/2l(l+1)})$, we deduce that for sufficiently large $k$ such that $\veps = \veps_k < \veps_0$, there exist constants $C, C' > 0$ that does not depend on $\veps$ satisfying the following: For all $\omega \in A_{\veps_0, \eta}$ and $\tau \ge C \veps^{-(l-1)/2l(l+1)}$,
\begin{equation*}
	\min \left\{ \vert \ta(\omega, \tau) \vert, \ \vert \ts(\omega, \tau) \vert \right\} \ge C', \ \sigma_l \varphi_l \ta (\omega, \tau) \ts (\omega, \tau)^l > 0.
\end{equation*}
This further implies that at time $\tau$,
\begin{equation}
	\int \ts(\omega)^{2 (l - 1)} \left( l^2 \ta(\omega)^2 \left( 1 - \veps^{2 \beta_l} \ts(\omega)^2 \right) + \ts(\omega)^2 \right) \diff \rho(\omega) \ge C'^{2l} \rho (A_{\veps_0, \eta}) > 0.
\end{equation}
According to Eq.~\eqref{eq:lth_risk_ODE}, we know that $\Risk_l$ will decrease to $0$ exponentially fast in an $O(1)$ time window after $\tau = C \veps^{-(l-1)/2l(l+1)}$, which contradicts our assumption~\eqref{eq:lth_contra_ass}. This proves that $\tau (\Delta) = O (\veps^{-(l-1) / 2l (l+1)})$ under setting (a). Next, we show that setting (b) can be reduced to setting (a). Under setting (b), let us denote
\begin{equation*}
	B_{\veps_0, \eta} = \left\{ \omega: \forall \veps < \veps_0, \ \sigma_l \varphi_l \ta(\omega, 0) \ts(\omega,0)^l < 0, \ \text{and} \ \ts (\omega, 0)^2 > (l + \eta) \ta (\omega, 0)^2 \right\}.
\end{equation*}
Then similar to the previous argument, there exists $\veps_0, \eta > 0$ such that $\rho (B_{\veps_0, \eta}) > 0$, and further we can choose $\veps_0$ and $\eta$ to be arbitrarily small. For $\omega \in B_{\veps_0, \eta}$, we have
\begin{equation*}
	\partial_\tau ( \ta(\omega)^2 ) < 0, \ \partial_\tau ( \ts(\omega)^2 ) < 0 \ \mbox{at} \ \tau = 0.
\end{equation*}
Hence, both $\ta(\omega)^2$ and $\ts(\omega)^2$ will decrease at initialization. Moreover, Eq.~\eqref{eq:l-th_simp} implies that
\begin{equation*}
	\partial_\tau ( \ta(\omega)^2 ) = \frac{\partial_\tau ( \ts(\omega)^2 )}{l \left( 1 - \veps^{2 \beta_l} \ts(\omega)^2 \right)} = - \frac{\veps^{- 2 \beta_l}}{l} \partial_\tau \log \left( 1 - \veps^{2 \beta_l} \ts(\omega)^2 \right).
\end{equation*}
Integrating both sides of the above equation, we obtain that
\begin{equation}\label{eq:relation_a_and_s}
	\ta (\omega, 0)^2 - \ta (\omega, \tau)^2 = - \frac{\veps^{ - 2 \beta_l}}{l} \left( \log \left( 1 - \veps^{2 \beta_l} \ts (\omega, 0)^2 \right) - \log \left( 1 - \veps^{2 \beta_l} \ts (\omega, \tau)^2 \right) \right),
\end{equation}
which is close to $ \left( \ts (\omega, 0)^2 - \ts (\omega, \tau)^2 \right) / l$ as long as $\ts (\omega, \tau) = O(1)$. To be accurate, let us define
\begin{equation*}
	\tau_{a, \omega} = \inf \{ \tau \ge 0: \ta(\omega, \tau) = 0 \},
\end{equation*}
then we know that $\ts (\omega, \tau_{a, \omega}) = \Omega (\veps^{1/2l(l+1)})$ and $\tau_{a, \omega} = O(\veps^{-(l-1) / 2l(l+1)})$ under the assumption~\eqref{eq:lth_contra_ass}, where the latter claim can be proved through making the change of variable $\ta' (\omega) = \veps^{-1 / 2l (l+1)} \ta (\omega)$ and $\ts' (\omega) = \veps^{-1 / 2l (l+1)} \ts (\omega)$. Note that after the time point $\tau_{a, \omega}$, the sign of $\ta (\omega)$ changes. Hence, $\varphi_l \sigma_l \ta(\omega) \ts(\omega)^l > 0$, and $\ta(\omega, \tau)^2$ and $\ts(\omega, \tau)^2$ will begin to increase for $\tau \ge \tau_{a, \omega}$. Similarly, we can show that in $O (\veps^{- (l - 1)/2 l (l + 1)})$ time after $\tau_{a, \omega}$, both $\ta (\omega)$ and $\ts (\omega)$ become of order $\veps^{1 / 2 l (l+1)}$, and we still have $\varphi_l \sigma_l \ta (\omega) \ts (\omega)^l > 0$. This reduces our case $(b)$ to case $(a)$.

We have proven that under settings (a) and (b), $\tau (\Delta) = \Theta (\veps^{-(l-1) / 2l (l+1)})$ for any fixed $\Delta \in (0, \varphi_l^2 / 2)$. This means that some of the neurons $(\ta(\omega), \ts (\omega))$ become of order $\Omega (1)$ and the $l$-th component of the target function is learnt at a timescale of order $\veps^{-(l-1) / 2l (l+1)}$. Next, we show that if the probability measure $\rho$ is discrete, then the evolution of $\Risk_l$ actually happens in an $O(1)$ time window. It suffices to prove that, for any $\Delta > 0$ a small constant ($\Delta < \varphi_l^2 / 4$),
\begin{equation}\label{eq:Oh_1_window}
	\tau (\Delta) - \tau \left( \frac{\varphi_l^2}{2} - \Delta \right) = O(1)
\end{equation}
as $\veps \to 0$. Note that by continuity and monotonicity of $\Risk_l$, we have
\begin{equation*}
	\Risk_l (\tau(\Delta)) = \Delta, \ \Risk_l \left( \tau \left( \frac{\varphi_l^2}{2} - \Delta \right) \right) = \frac{\varphi_l^2}{2} - \Delta, \ \mbox{and \ } \tau (\Delta) \ge \tau \left( \frac{\varphi_l^2}{2} - \Delta \right) .
\end{equation*}
By definition of $\Risk_l$, we know that $\forall \tau \ge \tau(\varphi_l^2 / 2 - \Delta)$,
\begin{equation*}
	\left\vert \int \ta(\nu, \tau) \ts(\nu, \tau)^l \diff \rho (\nu) \right\vert \ge \frac{1}{\vert \sigma_l \vert} \left( \vert \varphi_l \vert - \sqrt{\varphi_l^2 - 2 \Delta} \right) := r_l (\Delta) > 0.
\end{equation*}
Denote by $\{ (\ta_i, \ts_i) \}_{i \in [m] }$ the realizations of $\{ (\ta(\omega), \ts(\omega)) \}_{\omega \in \Omega}$ under the discrete measure $\rho$, and by $\{ p_i \}_{i \in [m]}$ the point masses of $\rho$. Then, we know that
\begin{align}
	& r_l (\Delta) \le \left\vert \int \ta(\nu, \tau) \ts(\nu, \tau)^l \diff \rho (\nu) \right\vert = \left\vert \sum_{j=1}^{m} p_j \ta_j (\tau) \ts_j (\tau)^l \right\vert \le \sum_{j=1}^{m} p_j \left\vert \ta_j (\tau) \ts_j (\tau)^l \right\vert,
\end{align}
which implies that $\exists j \in [m]$, s.t. $\left\vert \ta_j (\tau) \ts_j (\tau)^l \right\vert \ge r_l (\Delta)$. Applying Lemma~\ref{lem:simple_algebra} yields
\begin{align}
	& \int \ts(\omega)^{2 (l - 1)} \left( l^2 \ta(\omega)^2 \left( 1 - \veps^{2 \beta_l} \ts(\omega)^2 \right) + \ts(\omega)^2 \right) \diff \rho(\omega) \\
	\ge\, & p_j \ts_j^{2 (l - 1)} \left( l^2 \ta_j^2 \left( 1 - \veps^{2 \beta_l} \ts_j^2 \right) + \ts_j^2 \right) \ge \min_{j \in [m]} p_j \cdot c (l, r_l (\Delta)) > 0.
\end{align}
It then follows from Eq.~\eqref{eq:lth_risk_ODE} that $\Risk_l$ will decrease to $0$ exponentially fast, and Eq.~\eqref{eq:Oh_1_window} holds consequently. This completes the proof for settings (a) and (b).

We then focus on the case (c). By our assumption, for almost every $\omega$ there exists $\eta > 0$ (may depend on $\omega$) such that
\begin{equation*}
	\sigma_l \varphi_l \ta (\omega, 0) \ts (\omega, 0)^l < - \eta \veps^{1/2l}, \ \ts(\omega, 0)^2 < (l - \eta) \ta (\omega, 0)^2
\end{equation*}
for sufficiently small $\veps$.
Therefore, $\ts (\omega, \tau)^2$ and $\ta(\omega, \tau)^2$ will keep decreasing until one of them reaches $0$, which means that
\begin{equation}
	\int \ta(\nu) \ts (\nu)^l \diff \rho (\nu) = o(1) \implies \left\vert \varphi_l - \sigma_l \int \ta(\nu) \ts (\nu)^l \diff \rho (\nu) \right\vert \ge \frac{\vert \varphi_l \vert}{2}.
\end{equation}
According to Eq.~\eqref{eq:relation_a_and_s} and the inequality $\ts(\omega, 0)^2 < (l - \eta) \ta (\omega, 0)^2$, $\ta (\omega, \tau)^2$ will not reach $0$ until $\ts (\omega, \tau)^2$ reaches $0$. Furthermore, for any $\tau \ge 0$,
\begin{align*}
	 \ta (\omega, \tau)^2 =\, & \ta (\omega, 0)^2 + \frac{\veps^{ - 2 \beta_l}}{l} \left( \log \left( 1 - \veps^{2 \beta_l} \ts (\omega, 0)^2 \right) - \log \left( 1 - \veps^{2 \beta_l} \ts (\omega, \tau)^2 \right) \right) \\
	 \ge\, & \ta (\omega, 0)^2 + \frac{\veps^{ - 2 \beta_l}}{l} \log \left( 1 - \veps^{2 \beta_l} \ts (\omega, 0)^2 \right) \\ 
	 \ge\, & \frac{1}{l - \eta} \ts (\omega, 0)^2 - \frac{1}{l - \eta / 2} \ts (\omega, 0)^2 \ge\, c(\eta, l, \omega) \veps^{1 / l (l+1)},
\end{align*}
thus leading to
\begin{align}
	\partial_\tau ( \ts(\omega)^2 ) = \, & 2 l \sigma_l \ta(\omega) \ts(\omega)^l \left( 1 - \veps^{2 \beta_l} \ts(\omega)^2 \right) \left( \varphi_l - \sigma_l \int \ta(\nu) \ts(\nu)^l \diff \rho(\nu) \right) \\
	\le\, & - c (\eta, l, \omega) \veps^{1/2l (l+1)} (\ts(\omega)^2)^{l / 2}.
\end{align}
Using again the comparison theorem for ODE, we get that
\begin{equation}
	\vert \ts(\omega, \tau) \vert \le \left( \vert \ts(\omega, 0) \vert^{-l+2} + c(\eta, l, \omega) \veps^{1 / 2l(l+1)} \tau \right)^{-1/(l-2)}.
\end{equation}
Since $\ts (\omega, 0) \asymp \veps^{1 / 2l (l+1)}$, it follows immediately that for any $\Delta > 0$, there exists a constant $C_* (\omega, \Delta) > 0$ such that
\begin{equation}
	\tau \ge C_* (\omega, \Delta) \veps^{- (l-1) / 2l (l+1)} \implies \vert \ts(\omega, \tau) \vert \le \Delta \veps^{1 / 2l (l+1)}.
\end{equation}
This completes the discussion for case (c), thus concluding the proof of Theorem~\ref{thm:evolution_single}.

\begin{lem}\label{lem:simple_algebra}
	Let $r > 0$ be a constant that does not depend on $\veps$. Then there exists a constant $c = c(l, r) > 0$ that only depends on $l$ and $r$ such that the following holds: For any $a > 0$, $s > 0$ satisfying $a s^l \ge r$ and $\veps^{2 \beta_l} s^2 \le 1$, we have
	\begin{equation}
		s^{2 (l - 1)} \left( l^2 a^2 \left( 1 - \veps^{2 \beta_l} s^2 \right) + s^2 \right) \ge c.
	\end{equation}
\end{lem}

\begin{proof}
	If $s \ge 1$, then we immediately get
	\begin{equation*}
		s^{2 (l - 1)} \left( l^2 a^2 \left( 1 - \veps^{2 \beta_l} s^2 \right) + s^2 \right) \ge s^{2 l} \ge 1.
	\end{equation*}
    Otherwise, $1 - \veps^{2 \beta_l} s^2 \ge 1/2$, and consequently
    \begin{align*}
    	s^{2 (l - 1)} \left( l^2 a^2 \left( 1 - \veps^{2 \beta_l} s^2 \right) + s^2 \right) \ge\, & s^{2 (l - 1)} \left( \frac{l^2 a^2}{2} + s^2 \right) \ge s^{2(l-1)} \left( \frac{l^2 r^2}{2 s^{2l}} + s^2 \right) \\
    	=\, & \frac{l^2 r^2}{2 s^2} + s^{2l} \ge c(l, r),
    \end{align*}
    where the last line follows from the AM-GM inequality. This completes the proof.
\end{proof}

\section{Proofs of Theorem~\ref{thm:diff_gf_psgd} and \ref{thm:MainLearning}: learning with projected SGD}\label{sec:coupling_bd}
We will prove Theorem~\ref{thm:diff_gf_psgd} which bounds the distance between GF and projected SGD in sub-Sections~\ref{sec:diff_gf_gd} through \ref{sec:diff_sgd_psgd}, with sub-Section~\ref{sec:proof_MainLearning} devoted to the proof of Theorem~\ref{thm:MainLearning}.
Throughout this section, we use $M$ to refer to any constant that only depends on the $M_i$'s from Assumptions A1-A3, whereas the value of $M$ can change from line to line. We start with an elementary lemma that establishes the Lipschitz continuity of the gradient flow trajectory:
\begin{lem}[A priori estimate]\label{lem:a_estimate}
	There exists a constant $M$ that only depends on the $M_i$'s, such that for all $t \ge 0$, $\rho_t$ is supported on $[- M(1 + t / \veps), M(1 + t / \veps)] \times \S^{d - 1}$, namely $\vert a_i (t) \vert \le M(1 + t / \veps)$ for all $i \in [m]$. Moreover, for any $0 \le s \le t$, we have
	\begin{align*}
		\sup_{j \in [m]} \left\vert a_j (s) - a_j (t) \right\vert \le \, & \veps^{-1} M (t - s), \\
		\sup_{j \in [m]} \norm{u_j (s) - u_j (t)}_2 \le \, & M (1 + t / \veps)^2 (t - s).
	\end{align*}
\end{lem}
\begin{proof}
	First, notice that along the trajectory of gradient flow, the risk must be non-increasing. In fact, we have
	\begin{equation*}
		\partial_t \Risk = - m \sum_{i = 1}^{m} \left( \veps^{-2} (\partial_{a_i} \Risk)^2 +  (\nabla_{u_i} \Risk)^\top (I_d - u_i u_i^\top) (\nabla_{u_i} \Risk) \right) \le 0.
	\end{equation*}
	Therefore, we obtain that
	\begin{align*}
		\left\vert \partial_t a_i \right\vert = \, & \veps^{-1} \left\vert \E [y \sigma(\langle u_i, x \rangle)] - \frac{1}{m} \sum_{j = 1}^{m} a_j \E [\sigma(\langle u_i, x \rangle) \sigma(\langle u_j, x \rangle)] \right\vert = \veps^{-1} \left\vert \E \left[ (y - \hat{y}) \sigma(\langle u_i, x \rangle) \right] \right\vert \\
		\le \, & \veps^{-1} \E \left[ (y - \hat{y})^2 \right]^{1/2} \E \left[ \sigma(\langle u_i, x \rangle)^2 \right]^{1/2} \le \veps^{-1} \sqrt{2 \Risk (0)} \norm{\sigma}_{L^2} \le \veps^{-1} M,
	\end{align*}
	where the last line follows from our assumption. Since $\vert a_i (0) \vert \le M$, we know that $\vert a_i (t) \vert \le M (1 + t / \veps)$, and $\vert a_i (t) - a_i (s) \vert \le \veps^{-1} M (t - s)$. Moreover, according to Eq.~\eqref{eq:GF-2}, we have
	\begin{equation*}
		\norm{\partial_t u_i}_2 \le \vert a_i \vert \left( \norm{V'}_{\infty} + \norm{U'}_{\infty} \sup_{j \in [m]} \vert a_j \vert \right) \le M (1 + t / \veps)^2,
	\end{equation*}
	thus leading to
	\begin{equation*}
		\norm{u_i (s) - u_i (t)}_2 \le M (1 + t / \veps)^2 (t - s), \ \forall i \in [m].
	\end{equation*}
	This completes the proof.
\end{proof}

In what follows we define two discretized versions of Eq.s~\eqref{eq:GF-1} and \eqref{eq:GF-2}, namely the gradient descent (GD) and stochastic gradient descent (SGD) dynamics. They will serve as important intermediate objects for our proof.
\begin{itemize}
	\item Gradient descent: Let $\eta > 0$ be the step size, and let the initialization be the same as gradient flow: $(\tilde{a}_i (0), \tilde{u}_i (0)) = (a_i (0), u_i (0))$ for all $i \in [m]$. We have for $k \in \mathbb{N}$,
	\begin{equation}\label{eq:gd_MF}
		\begin{split}
			\tilde{a}_i (k + 1) - \tilde{a}_i (k) = \, & - m \veps^{-1} \eta \partial_{\tilde{a}_i (k)} \Risk \\
			&= \veps^{-1} \eta \Bigg( V ( \langle u_*, \tilde{u}_i (k) \rangle; \norm{u_*}_2, \norm{\tilde{u}_i (k)}_2 ) \\
			&\qquad\qquad- \frac{1}{m} \sum_{j=1}^{m} \tilde{a}_j (k) U (\langle \tilde{u}_i (k), \tilde{u}_j (k) \rangle; \norm{\tilde{u}_i (k)}_2, \norm{\tilde{u}_j (k)}_2 ) \Bigg) \\
			\tilde{u}_i (k + 1) - \tilde{u}_i (k) = \, & - m \eta \left( I_d - \tilde{u}_i (k) \tilde{u}_i (k)^\top \right) \nabla_{\tilde{u}_i (k)} \Risk \\
			= \, & \eta \tilde{a}_i (k) \left( I_d - \tilde{u}_i (k) \tilde{u}_i (k)^\top \right) \Bigg( \nabla_{\tilde{u}_i (k)} V ( \langle u_*, \tilde{u}_i (k) \rangle; \norm{u_*}_2, \norm{\tilde{u}_i (k)}_2 ) \\
			&\qquad - \frac{1}{m} \sum_{j=1}^{m} \tilde{a}_j (k)  \nabla_{\tilde{u}_i (k)} U (\langle \tilde{u}_i (k), \tilde{u}_j (k) \rangle; \norm{\tilde{u}_i (k)}_2, \norm{\tilde{u}_j (k)}_2 ) \Bigg),
		\end{split}
	\end{equation}
	where we recall from Eq.s~\eqref{eq:new_def_V} and \eqref{eq:new_def_U}:
	\begin{align*}
		V ( \langle u_*, \tilde{u}_i (k) \rangle; \norm{u_*}_2, \norm{\tilde{u}_i (k)}_2 ) = \, & \E \left[ \varphi (\langle u_*, x \rangle) \sigma (\langle \tilde{u}_i (k), x \rangle) \right] = \E \left[ y \sigma (\langle \tilde{u}_i (k), x \rangle) \right], \\
		\nabla_{\tilde{u}_i (k)} V ( \langle u_*, \tilde{u}_i (k) \rangle; \norm{u_*}_2, \norm{\tilde{u}_i (k)}_2 ) = \, & \E \left[ \varphi (\langle u_*, x \rangle) \sigma' (\langle \tilde{u}_i (k), x \rangle) x \right] = \E \left[ y \sigma' (\langle \tilde{u}_i (k), x \rangle) x \right], \\
		U (\langle \tilde{u}_i (k), \tilde{u}_j (k) \rangle; \norm{\tilde{u}_i (k)}_2, \norm{\tilde{u}_j (k)}_2) = \, & \E \left[ \sigma (\langle \tilde{u}_i (k), x \rangle) \sigma (\langle \tilde{u}_j (k), x \rangle) \right], \\
		\nabla_{\tilde{u}_i (k)} U (\langle \tilde{u}_i (k), \tilde{u}_j (k) \rangle; \norm{\tilde{u}_i (k)}_2, \norm{\tilde{u}_j (k)}_2) = \, & \E \left[ x \sigma' (\langle \tilde{u}_i (k), x \rangle) \sigma (\langle \tilde{u}_j (k), x \rangle) \right].
	\end{align*}
	By convention, we have $V(s) = V(s; 1, 1)$ and $U(s) = U(s; 1, 1)$ for $s \in [-1, 1]$.
	
	\item One-pass stochastic gradient descent: Under the same choice of the step size and initialization, and let $\{ (x_k, y_k) \}_{k \in \mathbb{N}^*}$ be i.i.d. samples from $\mathrm{P} \in \mathscr{P} (\R^d \times \R)$, where
	\begin{equation*}
		\mathrm{P} = \operatorname{Law} (x, y), \quad x \sim \sN (0, I_d), \ y = \varphi (\langle u_*, x \rangle).
	\end{equation*}
	The iteration equations for one-pass SGD read:
	\begin{equation}\label{eq:sgd_MF}
		\begin{split}
			\underline{a}_i (k + 1) - \underline{a}_i (k) = \, & \veps^{-1} \eta \left( y_{k+1} - \frac{1}{m} \sum_{j = 1}^{m} \underline{a}_j (k) \sigma (\langle \underline{u}_j (k), x_{k + 1} \rangle) \right) \sigma (\langle \underline{u}_i (k), x_{k+1} \rangle) \\
			\underline{u}_i (k + 1) - \underline{u}_i (k) = \, & \eta \underline{a}_i (k) \left( I_d - \underline{u}_i (k) \underline{u}_i (k)^\top \right) \left( y_{k+1} - \frac{1}{m} \sum_{j=1}^{m} \underline{a}_j (k) \sigma (\langle \underline{u}_j (k), x_{k+1} \rangle) \right) \\
			& \times \sigma' (\langle \underline{u}_i (k), x_{k+1} \rangle) x_{k+1}.
		\end{split}
	\end{equation}
	Note that Eq.~\eqref{eq:sgd_MF} can also be written as:
	\begin{align*}
		\underline{a}_i (k + 1) = \, & \underline{a}_i (k) + \veps^{-1} \eta \hat{F}_i (\underline{\rho}^{(m)} (k); z_{k + 1}) \\
		\underline{u}_i (k + 1) = \, & \underline{u}_i (k) + \eta \left( I_d - \underline{u}_i (k) \underline{u}_i (k)^\top \right) \hat{G}_i (\underline{\rho}^{(m)} (k); z_{k + 1}).
	\end{align*}
\end{itemize}

\subsection{Difference between GF and GD}\label{sec:diff_gf_gd}
For notational simplicity, we denote $\theta_i (t) = (a_i (t), u_i (t))$ for $i \in [m]$ and $t \ge 0$, and
\begin{equation*}
	\rho^{(m)} (t) = \frac{1}{m} \sum_{i = 1}^{m} \delta_{\theta_i (t)} = \frac{1}{m} \sum_{i=1}^{m} \delta_{(a_i (t), u_i (t))}.
\end{equation*}
Similarly, $\tilde{\theta}_i (k) = (\tilde{a}_i (k), \tilde{u}_i (k))$, and
\begin{equation*}
	\tilde{\rho}^{(m)} (k) = \frac{1}{m} \sum_{i=1}^{m} \delta_{\tilde{\theta}_i (k)} = \frac{1}{m} \sum_{i=1}^{m} \delta_{(\tilde{a}_i (k), \tilde{u}_i (k))}.
\end{equation*}
Moreover, for $\theta = (a, u)$ and $\rho \in \mathscr{P} (\R \times \R^d)$, we define the following two functionals:
\begin{align*}
	F (\theta, \rho) = \, & V (\langle u_*, u \rangle; \norm{u_*}_2, \norm{u}_2) - \int_{\R \times \R^d} a' U (\langle u, u' \rangle; \norm{u}_2, \norm{u'}_2) \rho (\mathrm{d} a', \mathrm{d} u'), \\
	G (\theta, \rho) = \, & a \left( I_d - u u^\top \right) \left( \nabla_u V (\langle u_*, u \rangle; \norm{u_*}_2, \norm{u}_2) - \int_{\R \times \R^d} a' \nabla_u U (\langle u, u' \rangle; \norm{u}_2, \norm{u'}_2) \rho (\mathrm{d} a', \mathrm{d} u') \right),
\end{align*}
and $H_{\veps} (\theta, \rho) = (\veps^{-1} F(\theta, \rho), G(\theta, \rho))$. Then, Eq.s~\eqref{eq:GF-1} and \eqref{eq:GF-2} and Eq.~\eqref{eq:gd_MF} can be rewritten as
\begin{equation*}
	\frac{\mathrm{d}}{\mathrm{d} t} \theta_i (t) = H_{\veps} (\theta_i (t), \rho^{(m)} (t)), \quad \tilde{\theta}_i (k + 1) - \tilde{\theta}_i (k) = \eta H_{\veps} (\tilde{\theta}_i (k), \tilde{\rho}^{(m)} (k)),
\end{equation*}
respectively. The lemma below will be used several times in the proof.
\begin{lem}\label{lem:grad_diff_bd}
	Denoting $\rho^{(m)} = (1 / m) \sum_{i=1}^{m} \delta_{\theta_i}$ and $\rho'^{(m)} = (1 / m) \sum_{i=1}^{m} \delta_{\theta'_i}$. If $\norm{u_i}_2 \le C$ and $\norm{u'_i}_2 \le C$ for all $i \in [m]$ ($C$ is any fixed absolute constant, for example, here we can take $C = 2$), then we have
	\begin{align}
		\left\vert F(\theta_i, \rho^{(m)}) - F(\theta'_i, \rho'^{(m)}) \right\vert \le \, & M \left( \left( 1 + \sup_{j \in [m]} \vert a_j \vert \right) \cdot \sup_{j \in [m]} \norm{u_j - u'_j}_2 + \sup_{j \in [m]} \left\vert a_j - a'_j \right\vert \right), \\
		\norm{ G(\theta_i, \rho^{(m)}) - G(\theta'_i, \rho'^{(m)}) }_2 \le \, & M \cdot \left( 1 + \sup_{j \in [m]} \left\vert a_j \right\vert \right)^2 \cdot \sup_{j \in [m]} \norm{u_j - u'_j}_2 \\
		& + M \cdot \sup_{j \in [m]} \left\vert a_j - a'_j \right\vert \cdot \left( 1 + \sup_{j \in [m]} \left\vert a_j \right\vert + \sup_{j \in [m]} \left\vert a_j - a'_j \right\vert \right),
	\end{align}
	where the constant $M$ only depends on the $M_i$'s. As a consequence, we obtain that
	\begin{align*}
		& \norm{ H_{\veps} (\theta_i, \rho^{(m)}) - H_{\veps} (\theta'_i, \rho'^{(m)}) }_2 \le \veps^{-1} \left\vert F(\theta_i, \rho^{(m)}) - F(\theta'_i, \rho'^{(m)}) \right\vert + \norm{ G(\theta_i, \rho^{(m)}) - G(\theta'_i, \rho'^{(m)}) }_2 \\
		\le \, & (\veps^{-1} + 1) M \cdot \left( \left( 1 + \sup_{j \in [m]} \left\vert a_j \right\vert \right)^2 \cdot \sup_{j \in [m]} \norm{u_j - u'_j}_2 + \sup_{j \in [m]} \left\vert a_j - a'_j \right\vert \cdot \left( 1 + \sup_{j \in [m]} \left\vert a_j \right\vert + \sup_{j \in [m]} \left\vert a_j - a'_j \right\vert \right) \right) \\
		\le \, & (\veps^{-1} + 1) M \cdot \left( \left( 1 + \sup_{j \in [m]} \left\vert a_j \right\vert \right)^2 + \sup_{j \in [m]} \norm{ \theta_j - \theta'_j}_2 \right) \cdot \sup_{j \in [m]} \norm{ \theta_j - \theta'_j}_2.
	\end{align*}
\end{lem}

\begin{proof}
	First, by triangle inequality, we have
	\begin{align*}
		& \left\vert F(\theta_i, \rho^{(m)}) - F(\theta'_i, \rho'^{(m)}) \right\vert \le \left\vert V(\langle u_*, u_i \rangle; \norm{u_*}_2, \norm{u_i}_2) - V(\langle u_*, u'_i \rangle; \norm{u_*}_2, \norm{u'_i}_2) \right\vert \\
		& + \frac{1}{m} \sum_{j=1}^{m} \left\vert a_j U (\langle u_i, u_j \rangle; \norm{u_i}_2, \norm{u_j}_2) - a'_j U (\langle u'_i, u'_j \rangle; \norm{u'_i}_2, \norm{u'_j}_2) \right\vert \\
		\le \, & \norm{\nabla V}_{\infty} \norm{u_i - u'_i}_2 + \frac{\norm{U}_{\infty}}{m} \sum_{j=1}^{m} \left\vert a_j - a'_j \right\vert + \frac{\norm{\nabla U}_{\infty}}{m} \sum_{j=1}^{m} \vert a_j \vert \cdot \left( \norm{u_i - u'_i}_2 + \norm{u_j - u'_j}_2 \right) \\
		\le \, & M \left( \left( 1 + \sup_{j \in [m]} \vert a_j \vert \right) \norm{u_i - u'_i}_2 + \sup_{j \in [m]} \left\vert a_j - a'_j \right\vert + \sup_{j \in [m]} \vert a_j \vert \cdot \sup_{j \in [m]} \norm{u_j - u'_j}_2 \right) \\
		\le \, & M \left( \left( 1 + \sup_{j \in [m]} \vert a_j \vert \right) \cdot \sup_{j \in [m]} \norm{u_j - u'_j}_2 + \sup_{j \in [m]} \left\vert a_j - a'_j \right\vert \right).
	\end{align*}
	Second, using again triangle inequality, we deduce that
	\begin{align*}
		& \norm{ G(\theta_i, \rho^{(m)}) - G(\theta'_i, \rho'^{(m)}) }_2 \\
		\stackrel{(i)}{\le} \, & 2 C \left\vert a_i \right\vert \norm{u_i - u'_i}_{2} \left( \norm{\nabla V}_{\infty} + \frac{\norm{\nabla U}_{\infty}}{m} \sum_{j=1}^{m} \left\vert a_j \right\vert \right) + C \left\vert a_i \right\vert \cdot \frac{\norm{\nabla U}_{\infty}}{m} \sum_{j=1}^{m} \left\vert a_j \right\vert \cdot \norm{u_j - u'_j}_2 \\
		& + C \left\vert a_i \right\vert \cdot \left( \norm{\nabla^2 V}_{\infty} \norm{u_i - u'_i}_2 + \frac{\norm{\nabla^2 U}_{\infty}}{m} \sum_{j=1}^{m} \left\vert a_j \right\vert \left( \norm{u_i - u'_i}_2 + \norm{u_j - u'_j}_2 \right) \right) \\
		& + C \left\vert a_i - a'_i \right\vert \left( \norm{\nabla V}_\infty + \frac{\norm{\nabla U}_{\infty}}{m} \sum_{j=1}^{m} \left\vert a_j \right\vert \right) + C \left( \left\vert a_i \right\vert + \left\vert a'_i - a_i \right\vert \right) \cdot \frac{\norm{\nabla U}_{\infty}}{m} \sum_{j=1}^{m} \left\vert a_j - a'_j \right\vert \\
		\le \, & 5 M \cdot \left( 1 + \sup_{j \in [m]} \left\vert a_j \right\vert \right)^2 \cdot \sup_{j \in [m]} \norm{u_j - u'_j}_2 + M \cdot \sup_{j \in [m]} \left\vert a_j - a'_j \right\vert \cdot \left( 1 + 2 \sup_{j \in [m]} \left\vert a_j \right\vert + \sup_{j \in [m]} \left\vert a_j - a'_j \right\vert \right) \\
		\le \, & M \cdot \left( 1 + \sup_{j \in [m]} \left\vert a_j \right\vert \right)^2 \cdot \sup_{j \in [m]} \norm{u_j - u'_j}_2 + M \cdot \sup_{j \in [m]} \left\vert a_j - a'_j \right\vert \cdot \left( 1 + \sup_{j \in [m]} \left\vert a_j \right\vert + \sup_{j \in [m]} \left\vert a_j - a'_j \right\vert \right),
	\end{align*}
	where $(i)$ follows from the inequality $\norm{u_i u_i^\top - (u'_i) (u'_i)^\top}_{\op} \le 2 C \norm{u_i - u'_i}_2$, which is a result of the following direct calculation:
	\begin{align*}
		\norm{u_i u_i^\top - (u'_i) (u'_i)^\top}_{\op} = \sup_{\norm{x}_2 = 1} \left\vert \langle x, u_i \rangle^2 - \langle x, u'_i \rangle^2 \right\vert \le 2 C \sup_{\norm{x}_2 = 1} \left\vert \langle x, u_i \rangle - \langle x, u'_i \rangle \right\vert = 2 C \norm{u_i - u'_i}_2.
	\end{align*}
	This completes the proof of Lemma~\ref{lem:grad_diff_bd}, since the ``as a consequence'' part follows naturally from the upper bounds obtained earlier.
\end{proof}

\begin{lem}\label{lem:risk_diff_bd}
	Following the notation and assumption of Lemma~\ref{lem:grad_diff_bd}, we have
	\begin{equation*}
		\left\vert \Risk (\rho^{(m)}) - \Risk (\rho'^{(m)}) \right\vert \le M \cdot \left( \left( 1 + \sup_{j \in [m]} \left\vert a_j \right\vert \right)^2 + \sup_{j \in [m]} \norm{ \theta_j - \theta'_j}_2 \right) \cdot \sup_{j \in [m]} \norm{ \theta_j - \theta'_j}_2.
	\end{equation*}
\end{lem}

\begin{proof}
	By definition of the risk function and triangle inequality, we deduce that
	\begin{align*}
		\left\vert \Risk (\rho^{(m)}) - \Risk (\rho'^{(m)}) \right\vert \le \, & \frac{1}{m} \sum_{i = 1}^{m} \left\vert a_i V(\langle u_*, u_i \rangle; \norm{u_*}_2, \norm{u_i}_2) - a'_i V(\langle u_*, u'_i \rangle; \norm{u_*}_2, \norm{u'_i}_2) \right\vert \\
		& + \frac{1}{m^2} \sum_{i, j = 1}^{m} \left\vert a_i a_j U(\langle u_i, u_j \rangle; \norm{u_i}_2, \norm{u_j}_2) - a'_i a'_j U(\langle u'_i, u'_j \rangle; \norm{u'_i}_2, \norm{u'_j}_2) \right\vert \\
		\le \, & \frac{\norm{\nabla V}_{\infty}}{m} \sum_{i=1}^{m} \left\vert a_i \right\vert \norm{u_i - u'_i}_2 + \frac{\norm{V}_{\infty}}{m} \sum_{i=1}^{m} \left\vert a_i - a'_i \right\vert \\
		& + \frac{\norm{U}_\infty}{m^2} \sum_{i, j = 1}^{m} \left( \left\vert a_i - a'_i \right\vert \left\vert a'_j \right\vert + \left\vert a_i \right\vert \left\vert a_j - a'_j \right\vert \right) \\
		& + \frac{\norm{\nabla U}_{\infty}}{m^2} \sum_{i, j=1}^{m} \left\vert a_i \right\vert \left\vert a_j \right\vert \cdot \left( \norm{u_i - u'_i}_2 + \norm{u_j - u'_j}_2 \right) \\
		\le \, & M \cdot \left( 1 + \sup_{j \in [m]} \left\vert a_j \right\vert \right)^2 \cdot \sup_{j \in [m]} \norm{u_j - u'_j}_2 \\
		& + M \cdot \sup_{j \in [m]} \left\vert a_j - a'_j \right\vert \cdot \left( 1 + \sup_{j \in [m]} \left\vert a_j \right\vert + \sup_{j \in [m]} \left\vert a_j - a'_j \right\vert \right) \\
		\le \, & M \cdot \left( \left( 1 + \sup_{j \in [m]} \left\vert a_j \right\vert \right)^2 + \sup_{j \in [m]} \norm{ \theta_j - \theta'_j}_2 \right) \cdot \sup_{j \in [m]} \norm{ \theta_j - \theta'_j}_2.
	\end{align*}
	This concludes the proof.
\end{proof}
First, let us define the error function
\begin{equation*}
	\Delta (t) = \sup_{k \in [0, t/\eta] \cap \mathbb{N}} \max_{i \in [m]} \norm{\tilde{\theta}_i (k) - \theta_i (k \eta)}_2,
\end{equation*}
and the stopping time $T_{\Delta} = \inf \{ t \ge 0: \Delta(t) \ge 1 \}$. For $k \in \mathbb{N}$ and $t = k \eta \le T_{\Delta}$, we have the following estimate:
\begin{align*}
	\norm{\theta_i (t) - \tilde{\theta}_i (k)}_2 \le \, & \int_{0}^{t} \norm{ H_{\veps} \left( \theta_i (s), \rho^{(m)} (s) \right) - H_{\veps} \left( \tilde{\theta}_i (\floor{s / \eta}), \tilde{\rho}^{(m)} (\floor{s / \eta}) \right)}_2 \mathrm{d} s \\
	\le \, & \int_{0}^{t} \norm{H_{\veps} \left( \theta_i (s), \rho^{(m)} (s) \right) - H_{\veps} \left( \theta_i (\eta \floor{s / \eta}), \rho^{(m)} (\eta \floor{s / \eta}) \right)}_2 \mathrm{d} s \\
	& + \int_{0}^{t} \norm{H_{\veps} \left( \theta_i (\eta \floor{s / \eta}), \rho^{(m)} (\eta \floor{s / \eta}) \right) - H_{\veps} \left( \tilde{\theta}_i (\floor{s / \eta}), \tilde{\rho}^{(m)} (\floor{s / \eta}) \right)}_2 \mathrm{d} s.
\end{align*}
For any $s \in [0, t]$, by Lemma~\ref{lem:a_estimate} and~\ref{lem:grad_diff_bd} we have (denote $[s] = \eta \floor{s / \eta}$, and notice that we can take $C=2$ since $t \le T_{\Delta}$)
\begin{align*}
	& \norm{ H_{\veps} \left( \theta_i (s), \rho^{(m)} (s) \right) - H_{\veps} \left( \theta_i ([s]), \rho^{(m)} ([s]) \right) }_2 \\
	\le \, & (\veps^{-1} + 1) M (1 + t / \veps)^4 (s - [s]) \le (\veps^{-1} + 1) M (1 + t / \veps)^4 \eta.
\end{align*}
Using again Lemma~\ref{lem:a_estimate} and~\ref{lem:grad_diff_bd}, we obtain that
\begin{align*}
	& \norm{H_{\veps} \left( \theta_i (\eta \floor{s / \eta}), \rho^{(m)} (\eta \floor{s / \eta}) \right) - H_{\veps} \left( \tilde{\theta}_i (\floor{s / \eta}), \tilde{\rho}^{(m)} (\floor{s / \eta}) \right)}_2 \\
	\le \, & (\veps^{-1} + 1) M (1 + \veps^{-1} s)^2 \Delta (s) + (\veps^{-1} + 1) M \Delta(s)^2,
\end{align*}
thus leading to
\begin{align*}
	\Delta (t) \le \, & (\veps^{-1} + 1) M t (1 + t / \veps)^4 \eta + (\veps^{-1} + 1) \int_{0}^{t} \left( M (1 + \veps^{-1} s)^2 \Delta (s) + M \Delta(s)^2 \right) \mathrm{d} s \\
	\le \, & (\veps^{-1} + 1) M t (1 + t / \veps)^4 \eta + (\veps^{-1} + 1) \int_{0}^{t} M (1 + \veps^{-1} s)^2 \cdot \max \left( \Delta (s), \Delta(s)^2 \right) \mathrm{d} s.
\end{align*}
For $s \le t \le T_{\Delta}$, we have $\Delta(s)^2 \le \Delta(s)$. Hence,
\begin{equation*}
	\Delta(t) \le (\veps^{-1} + 1) M t (1 + t / \veps)^4 \eta + (\veps^{-1} + 1) \int_{0}^{t} M (1 + \veps^{-1} s)^2 \Delta (s) \mathrm{d} s.
\end{equation*}
Applying Gr\"{o}nwall's inequality yields
\begin{equation*}
	\Delta (t) \le (\veps^{-1} + 1) M t (1 + t / \veps)^4 \eta \cdot \exp \left( (\veps^{-1} + 1 ) M t (1 + t / \veps)^2 \right) \le M \exp((\veps^{-1} + 1) M t (1 + t / \veps)^2) \eta.
\end{equation*}
Therefore, for all $T \ge 0$ and $\eta \le 1 / (M \exp((\veps^{-1} + 1) M T (1 + \veps^{-1} T)^2))$, we have
\begin{equation*}
	\sup_{k \in [0, t/\eta] \cap \mathbb{N}} \max_{i \in [m]} \norm{\tilde{\theta}_i (k) - \theta_i (k \eta)}_2 \le M \exp((\veps^{-1} + 1) M t (1 + t / \veps)^2) \eta \le 1, \ \forall t \le \min(T, T_{\Delta}).
\end{equation*}
This proves $T \le T_{\Delta}$, and consequently
\begin{equation*}
	\sup_{k \in [0, t/\eta] \cap \mathbb{N}} \max_{i \in [m]} \norm{\tilde{\theta}_i (k) - \theta_i (k \eta)}_2 \le M \exp((\veps^{-1} + 1) M t (1 + t / \veps)^2) \eta, \ \forall t \in [0, T],
\end{equation*}
which immediately implies that
\begin{equation*}
	\sup_{k \in [0, t/\eta] \cap \mathbb{N}} \max_{i \in [m]} \left\vert \tilde{a}_i (k) \right\vert \le M(1 + t / \veps) + 1 \le M (1 + t / \veps).
\end{equation*}
Finally, with the aid of Lemma~\ref{lem:risk_diff_bd}, we get the following upper bound on the difference between the risk of gradient flow and gradient descent:
\begin{align*}
	\sup_{k \in [0, t/\eta] \cap \mathbb{N}} \left\vert \Risk ( \rho^{(m)} (k \eta) ) - \Risk ( \tilde{\rho}^{(m)} (k) ) \right\vert \le \, & M (M^2 (1 + t / \veps)^2 + 1) M \exp((\veps^{-1} + 1) M t (1 + t / \veps)^2) \eta \\
	\le \, & M \exp((\veps^{-1} + 1) M t (1 + t / \veps)^2) \eta.
\end{align*}
To summarize, we have the following:
\begin{thm}[Difference between GF and GD]\label{thm:diff_gf_gd}
	There exists a constant $M$ that only depends on the $M_i$'s, such that for any $T \ge 0$ and
	\begin{equation*}
		\eta \le \frac{1}{M \exp((\veps^{-1} + 1) M T (1 + \veps^{-1} T)^2)},
	\end{equation*}
	the following holds for all $t \in [0, T]$:
	\begin{align}
		\sup_{k \in [0, t/\eta] \cap \mathbb{N}} \max_{i \in [m]} \left\vert \tilde{a}_i (k) \right\vert \le \, & M(1 + t / \veps), \\
		\sup_{k \in [0, t/\eta] \cap \mathbb{N}} \max_{i \in [m]} \norm{\tilde{\theta}_i (k) - \theta_i (k \eta)}_2 \le \, & M \exp((\veps^{-1} + 1) M t (1 + \veps^{-1} T)^2) \eta, \\
		\sup_{k \in [0, t/\eta] \cap \mathbb{N}} \left\vert \Risk ( \rho^{(m)} (k \eta) ) - \Risk ( \tilde{\rho}^{(m)} (k) ) \right\vert \le \, & M \exp((\veps^{-1} + 1) M t (1 + \veps^{-1} T)^2) \eta.
	\end{align}
\end{thm}

\subsection{Difference between GD and SGD}\label{sec:diff_gd_sgd}
The proof for this section is almost identical to Appendix C.5 in \citep{mei2019mean}. The only difference is that, here we need to verify that $(I_d - u u^\top) \sigma' (\langle u, x \rangle) x$ is an $M_3$-sub-Gaussian random vector. This follows from the identity $(I_d - u u^\top) x = x - \langle u, x \rangle u$ and Assumption A3. We thus obtain the following interpolation bound between GD and SGD:
\begin{thm}[Difference between GD and SGD]\label{thm:diff_gd_sgd}
	There exists a constant $M$ that only depends on the $M_i$'s, such that for any $T, z \ge 0$ and
	\begin{equation*}
		\eta \le \frac{1}{(d + \log m + z^2) M \exp((\veps^{-1} + 1) M T (1 + \veps^{-1} T)^2)},
	\end{equation*}
	the following happens with probability at least $1 - \exp(- z^2)$: For all $t \in [0, T]$, we have
	\begin{align}
		\sup_{k \in [0, t/\eta] \cap \mathbb{N}} \max_{i \in [m]} \left\vert \underline{a}_i (k) \right\vert \le \, & M(1 + t / \veps), \\
		\sup_{k \in [0, t/\eta] \cap \mathbb{N}} \max_{i \in [m]} \norm{\tilde{\theta}_i (k) - \underline{\theta}_i (k)}_2 \le \, & M \exp((\veps^{-1} + 1) M t (1 + \veps^{-1} T)^2) \sqrt{\eta} \left( \sqrt{d + \log m} + z \right), \\
		\sup_{k \in [0, t/\eta] \cap \mathbb{N}} \left\vert \Risk ( \underline{\rho}^{(m)} (k) ) - \Risk ( \tilde{\rho}^{(m)} (k) ) \right\vert \le \, & M \exp((\veps^{-1} + 1) M t (1 + \veps^{-1} T)^2) \sqrt{\eta} \left( \sqrt{d + \log m} + z \right).
	\end{align}
\end{thm}

\subsection{Difference between SGD and projected SGD}\label{sec:diff_sgd_psgd}
The aim of this section is to prove a coupling bound between the trajectory of SGD and that of projected SGD, thus finally leading to an upper bound on the difference between the risk of projected gradient flow and projected SGD. To begin with, let us fix $T, z \ge 0$ and choose
\begin{equation*}
	\eta \le \frac{1}{(d + \log m + z^2) M \exp((\veps^{-1} + 1) M T (1 + \veps^{-1} T)^2)}
\end{equation*}
as in Theorem~\ref{thm:diff_gf_psgd}, where $M$ is a large enough constant (to be determined later). Define
\begin{equation*}
	T_{\theta} = \inf \left\{ t \ge 0: \max_{k \in [0, t / \eta] \cap \mathbb{N}} \max_{i \in [m]}  \left\vert \overline{a}_i (k) \right\vert \ge 2 M (1 + t / \veps), \ \text{or} \ \max_{k \in [0, t / \eta] \cap \mathbb{N}} \max_{i \in [m]}  \norm{ \overline{u}_i (k)}_2 \ge 2 \right\},
\end{equation*}
then for $k \le \min(T, T_{\theta}) / \eta$ and $i \in [m]$, we have (note that here $t = k \eta$)
\begin{align*}
	\norm{\hat{G}_i (\overline{\rho}^{(m)} (k); z_{k + 1})}_2 \le \, & M \left\vert \bar{a}_i (k) \right\vert \left( 1 + \max_{i \in [m]}  \left\vert \overline{a}_i (k) \right\vert \right) \norm{\sigma'(\langle \bar{u}_i (k), x_{k + 1} \rangle) x_{k+1}}_2 \\
	\le \, & M (1 + t / \veps)^2 \norm{\sigma'(\langle \bar{u}_i (k), x_{k + 1} \rangle) x_{k+1}}_2.
\end{align*}
Denoting $\cF_k = \sigma (\bar{\theta} (0), z_1, \cdots, z_k)$, we know from Assumption A3 that, conditioning on $\cF_k$, $\sigma'(\langle \bar{u}_i (k), x_{k + 1} \rangle) x_{k+1}$ is an $M_3$-sub-Gaussian random vector. By well-known results on Euclidean norm of sub-Gaussian random vectors (see, e.g., \cite{jin2019short}), we know that there exists a constant $M$ satisfying
\begin{equation*}
	\P \left( \norm{\sigma'(\langle \bar{u}_i (k), x_{k + 1} \rangle) x_{k+1}}_2 \ge M \left( \sqrt{d} + \sqrt{\log (1 / \delta)} \right) \right) \le \delta.
\end{equation*}
Choosing $\delta = \eta \exp(-z^2)/(mT)$ and applying a union bound gives
\begin{equation*}
	\P \left( \max_{k \in [0, \min(T, T_{\theta}) / \eta] \cap \mathbb{N}} \max_{i \in [m]} \norm{\sigma'(\langle \bar{u}_i (k), x_{k + 1} \rangle) x_{k+1}}_2 \le M \left( \sqrt{d + \log m} + z + T^2 \right) \right) \ge 1 - \exp(- z^2).
\end{equation*}
Therefore, with probability at least $1 - \exp(- z^2)$, for all $k \le \min(T, T_{\theta}) / \eta$ and $i \in [m]$, we have
\begin{equation*}
	\norm{\hat{G}_i (\overline{\rho}^{(m)} (k); z_{k + 1})}_2 \le M (1 + t / \veps)^2 \left( \sqrt{d + \log m} + z + T^2 \right).
\end{equation*}
The above bound also holds for the trajectory of SGD, namely after replacing $\overline{\rho}^{(m)} (k)$ with $\underline{\rho}^{(m)} (k)$. Now, let us define the approximation error $\Delta_i (k) = \underline{u}_i (k) - \overline{u}_i (k)$ for $i \in [m]$ and $k \in \mathbb{N}$, then we get the following decomposition:
\begin{equation*}
	\Delta_i (l) = \sum_{k=0}^{l - 1} \left( \Delta_i (k + 1) - \Delta_i (k) \right) = \sum_{k = 0}^{l - 1} \left( \E \left[ \Delta_i (k + 1) - \Delta_i (k) \vert \cF_k \right] + Z_i (k+1) \right),
\end{equation*} 
where $Z_i (k+1) = \Delta_i (k + 1) - \Delta_i (k) - \E \left[ \Delta_i (k + 1) - \Delta_i (k) \vert \cF_k \right]$ has zero mean. With our choice of $\eta$, one can verify that as long as $\max (d, m, z) \to \infty$, Lemma~\ref{lem:aux_diff_bd} is applicable to
\begin{equation*}
	u_1 = \underline{u}_i (k), \ g_1 = \hat{G}_i (\underline{\rho}^{(m)} (k); z_{k + 1}), \ u_2 = \overline{u}_i (k), \ g_2 = \hat{G}_i (\overline{\rho}^{(m)} (k); z_{k + 1}).
\end{equation*} 
Hence, we deduce from the definition of $\Delta_i (k)$ that
\begin{align*}
	\Delta_i (k + 1) - \Delta_i (k) =\, & \left( \underline{u}_i (k + 1) - \underline{u}_i (k) \right) - \left( \overline{u}_i (k + 1) - \overline{u}_i (k) \right) = (v_1 - u_1) - (v_2 - u_2) \\
	=\, & \eta \left( \left( I_d - u_1 u_1^\top \right) g_1 - \left( I_d - u_2 u_2^\top \right) g_2 \right) + O \left( \eta^2 \norm{g_2}_2^2 \right),
\end{align*}
thus leading to the following estimate:
\begin{align*}
	\norm{\E \left[ \Delta_i (k + 1) - \Delta_i (k) \vert \cF_k \right]}_2 \le\, & \eta \norm{\E \left[ \left( I_d - u_2 u_2^\top \right) (g_1 - g_2) \Big\vert \cF_k \right]}_2 + \eta \norm{\E \left[ \left( u_2 u_2^\top - u_1 u_1^\top \right) g_1 \Big\vert \cF_k \right]}_2 \\
	& + C \eta^2 \E \left[ \norm{g_2}_2^2 \Big\vert \cF_k \right] \\
	\stackrel{(i)}{\le}\, & \eta \norm{\E \left[ (g_1 - g_2) \vert \cF_k \right]}_2 + C \eta \norm{u_1 - u_2}_2 \norm{\E \left[ g_1 \vert \cF_k \right]}_2 + C \eta^2 \E \left[ \norm{g_2}_2^2 \Big\vert \cF_k \right] \\
	=\, & \eta \norm{\E \left[ \left( \hat{G}_i (\underline{\rho}^{(m)} (k); z_{k + 1}) - \hat{G}_i (\overline{\rho}^{(m)} (k); z_{k + 1}) \right) \Big\vert \cF_k \right]}_2 \\
	& + C \eta \norm{\overline{u}_i (k) - \underline{u}_i (k)}_2 \cdot \norm{\E \left[ \hat{G}_i (\underline{\rho}^{(m)} (k); z_{k + 1}) \Big\vert \cF_k \right]}_2 \\
	& + C \eta^2 \E \left[ \norm{\hat{G}_i (\overline{\rho}^{(m)} (k); z_{k + 1})}_2^2 \Big\vert \cF_k \right],
\end{align*}
where $(i)$ is due to the fact that $u_1, u_2 \in \sigma(\cF_k)$, and $\norm{u_1 u_1^\top - u_2 u_2^\top}_{\op} \le C \norm{u_1 - u_2}_2$. According to the definition of $\hat{G}_i$, we obtain that
\begin{align*}
	& \E \left[ \left( \hat{G}_i (\underline{\rho}^{(m)} (k); z_{k + 1}) - \hat{G}_i (\overline{\rho}^{(m)} (k); z_{k + 1}) \right) \Big\vert \cF_k \right] \\
	=\, & \underline{a}_i (k) \Big( \nabla_{\underline{u}_i (k)} V \left( \langle u_*, \underline{u}_i (k) \rangle; \norm{u_*}_2, \norm{\underline{u}_i (k)}_2 \right) \\
	& - \frac{1}{m} \sum_{j=1}^{m} \underline{a}_j (k) \nabla_{\underline{u}_i (k)} U \left( \langle \underline{u}_i (k), \underline{u}_j (k) \rangle; \norm{\underline{u}_i (k)}_2, \norm{\underline{u}_j (k)}_2 \right) \Big) \\
	& - \overline{a}_i (k) \Big( \nabla_{\overline{u}_i (k)} V \left( \langle u_*, \overline{u}_i (k) \rangle; \norm{u_*}_2, \norm{\overline{u}_i (k)}_2 \right) \\
	& - \frac{1}{m} \sum_{j=1}^{m} \overline{a}_j (k) \nabla_{\overline{u}_i (k)} U \left( \langle \overline{u}_i (k), \overline{u}_j (k) \rangle; \norm{\overline{u}_i (k)}_2, \norm{\overline{u}_j (k)}_2 \right) \Big),
\end{align*}
thus leading to (using the same argument as in the proof of Lemma~\ref{lem:grad_diff_bd})
\begin{align*}
	& \norm{\E \left[ \left( \hat{G}_i (\underline{\rho}^{(m)} (k); z_{k + 1}) - \hat{G}_i (\overline{\rho}^{(m)} (k); z_{k + 1}) \right) \Big\vert \cF_k \right]}_2 \le M \left( 1 + \veps^{-1} T \right)^2 \cdot \sup_{j \in [m]} \norm{\overline{u}_j (k) - \underline{u}_j (k)}_2 \\
	& + M \left( 1 + \veps^{-1} T + \sup_{j \in [m]} \left\vert \overline{a}_j (k) - \underline{a}_j (k) \right\vert \right) \cdot \sup_{j \in [m]} \left\vert \overline{a}_j (k) - \underline{a}_j (k) \right\vert,
\end{align*}
and
\begin{align*}
	\norm{\E \left[ \hat{G}_i (\underline{\rho}^{(m)} (k); z_{k + 1}) \Big\vert \cF_k \right]}_2 \le M (1 + \veps^{-1} T)^2.
\end{align*}
Moreover, by (conditional) sub-Gaussianity of the $\hat{G}_i$'s, we know that
\begin{equation*}
	\E \left[ \norm{\hat{G}_i (\overline{\rho}^{(m)} (k); z_{k + 1})}_2^2 \Big\vert \cF_k \right] \le M^2 (1 + \veps^{-1} T)^4 \E \left[ \norm{\sigma'(\langle \bar{u}_i (k), x_{k + 1} \rangle) x_{k+1}}_2^2 \vert \cF_k \right] \le M (1 + \veps^{-1} T)^4 d.
\end{equation*}
Combining the above estimates, it then follows that
\begin{align*}
	\norm{\E \left[ \Delta_i (k + 1) - \Delta_i (k) \vert \cF_k \right]}_2 \le\, & \eta M \left( 1 + \veps^{-1} T \right)^2 \cdot \sup_{j \in [m]} \norm{\overline{u}_j (k) - \underline{u}_j (k)}_2 + \eta^2 M (1 + \veps^{-1} T)^4 d \\
	& + \eta M \left( 1 + \veps^{-1} T + \sup_{j \in [m]} \left\vert \overline{a}_j (k) - \underline{a}_j (k) \right\vert \right) \cdot \sup_{j \in [m]} \left\vert \overline{a}_j (k) - \underline{a}_j (k) \right\vert.
\end{align*}

Using the same proof technique as in Appendix C.5 of \cite{mei2019mean}, we conclude that
\begin{equation*}
	\P \left( \max_{i \in [m]} \max_{l \in [0, \min(T, T_{\theta}) / \eta] \cap \mathbb{N}} \norm{\sum_{k = 0}^{l - 1} Z_i (k + 1)}_2 \ge M (1 + \veps^{-1} T)^2 \left( \sqrt{d + \log m} + z + T^2 \right) \sqrt{T \eta} \right) \le \exp (- z^2).
\end{equation*}
Similarly as in the proof of Theorem~\ref{thm:diff_gf_gd}, we define
\begin{equation*}
	\Delta (t) = \max_{l \in [0, t / \eta] \cap \mathbb{N}} \max_{i \in [m]} \norm{\overline{\theta}_i (l) - \underline{\theta}_i (l)}_2, \quad T_{\Delta} = \inf \{ t \ge 0: \Delta (t) \ge 1 \}.
\end{equation*}
Then, for $l \le \min (T, T_{\theta}, T_{\Delta}) / \eta$, we have
\begin{align*}
	\sup_{i \in [m]} \norm{\overline{u}_i (l) - \underline{u}_i (l)}_2 =\, & \sup_{i \in [m]} \norm{\Delta_i (l)}_2 \le \sup_{i \in [m]} \left\{ \sum_{k = 0}^{l - 1} \norm{\E \left[ \Delta_i (k + 1) - \Delta_i (k) \vert \cF_k \right]}_2 + \norm{\sum_{k = 0}^{l - 1} Z_i (k + 1)}_2 \right\} \\
	\le\, & \eta M (1 + \veps^{-1} T)^2 \sum_{k=0}^{l - 1} \Delta (k \eta) + l \eta^2 M (1 + \veps^{-1} T)^4 d \\
	& + M (1 + \veps^{-1} T)^2 \left( \sqrt{d + \log m} + z + T^2 \right) \sqrt{T \eta}.
\end{align*}
Proceeding with the same argument, it follows that
\begin{align*}
	\sup_{i \in [m]} \left\vert \overline{a}_i (l) - \underline{a}_i (l) \right\vert \le \veps^{-1} \eta M (1 + \veps^{-1} T)^2 \sum_{k=0}^{l - 1} \Delta (k \eta) + \veps^{-1} M (1 + \veps^{-1} T)^2 \left( \sqrt{d + \log m} + z + T^2 \right) \sqrt{T \eta}.
\end{align*}
Therefore, we finally conclude that
\begin{align*}
	\Delta (l \eta) \le\, & (\veps^{-1} + 1) \eta M (1 + \veps^{-1} T)^2 \sum_{k=0}^{l - 1} \Delta (k \eta) + l \eta^2 M (1 + \veps^{-1} T)^4 d \\
	& + (\veps^{-1} + 1) M (1 + \veps^{-1} T)^2 \left( \sqrt{d + \log m} + z + T^2 \right) \sqrt{T \eta} \\
	\le\, & (\veps^{-1} + 1) \eta M (1 + \veps^{-1} T)^2 \sum_{k=0}^{l - 1} \Delta (k \eta) + (\veps^{-1} + 1 ) M (1 + \veps^{-1} T)^4 \left( \sqrt{d + \log m} + z + T^2 \right) \sqrt{T \eta}.
\end{align*}
Applying Gr\"{o}nwall's inequality (discrete version) yields that
\begin{align*}
	\Delta (l \eta) \le\, & (\veps^{-1} + 1 ) M (1 + \veps^{-1} T)^4 \left( \sqrt{d + \log m} + z + T^2 \right) \sqrt{T \eta} \\
	& \times \left( 1 + (\veps^{-1} + 1 ) l \eta M (1 + \veps^{-1} T)^2 \exp \left( (\veps^{-1} + 1 ) l \eta M (1 + \veps^{-1} T)^2 \right) \right) \\
	\le\, & M \exp \left( (\veps^{-1} + 1 ) M T (1 + \veps^{-1} T)^2 \right) \left( \sqrt{d + \log m} + z + T^2 \right) \sqrt{\eta} \\
	\le\, & M \exp \left( (\veps^{-1} + 1 ) M T (1 + \veps^{-1} T)^2 \right) \left( \sqrt{d + \log m} + z \right) \sqrt{\eta},
\end{align*}
as long as $\max(d, m, z) \to \infty$ with $T = O(1)$. Note that the above inequality holds for all $l \in [0, \min(T, T_{\theta}, T_{\Delta}) / \eta] \cap \mathbb{N}$ with probability at least $1 - \exp(- z^2)$, which further implies that $T_{\theta}, T_{\Delta} \ge T$, and consequently
\begin{equation*}
	\sup_{k \in [0, T/\eta] \cap \mathbb{N}} \max_{i \in [m]} \left\vert \overline{a}_i (k) \right\vert \le 2M(1 + \veps^{-1} T).
\end{equation*}
Applying again Lemma~\ref{lem:risk_diff_bd}, we deduce that
\begin{equation*}
	\sup_{k \in [0, T/\eta] \cap \mathbb{N}} \left\vert \Risk ( \underline{\rho}^{(m)} (k) ) - \Risk ( \overline{\rho}^{(m)} (k) ) \right\vert \le \left( \sqrt{d + \log m} + z \right) M \exp((\veps^{-1} + 1) M T (1 + \veps^{-1} T)^2) \sqrt{\eta}.
\end{equation*}
Combining the above estimates gives the following:
\begin{thm}[Difference between SGD and projected SGD]\label{thm:diff_sgd_psgd}
	There exists a constant $M$ that only depends on the $M_i$'s, such that for any $T, z \ge 0$ and
	\begin{equation*}
		\eta \le \frac{1}{(d + \log m + z^2) M \exp((\veps^{-1} + 1 ) M T (1 + \veps^{-1} T)^2)},
	\end{equation*}
	the following happens with probability at least $1 - \exp(- z^2)$: For all $t \in [0, T]$, we have
	\begin{align}
		\sup_{k \in [0, t/\eta] \cap \mathbb{N}} \max_{i \in [m]} \left\vert \overline{a}_i (k) \right\vert \le \, & M(1 + t / \veps), \\
		\sup_{k \in [0, t/\eta] \cap \mathbb{N}} \max_{i \in [m]} \norm{\overline{\theta}_i (k) - \underline{\theta}_i (k)}_2 \le \, & M \exp((\veps^{-1} + 1) M t (1 + \veps^{-1} T)^2) \sqrt{\eta} \left( \sqrt{d + \log m} + z \right), \\
		\sup_{k \in [0, t/\eta] \cap \mathbb{N}} \left\vert \Risk ( \underline{\rho}^{(m)} (k) ) - \Risk ( \overline{\rho}^{(m)} (k) ) \right\vert \le \, & M \exp((\veps^{-1} + 1 ) M t (1 + \veps^{-1} T)^2) \sqrt{\eta} \left( \sqrt{d + \log m} + z \right).
	\end{align}
\end{thm}
Theorem~\ref{thm:diff_gf_psgd} then follows as a result of combining Theorem~\ref{thm:diff_gf_gd}, Theorem~\ref{thm:diff_gd_sgd}, and Theorem~\ref{thm:diff_sgd_psgd}.
\begin{lem}\label{lem:aux_diff_bd}
	Let $v_1 = u_1 + \eta (I_d - u_1 u_1^\top) g_1$, $v_2 = \operatorname{Proj}_{\S^{d - 1}} (u_2 + \eta g_2)$, where $\norm{u_2}_2 = 1$ and $\eta \norm{g_2}_2 \le 1/2$. Then we have
	\begin{equation*}
		(v_1 - u_1) - (v_2 - u_2) =\, \eta \left( \left( I_d - u_1 u_1^\top \right) g_1 - \left( I_d - u_2 u_2^\top \right) g_2 \right) + O \left( \eta^2 \norm{g_2}_2^2 \right).
	\end{equation*}
\end{lem}

\begin{proof}
	Using Taylor expansion, we know that
	\begin{align*}
		v_2 = \operatorname{Proj}_{\S^{d - 1}} (u_2 + \eta g_2) = \, & (u_2 + \eta g_2) \left( 1 + 2 \eta \langle u_2, g_2 \rangle + \eta^2 \norm{g_2}_2^2 \right)^{-1/2} \\
		= \, & (u_2 + \eta g_2) \left( 1 - \eta \langle u_2, g_2 \rangle + O(\eta^2 \norm{g_2}_2^2) \right) \\ 
		= \, & \left( 1 - \eta \langle u_2, g_2 \rangle \right) u_2 + \eta g_2 + O(\eta^2 \norm{g_2}_2^2) \\
		= \, & u_2 + \eta (I_d - u_2 u_2^\top) g_2 + O(\eta^2 \norm{g_2}_2^2),
	\end{align*}
    which implies
    \begin{equation*}
    	v_2- u_2 = \eta (I_d - u_2 u_2^\top) g_2 + O(\eta^2 \norm{g_2}_2^2).
    \end{equation*}
    The proof is completed by noting that
    \begin{equation*}
    	v_1 - u_1 = \eta (I_d - u_1 u_1^\top) g_1.
    \end{equation*}
\end{proof}

\subsection{Proof of Theorem~\ref{thm:MainLearning}}\label{sec:proof_MainLearning}
By our assumption, we know that the \modif{canonical learning order} holds up to level $L$, and that
\begin{equation*}
	\frac{1}{2} \sum_{k \ge L+1} \varphi_k^2 \le \frac{\delta}{4}.
\end{equation*}
Then, according to Definition~\ref{def:StandardScenario}, there exists $\veps_* = \veps_* (\delta)$, $T_0 = T_0 (\delta)$ such that for all $\veps \le \veps_*$ and $T = T(\veps, \delta) = T_0 (\delta) \veps^{1/2L}$, one has
\begin{equation*}
	\Risk_{\infty} (T, \veps) = \lim_{m \to \infty} \lim_{d \to \infty} \Risk (a(T), u(T)) \le \frac{\delta}{3}.
\end{equation*}
Moreover, from Section~\ref{sec:LargeNet} we know that with probability at least $1 - e^{-C' m}$ over the i.i.d. initialization,
\begin{equation}
	\sup_{t \in [0, T]} \left\vert \Risk (a(t), u(t)) - \Risk_{\infty} (t, \veps) \right\vert \le \left( \frac{1}{\sqrt{d}} + \frac{1}{\sqrt{m}} \right) C M' \exp(M' T (1+T)^2 / \veps^2),
\end{equation}
where $M'$ only depends on $(\sigma, \varphi, \rP_A)$.
Now we choose $\veps \le \veps_*$ and $T = T(\veps, \delta) = T_0 (\delta) \veps^{1 / 2 L}$. It then follows that
\begin{equation}
	\Risk (a(T), u(T)) \le \frac{\delta}{3} + \left( \frac{1}{\sqrt{d}} + \frac{1}{\sqrt{m}} \right) C M' \exp(M' T^3 / \veps^2).
\end{equation}
According to Theorem~\ref{thm:diff_gf_psgd}, we know that with probability at least $1 - \exp(-z)$,
\begin{equation}
	\left\vert \Risk (\oa (n), \ou (n)) - \Risk (a(T), u(T)) \right\vert \le \sqrt{\eta (d + \log m + z)} M' \exp \left( M' T^3 / \veps^3 \right)
\end{equation}
with $n = T / \eta = T(\veps, \delta) / \eta$. We now take
\begin{equation*}
	M = M(\veps, \delta) = \max \left\{ \frac{9M'^2 \exp(2M' T(\veps, \delta)^3 / \veps^3)}{\delta^2}, \ \frac{36 C^2 M'^2 \exp(2M' T(\veps, \delta)^3 / \veps^2)}{\delta^2} \right\}.
\end{equation*}
Then, by our choice of $m$ and $d$, we know that $\Risk (a(T), u(T)) \le 2 \delta / 3$. Further, taking
\begin{equation}
	\eta = \frac{1}{M(d + \log m + z)}, \ n = MT (d + \log m + z),
\end{equation}
we obtain that
\begin{equation}
	\Risk(\oa (n), \ou (n)) \le \Risk(a(T), u(T)) + \frac{\delta}{3} \le \delta.
\end{equation}
The above happens with probability $1 - \exp(-C'm) - \exp(-z)$. Hence, our conclusion follows naturally from the assumption $m \ge z$.

\section{Counterexamples to the \modif{canonical learning order}}\label{sec:counter_ex}
\subsection{Case 1: $\sigma_k = 0$ for some $k \in \mathbb{N}$}\label{sec:counter_ex1}
For any fixed $(a, u) = (a_i, u_i)_{1 \le i \le m}$, we have
\begin{align*}
	\E \left[ f(x; a, u) \He_k (\langle u_*, x \rangle) \right] =\, & \frac{1}{m} \sum_{i=1}^{m} a_i \E \left[ \sigma \left( \langle u_i, x \rangle \right) \He_k (\langle u_*, x \rangle) \right] \\
	=\, & \frac{1}{m} \sum_{i=1}^{m} a_i \sigma_k \langle u_i, u_* \rangle^k = 0.
\end{align*}
Moreover, the risk is always lower bounded by
\begin{align*}
	& \Risk (a, u) = \frac{1}{2} \E \left[ \left( \varphi (\langle u_*, x \rangle) - f(x; a, u) \right)^2 \right] \\
	=\, & \frac{1}{2} \E \left[ \left( \varphi_k \He_k (\langle u_*, x \rangle) + \left( \varphi (\langle u_*, x \rangle) - \varphi_k \He_k (\langle u_*, x \rangle) - f(x; a, u) \right) \right)^2 \right] \\
	\stackrel{(i)}{=}\, & \frac{1}{2} \varphi_k^2 + \frac{1}{2} \E \left[ \left( \varphi (\langle u_*, x \rangle) - \varphi_k \He_k (\langle u_*, x \rangle) - f(x; a, u) \right)^2 \right] \ge \frac{1}{2} \varphi_k^2,
\end{align*}
where $(i)$ follows from orthogonality between $\He_k (\langle u_*, x \rangle)$ and $f(x; a, u)$.

\subsection{Case 2: $\varphi_0 = \cdots = \varphi_k = 0$ for some $k \ge 1$}\label{sec:counter_ex2}
We consider the reduced mean-field equations~\eqref{eq:coupled_MF}:
\begin{equation*}
	\begin{split}
		\veps\partial_t a_i = \, & V(s_i) - \frac{1}{m} \sum_{j=1}^{m} a_j U(s_i s_j) \, ,\\
		\partial_t s_i = \, &  a_i \left( 1 - s_i^2 \right) \left( V'(s_i) - \frac{1}{m} \sum_{j = 1}^{m} a_j U'(s_i s_j) s_j \right) \, .
	\end{split}
\end{equation*}
Note that if $\varphi_0 = \varphi_1 = 0$, then $V'(s) = s \cdot v(s)$ for some continuous function $v$. Denoting $a = (a_1, \cdots, a_m)$ and $s = (s_1, \cdots, s_m)^\top$, the above equation regarding the evolution of the $s_i$'s can be written as
\begin{equation*}
	s'(t) = A(a(t), s(t)) s(t),
\end{equation*}
where $A(a, s)$ is a matrix-valued function satisfying
\begin{align*}
	A_{ij} (a, s) =\, & a_i (1 - s_i^2) \left( v(s_i) \mathbf{1}_{i = j} - \frac{a_j}{m} U'(s_i s_j) \right), \ \forall i, j \in [m].
\end{align*}
Using the similar a priori estimate as in the proof of Lemma~\ref{lem:a_bound_r_perp}, we can show that
\begin{equation*}
	\sup_{t \in [0, T]} \norm{A(a(t), s(t))}_{\op} \le C(T) < \infty
\end{equation*}
for any finite time $T$, which immediately implies that $s(t) \equiv 0$ for $t \in [0, T]$. Therefore, we won't be able to learn any component of $\varphi$ with degree $\ge 1$. 

\subsection{Case 3: $\varphi_k = 0$ for some $k \ge 1$}\label{sec:counter_ex3}
We may assume $\sigma_k \neq 0$, and analyze the simplified ODE system \eqref{eq:l-th_simp}, which reduces to
\begin{equation}
	\begin{split}
		\partial_\tau \ta(\omega) = \, & - \sigma_k^2 \ts(\omega)^k \int \ta(\nu) \ts(\nu)^k \diff \rho(\nu) \\
		\partial_\tau \ts(\omega) = \, & - k \sigma_k^2 \ta(\omega) \ts(\omega)^{k - 1} \left( 1 - \veps^{2 \beta_k} \ts(\omega)^2 \right) \int \ta(\nu) \ts(\nu)^k \diff \rho(\nu).
	\end{split}
\end{equation}
We thus obtain the following equations:
\begin{equation}
\begin{split}
	\partial_{\tau} \int \ta(\omega)^2 \diff \rho (\omega) =\, & - 2 \sigma_k^2 \left( \int \ta(\omega) \ts(\omega)^k \diff \rho(\omega) \right)^2 \le 0, \\
	\partial_{\tau} \int \ts(\omega)^2 \diff \rho (\omega) =\, & - 2 k \sigma_k^2 \int \ta(\omega) \ts(\omega)^k \left( 1 - \veps^{2 \beta_k} \ts(\omega)^2 \right) \diff \rho(\omega) \cdot \int \ta(\omega) \ts(\omega)^k \diff \rho(\omega) \\
	\le\, & 2 k \sigma_k^2 \veps^{2 \beta_k} \int \ta(\omega) \ts(\omega)^{k+2} \diff \rho(\omega) \cdot \int \ta(\omega) \ts(\omega)^k \diff \rho(\omega),
\end{split}
\end{equation}
which means that for any $\tau \ge 0$,
\begin{equation}
	\int \ta(\omega, \tau)^2 \diff \rho (\omega) \le \int \ta(\omega, 0)^2 \diff \rho (\omega) = O(\veps^{1/k(k+1)}) = o(1).
\end{equation}
Therefore, most of the neurons cannot evolve to the magnitude of $\Omega (1)$ in the process of learning the $k$-th component, and therefore fails to provide an effective initialization for learning the next component $\varphi_{k+1}$.

\end{document}